%% file: main.tex
\documentclass[11pt]{article}
\usepackage[OT1,T3]{fontenc}
\usepackage{smile}
\usepackage{acronym}
\usepackage{bold-extra}

\def\prompt{\mathtt{pt}}
\def\actor{Actor}
\def\reporter{Reporter}
\def\planner{Planner}
\def\attn{\mathtt{attn}}

\def\tdo{\mathrm{\bf do}\hspace{0.6mm}}
\def\proj{\mathtt{Proj}}
\def\Softmax{\mathtt{Softmax}}
\def\mha{\mathtt{Mha}}
\def\relu{\mathtt{ReLU}}
\def\wvert{\hspace{0.5mm}\vert\hspace{0.5mm}}
\def\wVert{\hspace{0.5mm}\Vert\hspace{0.5mm}}

\acrodef{hmdp}[HMDP]{Hierarchical Markov Decision Process}
\acrodef{mdp}[MDP]{Markov Decision Process}
\acrodef{pomdp}[POMDP]{Partially Observable Markov Decision Process}
\acrodef{bail}[BAIL]{Bayesian aggregated imitation learning}

\definecolor{NavyBlue}{RGB}{51,112,185}
\newcommand{\myblue}[1]{\textcolor{blue!70!black}{#1}}

\title{From Words to Actions: Unveiling the Theoretical Underpinnings of LLM-Driven Autonomous Systems}

\author{Jianliang He\thanks{Equal contribution.} \thanks{Fudan University. Email: \texttt{hejl20@fudan.edu.cn}} \qquad
    Siyu Chen\footnotemark[1] \thanks{Yale University. Email: \texttt{\{siyu.chen.sc3226,zhuoran.yang\}@yale.edu}.} \qquad
    Fengzhuo Zhang\thanks{National University of Singapore. Email: \texttt{fzzhang@u.nus.edu}.} \qquad
    Zhuoran Yang\footnotemark[3]}

\date{}

\begin{document}
\maketitle
\begin{abstract}
In this work, from a theoretical lens, we aim to   understand why large language model (LLM) empowered agents  are able to solve decision-making problems in the physical world. 
To this end, consider a hierarchical reinforcement learning (RL) model where the  LLM Planner and the Actor perform high-level task planning   and low-level execution, respectively.   
Under this model, the LLM Planner navigates a partially observable Markov decision process (POMDP) by iteratively generating language-based subgoals via prompting. 
Under proper assumptions on the pretraining data, we prove that  the pretrained LLM Planner effectively performs Bayesian aggregated imitation learning (BAIL) through in-context learning. 
Additionally, we highlight the  necessity for exploration beyond the subgoals derived from BAIL by proving that naively executing the subgoals returned by LLM leads to a linear regret. 
As a remedy, we introduce an $\epsilon$-greedy exploration strategy to BAIL, which is proven to incur sublinear regret when the pretraining error is small.
Finally, we extend our theoretical framework to include scenarios where the LLM Planner serves as a world model for inferring the transition model of the environment and to multi-agent settings, enabling coordination among multiple Actors.
\end{abstract}

{

\tableofcontents
}
\newpage

\input{tex/main/intro.tex}

\input{tex/main/preliminaries.tex}

\input{tex/main/setup.tex}

\input{tex/main/perfect.tex}
\input{tex/main/pretrain.tex}

\bibliographystyle{apalike}
\bibliography{reference}

\newpage 
\appendix
\begin{center}
\LARGE {Appendix for  ``From Words to Actions: Unveiling the Theoretical Underpinnings of LLM-Driven Autonomous Systems''}
\end{center}
\input{tex/appendix/appendix_background.tex}
\input{tex/main/extention.tex}
\input{tex/appendix/appendix_perfect.tex}

\input{tex/appendix/appendix_imperfect.tex}

\input{tex/appendix/appendix_extention.tex}
\input{tex/appendix/appendix_lemma.tex}

\end{document}

%% file: tex/main/intro.tex
\section{Introduction}

The advent of large language models (LLMs) such as GPT-4 \citep{openai2023gpt} and Llama 2 \citep{touvron2023llama} has marked a significant leap in artificial intelligence, thanks to their striking capabilities in understanding language and performing complex reasoning tasks. These capabilities of LLMs have led to the emergence of LLM-empowered agents (LLM Agents), where LLMs are used in conjunction with tools or actuators to solve decision-making problems in the physical world. LLM Agents have showcased promising empirical successes in a wide range of applications, including autonomous driving \citep{wang2023empowering,fu2024drive}, robotics \citep{brohan2023can,li2023interactive}, and personal assistance \citep{liu2023reason,nottingham2023embodied}. This progress signifies a crucial advancement in the creation of intelligent decision-making systems, distinguished by a high degree of autonomy and seamless human-AI collaboration.

LLMs only take natural languages as input.
To bridge the language and physical domains, LLM-agents typically incorporate three critical components: an LLM \planner, a physical  \actor, and a multimodal \reporter, functioning respectively as the brain, hands, and eyes of the LLM-agent, respectively.  
Specifically, upon receiving a task described by a human user, the LLM Planner breaks down the overall task into a series of subgoals. 
Subsequently, the Actor implements each subgoal  in the physical world through a sequence of actions. Meanwhile, the Reporter monitors changes in the physical world and conveys this information back to the LLM Planner in natural language form. This dynamic interaction among Planner, Actor, and Reporter empowers LLM Agents to understand the environment, formulate informed decisions, and execute actions effectively, thus seamlessly integrating high-level linguistic subgoals with low-level physical task execution.

The revolutionary approach of LLM Agents represents a paradigm shift away from traditional learning-based decision-making systems.
Unlike these conventional systems, LLM Agents are not tailored to any specific task. Instead, they rely on the synergy of their three distinct components—each trained separately and often for different objectives. 
In particular, the LLM Planner is trained to predict the next token in a sequence on vast document data. 
Moreover, when deployed to solve a task, the way to interact with the LLM Planner is via prompting with the LLM fixed. 
The Actor, as language-conditioned policies, can be trained by RL or imitation learning.
Moreover, the Reporter, as a multimodal model, is trained to translate the physical states (e.g., images) into natural language.
This unique configuration prompts critical research questions regarding the theoretical underpinnings of LLM Agents, particularly concerning their decision-making effectiveness.

\begin{figure}[ht]
\centering
\includegraphics[width=0.85\textwidth]{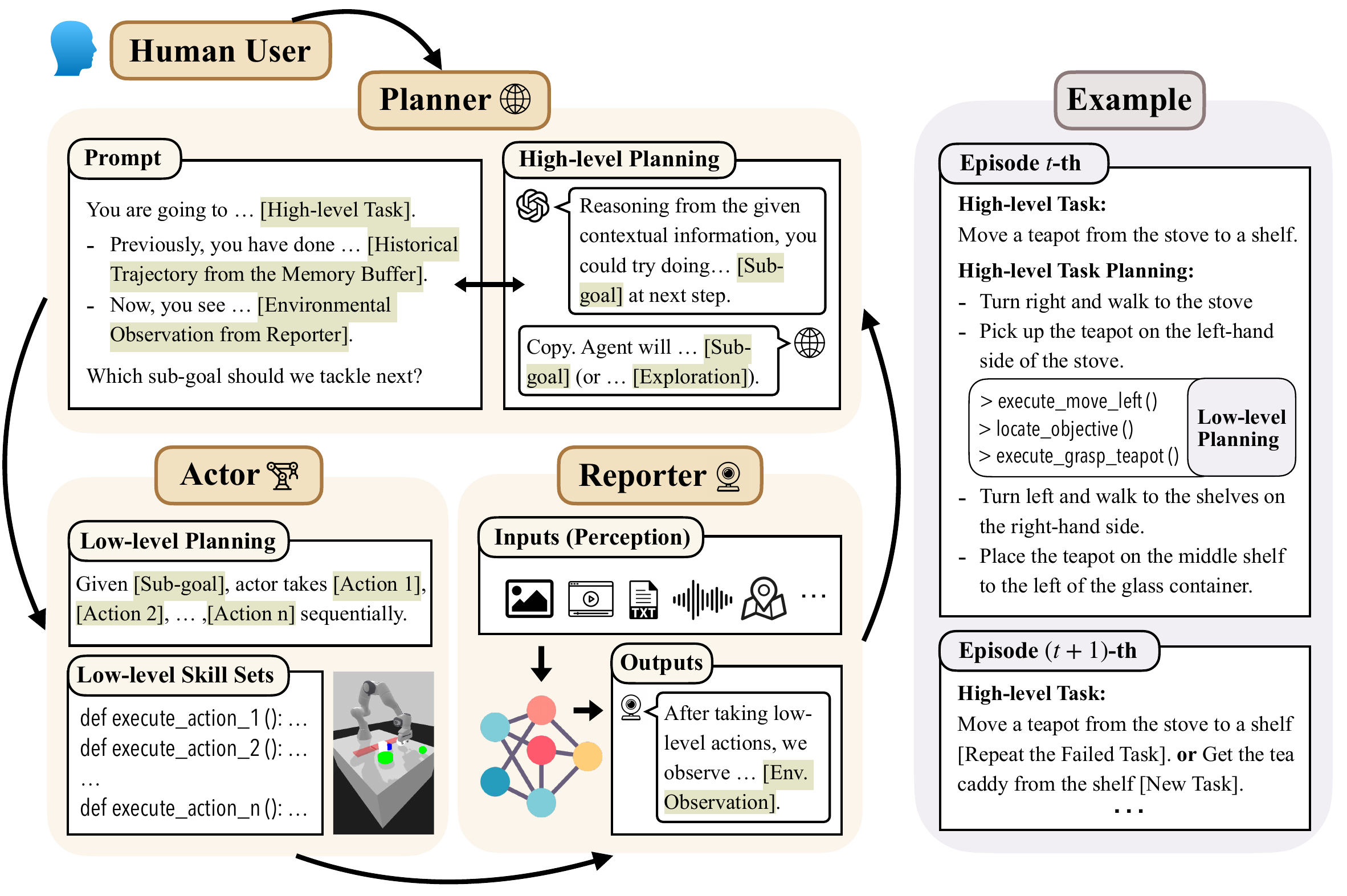}
\caption{\small Overview of the Planner-Actor-Reporter (PAR) system as LLM Agents. Acting as a central controller, the \planner\ conducts the high-level planning by storing the history and reasoning through the iterative use of the ICL ability of LLMs, coupled with explorations. The \actor\ handles low-level planning and executes subgoals using pre-programmed skill sets, and the \reporter\ perceives and processes multimodal information from environment to reinforce the ongoing planning.}
\label{fig:taskplannning}
\vspace{-3mm}
\end{figure}

In this work, we  make an initial step toward  developing a theoretical framework for understanding the dynamics and effectiveness of LLM Agents. Specifically, we aim to answer the following questions:
\myblue{(a)} What is a theoretical model for understanding the performance of LLM Agents?
\myblue{(b)} How do pretrained LLMs solve decision-making problems in the physical world via prompting?
\myblue{(c)} How does an LLM Agent address the exploration-exploitation tradeoff?
\myblue{(d)} How do the statistical errors of the pretrained LLM and Reporter affect the overall performance of the LLM Agent?

To address Question \myblue{(a)}, we propose analyzing LLM Agents within a hierarchical reinforcement learning framework \citep{barto2003recent,pateria2021hierarchical}, positioning the LLM Planner and the Actor as policies operating within high-level POMDPs and low-level MDPs, respectively (\S\ref{sec:PAR}). 
Both levels share the same state space—namely, the physical state—though the LLM Planner does not directly observe this state but instead receives a language-based description from the Reporter, effectively navigating a POMDP. 
The action space of the high-level POMDP is  the set of language subgoals. 
Meanwhile,  the state transition kernel is determined by the pretrained Actor, and thus is associated with a variable $z$ that summarizes its dependency on low-level Actor. Such a variable is unknown to the LLM Planner. 
After pretraining, without prior knowledge of the Actor's quality or the physical environment, the LLM Planner attempts to solve the high-level POMDP by iteratively generating a sequence of subgoals based on feedback from the Reporter via prompting.
Under this framework, the overall performance of the LLM Agent can be captured by the regret in terms of finding the optimal policy of the hierarchical RL problem in the online setting (\S\ref{sec:metric}).

Furthermore, to answer Question \myblue{(b)}, we prove that when the pretraining data includes a mixture of expert trajectories,   during the prompting stage, the pretrained LLM Planner essentially performs Bayesian aggregated imitation learning (BAIL) through in-context learning (Theorem \ref{thm:BAIL}). 
This process involves constructing a posterior distribution over the hidden parameter 
$z$of the transition kernel, followed by generating subgoals that emulate a randomly selected expert policy, weighted according to this posterior distribution.  
Such a Bayesian learning mechanism is encoded by the LLM architecture and is achieved through prompting.

However, since the LLM has no prior knowledge of the physical environment, it needs to guide the Actor to explore the physical environment. We prove that merely adhering to BAIL-derived subgoals can lead to the inadequate exploration, resulting in a linear regret (Proposition \ref{prop:hardexample}). 
To mitigate this, i.e., Question \myblue{(c)}, we introduce an 
$\epsilon$-greedy
 exploration strategy, which occasionally deviates from BAIL subgoals in favor of exploration, significantly enhancing learning efficacy by ensuring a  sublinear regret (Theorem \ref{thm:regret}). 
Specifically, to address Question \myblue{(d)} we establish that the regret is bounded by a sum of two terms (Theorem \ref{thm:practical regret}): a 
 $\sqrt{T}$-regret related to the number of episodes the LLM Agent is deployed to the hierarchical RL problem, and an additional term representing the  statistical error from pretaining the LLM Planner and Reporter via maximum likelihood estimation (MLE) and contrastive learning, respectively (Theorem \ref{thm:llmpretrain}, \ref{thm:translatorpretrain}).  

Finally, we extend our analysis to scenarios where the Planner utilizes the LLM as  world model for inferring the upper-level POMDP's transition model via Bayesian model aggregation (Proposition \ref{thm:bawm}, Corollary \ref{cor:bawm_regret}). Our theoretical framework also accommodates a multi-agent context, where the LLM Planner coordinates with a collaborative team of low-level actors (Corollary \ref{corol:perfectregret}).

%% file: tex/main/preliminaries.tex
\section{Preliminaries and Related Works}\label{sec:preliminaries}

\paragraph{Large Language Models.}
The Large Language Models (LLMs) such as ChatGPT \citep{brown2020language}, GPT-4 \citep{openai2023gpt}, Llama \citep{touvron2023llama}, and Gemini \citep{team2023gemini}, are pretrained on vast text corpora to predict in an \emph{autoregressive} manner. Starting from an initial token $\ell_1\in\mathfrak{L}\subseteq\RR^{d}$, where $d$ denotes the dimension of token vector and $\mathfrak{L}$ denotes the language space, the LLM, with parameters $\theta\in\Theta$, predicts the next token with $\ell_{t+1}\sim\mathtt{LLM}_\theta(\cdot\wvert S_t)$, where $S_t=(\ell_1,\dots,\ell_t)$ and $t\in\mathbb{N}$. Each token $\ell_t\in\mathfrak{L}$ specifies a word or word's position, and the token sequence $S_t$ resides in the space of token sequences $\mathfrak{L}^*$. Such an autoregressive generating process terminates when the stop sequence token is generated.
\vspace{-10pt}
\paragraph{In-Context Learning.} LLMs haved exhibited robust reasoning capabilities and a crucial aspect of their reasoning prowess is the in-context learning (ICL) ability. This ability is further enhanced through additional training stages \citep{iyer2022opt}, careful selection and arrangement of informative demonstrations \citep{liu2021makes, kim2022self}, explicit instruction \citep{honovich2022instruction}, and use of prompts to stimulate chain of thoughts \citep{wei2022chain}. Unlike fine-tuned models customized for specific tasks, LLMs showcase comparable capabilities by learning from the informative \emph{prompts} \citep{li2022pre,liu2022p}. Assume that prompt, denoted by $\prompt_t=(\ell_1,\dots,\ell_t)\in\mathfrak{L}^*$, is generated based on a latent variable $z\in\mathcal{Z}$ autoregressively. The token follows a generating distribution such that $\ell_{t}\sim\mathbb{P}(\cdot\wvert\prompt_{t-1},z)$ and $\prompt_t=(\prompt_{t-1},\ell_{t})$, where $\mathcal{Z}$ denotes the space of hidden information or concepts. The latent structure is commonly employed in language models, including topic models like LDA \citep{blei2003latent}, BERT \citep{devlin2018bert}, generative models like VAE \citep{kusner2017grammar}, T5 \citep{raffel2020exploring}, and is also widely adopted in the theoretical analysis of ICL \citep{xie2021explanation,zhang2023and}. 
Theoretical understanding of ICL is an active area of research. Since real-world datasets used for LLM pretraining are difficult to model theoretically and are very large, ICL has also been studied in stylized setups \citep{xie2021explanation,muller2021transformers,garg2022can,chan2022data,hahn2023theory,zhang2023and}. In this paper, we build upon the framework attributing the ICL capability to Bayesian inference \citep{xie2021explanation,jiang2023latent,zhang2023and}, which posits that the pretrained LLMs predict the next token with probability by aggregating the generating distribution concerning latent variable $z\in\mathcal{Z}$ over the posterior distribution. Moreover, a series of practical experiments, including \citet{wang2023large,ahuja2023context}, provide empirical support for this Bayesian statement.

\vspace{-10pt}
\paragraph{LLM Agents.}
LLMs, as highlighted in \cite{openai2023gpt}, are powerful tools for the task planning \citep{wei2022emergent,hu2023language}. The success of LLM agent marks a shift from task-specific policies to a pretrain-finetune-prompt paradigm. By breaking down the complex tasks into subgoals, LLM Agent facilitates the effective zero-shot resource allocation across environments. For instance, envision a scenario where a robotic arm is tasked with ``\emph{move a teapot from the stove to a shelf}", a task for which the robotic arm may not be pretrained.  
However, leveraging LLMs allows the decomposition of the task into a sequence of executable subgoals: ``\emph{grasp the teapot}", ``\emph{lift the teapot}", ``\emph{move the teapot to the shelf}", and ``\emph{release the teapot}".

In the conventional task-planning and decision-making problems, symbolic planners have commonly been employed to transform them into search problems \citep{bonet2001planning,ghallab2004automated} or to design distinct reinforcement learning or control policies for each specific scenario. Recent empirical studies have shifted towards leveraging LLMs as symbolic planners in various domains, including robotic control \citep{mandi2023roco,brohan2023can,li2023interactive,du2023video}, autonomous driving \citep{wang2023empowering,fu2024drive} and personal decision assistance \citep{li2022pre,lin2023grounded,hu2023tree,liu2023reason,nottingham2023embodied}. Another recent line of research has been dedicated to devising diverse prompting schemes to enhance the reasoning capability of LLMs \citep{wei2022chain,yao2023tree,yao2023retroformer,hao2023reasoning}. Despite the considerable empirical success, there is a lack of comprehensive theoretical analysis on LLM Agent. In this paper, we formalize this approach into a hierarchical LLM-empowered planning framework and provide a theoretical analysis of its performance.
 Two recent works by \citet{liu2023reason} and \citet{lee2023supervised} also aim to establish provable algorithms for planning with LLMs or decision-pretrained Transformers (DPT).  In comparison, we discuss both the plausibility of taking LLMs as a subgoal generator \citep{lee2023supervised} and simulated world model \citep{liu2023reason}. Furthermore, we provide a statistical guarantee for pretrained models and conduct a detailed examination of the algorithm's performance in practical settings, bringing our analysis closer to real-world applications.

%% file: tex/main/setup.tex
\section{Theoretical Framework for LLM Agents}
To formalize the architecture of LLM Agents, we propose a general theoretical framework---\underline{P}lanner-\underline{A}ctor-\underline{R}eporter (PAR) system. Furthermore, the problem is modeled as a hierarchical RL problem \citep{pateria2021hierarchical}. Specifically, the \planner, empowered by LLMs, conducts high-level task planning within the language space; the \actor, pretrained before deployment, undertakes low-level motion planning within the physical world; and the \reporter, equipped with a sensor to sense the physical environment, processes the information and feeds it back to the \planner, bridging the gap between language space and the physical world (see \S\ref{sec:PAR}). Additionally, we present the performance metric and pretraining methods of LLMs for the \planner\ and translators for the \reporter\ in \S\ref{sec:metric}.

\subsection{Planner-Actor-Reporter System}\label{sec:PAR} 
In this section, we delve into details of the PAR system under \ac{hmdp}. 
At the high level, the \planner\ empowered by LLM handles task planning by decomposing tasks into subgoals to solve a language-conditioned \ac{pomdp} with a finite horizon $H$.
At the low level, the \actor\ translates these subgoals into the actionable steps in the physical world to handle a language-conditioned \ac{mdp} with a finite horizon $H_a$\footnote{Throughout the paper, we use the notation $\bar{\cdot}$ to distinguish low-level elements from their high-level counterparts.}. Please refer to the right panel of Figure \ref{fig:taskplannning} for a detailed example of LLM Agent, and see Figure \ref{fig:HMDP} for an overview of the hierarchical interactive process.

\paragraph{Low-level MDP.~}Let $\cG\subseteq\mathfrak{L}$ be the space of language subgoals, $\cS$ and $\cA$ respectively denote the space of physical states and actions. At high-level step $h$, the low-level MDP is specified by a transition kernel $\TT_h=\{\TT_{h,\bar{h}}\}_{\bar{h}\in[H_a]}$ and the rewards that depends on subgoal $g\in\cG$. Following this, the \actor\ is modelled as a language-conditioned policy $\mu=\{\mu_g\}_{g\in\cG}$, where $\mu_g=\{\mu_{\bar{h}}(\cdot|\cdot,g)\}_{\bar{h}\in[H_a]}$ and $\mu_{\bar{h}}:\cS\times\cG\mapsto\Delta(\cA)$. Assume that the \actor\ stops at step $H_a+1$, regardless of the subgoal achievement. Subsequently, the \planner\ receives the observation of the current state $\bar{s}_{h,H_a+1}$ from the \reporter, and sends a new subgoal to the \actor\ based on the historical feedback.

\paragraph{High-level POMDP.~}Suppose that a low-level episode corresponds to a single high-level action of the \planner. Thus, the high-level \ac{pomdp} reuses the physical state space $\cS$ as the state space, but takes the subgoal space $\cG$ as the action space instead. Following this, the high-level transition kernel is jointly determined by the low-level policy $\mu$ and the physical transition kernel $\TT$ such that 
\begin{align}
\PP_{z,h}(s'&\wvert s,g)=\PP\big(\bar{s}_{h,H_a+1}=s'\wvert\bar{s}_{h,1}=s,a_{h,1:\bar{h}}\sim\mu_g,\bar{s}_{h,2:\bar{h}+1}\sim\TT_{h}\big),
\end{align} 
where we write $z =(\TT,\mu)$. 
Since the LLM-empowered \planner\ cannot directly process the physical states, it relies on some (partial) observations generated by the \reporter. Specifically, let $o_h\in\cO$ describe the physical state $s_h\in\cS$ in language through a translation distribution $\OO:\cO\mapsto\Delta(\cS)$, where $\cO\subseteq\mathfrak{L}$ denotes the space of observations. At each step $h\in[H]$, a reward $r_{h}(o_h,\omega)\in[0,1]$ is obtained, which depends on both the observation and the task $\omega\in\Omega$ assigned by human users. Here, $\Omega\subseteq\mathfrak{L}$ denotes the space of potential tasks in language. 

\paragraph{Interactive Protocol.~}
The \planner\ aims to determine a sequence of subgoal $\{g_h\}_{h\in[H]}$ such that when the \actor\ is equipped with policy $\pi=\{\pi_{h}\}_{h\in[H]}$, these subgoals maximize the expected sum of rewards. During task planning, the \planner\ must infer both \actor's intention, i.e., policy $\mu$, and the environment, i.e., physical transition kernel $\TT$, from the historical information. Thus, $z$ constitutes all the latent information to the high-level \planner, and denote $\cZ$ as the space of all potential latent variables with $|\cZ|<\infty$.
\begin{figure*}[t]
\centering
\includegraphics[width=0.95\textwidth]{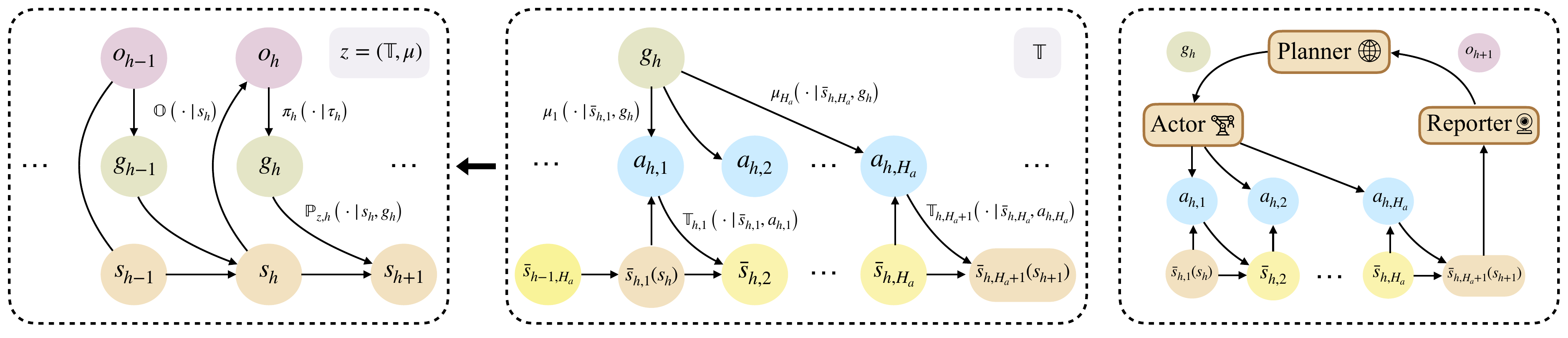}
\caption{\small Illustration of structure of \ac{hmdp}. The low-level MDP is featured by transition kernel $\TT$, which characterizes the dynamics of the physical environment. The high-level transition is a result of a sequence of low-level actions in the physical environment, guided by policies $\mu=\{\mu_g\}_{g\in\cG}$. Thus, high-level \ac{pomdp} incorporates latent information $z=(\TT,\mu)$ originated from the low-level.}
\label{fig:HMDP}
\end{figure*}
To summarize, the interactive protocol is as below: at the beginning of each episode $t$, \planner\ receives a task $\omega_t$. At step $h$, each module follows:
\vspace{-10pt}
\paragraph{Module~1:~Planner.} After collecting $o_{h}^t$ from \reporter,  the \planner\ leverages LLMs for recommendations on task decomposition, and the policy is denoted by $\pi^t_{h,\mathtt{LLM}}:\cT^*\times(\mathcal{O}\times\mathcal{G})^{h-1}\times\mathcal{O}\times\Omega\mapsto\Delta(\mathcal{G})$, where $\cT^*$ represents the space of the trajectory sequence with arbitrary length. LLM's recommendation is obtained by invoking the ICL ability with the history-dependent prompt:
\begin{equation}
 \prompt_h^t= \mathcal{H}_t\cup\left\{\omega^t,\tau_{h}^t\right\},\quad \mathcal{H}_t=\bigcup_{i=1}^{t-1}\left\{\omega^i,\tau_{H}^i\right\},
	\label{eq:prompt}
\end{equation}
where $\mathcal{H}_t\in\cT^*$ denotes the historical context and $\tau_{h}^t=\{o_{1}^t,g_{1}^t,\dots,o_{h}^t\}$ is the trajectory until $h$-th step. In the PAR system, \planner\ retains autonomy and is not obligated to follow LLM's recommendations. Let $\pi^t_h$ be the \planner's policy, which partially leverages the LLM's recommendation $\pi^t_{h,\mathtt{LLM}}(\cdot\wvert \tau_{h}^t,\omega^t)=\mathtt{LLM}_\theta(\cdot\wvert\prompt_h^t)$. The \planner\ selects $g_{h}^t\sim\pi_h^t(\cdot\wvert\tau_{h}^t,\omega^t)$, and sends it to the \actor.
 
\vspace{-5pt}
\paragraph{Module~2:~Actor.} Upon receiving $g_h^t$ from \planner,  the \actor\ plans to implement $g_h^t$ in physical world with pretrained skill sets, denoted by a subgoal-conditioned policy $\mu=\{\mu_g\}_{g\in\cG}$. A sequence of actions $\{a_{h,\bar{h}}\}_{\bar{h}\in[H_a]}$ is executed, where the dynamics follows $a_{h,\bar{h}}\sim\mu_{\bar{h}}(\cdot\wvert\bar{s}_{h,\bar{h}},g_h^t)$ and $\bar{s}_{h,\bar{h}+1}\sim\TT_{h,\bar{h}}(\cdot\wvert\bar{s}_{h,\bar{h}},a_{h,\bar{h}})$ starting from $\bar{s}_{h, 1}=s_h^t$. The low-level episode concludes at $s_{h+1}^t=\bar{s}_{h, H_a+1}$.
\vspace{-5pt}
\paragraph{Module~3:~Reporter.} After the low-level episode concludes, the \reporter\ collects and reports the current state $s_h^t$ via observation $o^t_{h+1}$ generated from $\mathbb{O}_\gamma(\cdot\wvert s^t_{h+1})$, where $\mathbb{O}_\gamma:\mathcal{S}\mapsto\Delta(\mathcal{O})$ denotes the distribution of the pretrained translator. Subsequently, the observation $o_{h+1}^t$ of the current state is sent back to the \planner, reinforcing to the ongoing task planning.

\vspace{10pt}
The strength of the PAR system lies in its resemblance to RL \citep{sutton2018reinforcement}, allowing the \planner\ to iteratively adjust its planning strategy based on feedback from the \reporter. Moreover, the \reporter\ empowers the system to process the real-time information and the integration of multiple modalities of raw data like RGB, images, LiDAR, audio, and text \citep{li2023multimodal,xu2023multimodal}. The \actor's skill sets can effectively be pretrained using the goal-conditioned RL \citep{chane2021goal,liu2022goal}, language-to-environment grounding \citep{brohan2023can,huang2022language} or pre-programmed manually \citep{singh2023progprompt}.

\subsection{Performance Metric and Pretraining}\label{sec:metric}

\paragraph{Performance Metric.} In this paper, we focus on the \emph{performance of the high-level \planner}, and regard the low-level \actor\ as an autonomous agent that can use the pretrained skill sets following a fixed policy. For any latent variable $z\in\mathcal{Z}$ and policy $\pi=\{\pi_h\}_{h\in[H]}$ with $\pi_h:(\mathcal{O}\times\mathcal{G})^{h-1}\times\mathcal{O}\times\Omega\mapsto\Delta(\mathcal{G})$, the value function is defined as
\begin{equation}
\mathcal{J}_{z}(\pi,\omega):=\mathbb{E}_{\pi}\left[\sum_{h=1}^Hr_{h}\left(o_{h},\omega\right)\right],
\label{eq:value}	
\end{equation}
where the expectation is taken concerning the initial state $s_1\sim\rho$, policy $\pi$, ground-truth translation distribution $\mathbb{O}$, and transition kernel $\mathbb{P}_z$. For all $(z,\omega)\in\cZ\times\Omega$, there exists an optimal policy $\pi_{z}^*\left(\omega\right)={\rm argmax}_{\pi\in\Pi} \mathcal{J}_{z}(\pi,\omega)$, where $\Pi=\{\pi=\{\pi_h\}_{h\in[H]},\pi_h:(\mathcal{O}\times\mathcal{G})^{h-1}\times\mathcal{O}\times\Omega\mapsto\Delta(\mathcal{G})\}$. 

To characterize the performance under practical setting, we denote $\mathcal{\hat J}_{z}(\pi,\omega)$ as the value function concerning the pretrained translator $\mathbb{O}_{\hat\gamma}$, and for all $\omega\in\Omega$, let $\hat{\pi}_{z}^*\left(\omega\right)={\rm argmax}_{\pi\in\Pi}\ \hat{\mathcal{J}}_{z}(\pi,\omega)$ be the optimal policy in practice. Then, the regret under practical setting is defined as
\begin{equation}
{\rm Reg}_z(T):=\sum_{t=1}^T\EE_{\cH_t}\left[\mathcal{\hat{J}}_z(\hat{\pi}_z^*,\omega^t)-\mathcal{\hat{J}}_z(\hat{\pi}^t,\omega^t)\right],
\end{equation}
where $\{\hat{\pi}^t\}_{t\in[T]}$ represents the \planner's policy empowered by a pretrained $\mathtt{LLM}_{\hat{\theta}}$ and the expectation is taken with respect to the context $\cH_t$ defined in \eqref{eq:prompt} generated by taking $\{\hat{\pi}^{i}\}_{i<t}$ sequentially. Here, we focus on the performance when the \planner\ collaborates with a pretrained PAR system in an environment characterized by $z$ and pretrained \reporter. Our goal is to design a sample-efficient algorithm that achieves a sublinear regret, i.e., ${\rm Reg}_z(T)=o(T)$. 

\paragraph{Pretraining Dataset Collection.}
The pretraining dataset consists of $N_{\rm p}$ independent samples with $T_{\rm p}$ episodes such that $\cD=\{D_n\}_{n\in[N_{\rm p}]}$, where $D_n=\{z\}\cup\{\omega^t,\tau_H^t,g_{1:H}^{t,*},s_{1:H}^t\}_{t\in[T_{\rm p}]}$. For each sample, $z\sim\cP_\cZ$ specifies a low-level MDP with language-conditioned policies and $\omega^t{\sim}\cP_\Omega$ specifies the sequence of high-level tasks. Here, $\mathcal{P}_\mathcal{Z}$ and $\cP_\Omega$ denote the prior distributions. We assume that the joint distribution of each data point $D$ in the dataset, denoted by $\PP_\cD$, follows that:
\begin{align}
\PP_\cD(D)&=\cP_\cZ(z)\cdot\prod_{t=1}^{T_{\rm p}}\cP_\Omega(\omega^t)\cdot\prod_{h=1}^{H_{\rm }}\pi^*_{z,h}(g_h^{t,*}\wvert\tau_h^t,\omega^t)\nonumber\\
&\quad \cdot\mathbb{O}(o_h^t\wvert s_h^{t})\cdot\pi^b_{h}(g_h^t\wvert\tau_h^t,\omega^t)\cdot\PP_{z,h}(s_{h+1}^t\wvert s_h^{t},g_h^{t}),
\label{eq:joint_distribution}
\end{align}
where $\pi^b=\{\pi^b_h\}_{h\in[H]}$ is the behavior policy that features how the contextual information is collected, and additionally the label, i.e., optimal subgoal, is sampled from the optimal policy $\pi^*_{z}$ by experts. Subsequently, the latent information $z$ is hidden from the context. 
\paragraph{LLM Pretraining.}
To pretrain LLMs, we adopt a supervised learning approach concerning the transformer structure, aligning with the celebrated LLMs such as BERT and GPT \citep{devlin2018bert,brown2020language}. Specifically, the pretraining data is constructed based on $\cD$. For clarity, we extract the language data without expert knowledge and write the collected data into a sequence of ordered tokens, i.e., sentences or paragraphs. For the $n$-th sample $D_n$, we write
\begin{align}
(\ell_1^n,\dots,\ell^n_{\bar{T}_{\rm p}}):=\left(\omega^{n,t},o_1^{n,t},g_1^{n,t},\dots,o_{H-1}^{n,t},g_{H-1}^{n,t},o_H^{n,t}\right)_{t\in[T_{\rm p}]},
\end{align}
 with a length of $\bar{T}_{\rm p}=2HT_{\rm p}$, which contains $T_{\rm p}$ episodes with one task, $H$ observations and $H-1$ subgoals each. Following this, LLM's pretraining dataset is autoregressively constructed with the expert guidance, denoted by $\mathcal{D}_\mathtt{LLM}=\{(\tilde{\ell}_t^{n},S_t^n)\}_{(n,t)\in[N_{\rm p}]\times[\bar{T}_{\rm p}]}$, where $S_{t+1}^n=(S_t^n,\ell^n_t)$ and let
 $$
\left\{
\begin{aligned}
	&\tilde{\ell}_{t'}^n=g_h^{n,t,*}\hspace{0.4cm}\text{if~}t'=2H(t-1)+2h+1,\\
	&\tilde{\ell}_{t'}^n=g_h^{n,t}\hspace{0.65cm}\text{otherwise}.
\end{aligned}\right.
$$
In other words, when pretraining to predict the next subgoal, we replace the one sampled from the behavior policy with the one from the optimal policy. In practice, sentences with expert knowledge can be collected from online knowledge platforms such as Wikipedia \citep{merity2016pointer, reid2022can}. Following the pretraining algorithm of BERT and GPT, the objective is to minimize the cross-entropy loss, which can be summarized as $\hat{\theta}={\rm argmin}_{\theta\in\Theta}\ \mathcal{L}_\mathrm{CE}(\theta;\mathcal{D}_\mathtt{LLM})$ with
\begin{align}
	\mathcal{L}_\mathrm{CE}(\theta;\mathcal{D}_\mathtt{LLM}):=\mathbb{\hat E}_{\mathcal{D}_\mathtt{LLM}}\left[-\log\mathtt{LLM}_\theta(\ell\wvert S)\right], 
	\label{eq:crossentropyloss}
\end{align}
and $\mathtt{LLM}_{\hat{\theta}}$ is the pretrained LLM by algorithm in \eqref{eq:crossentropyloss}. More details are deferred to \S\ref{sec:llmpretrain}.

\paragraph{Translator Pretraining.}
To pretrain translators, we employ a self-supervised contrastive learning approach, which aligns with celebrated vision-language models such as CLIP \citep{radford2021learning} and ALIGN \citep{jia2021scaling}. Let $\mathcal{D}_\mathtt{Rep}$ be the contrastive pretraining dataset for translators, which is also constructed upon the dataset $\cD$. Following the framework adopted in \cite{qiu2022contrastive,zhang2022making}, for each observation-state pair $(o,s)\in\cD$, a positive or a negative data point, labelled as $y=1$ and $y=0$, is generated with equal probability, following that
\begin{itemize}[leftmargin=*]
	\setlength{\itemsep}{-1pt}
	\item[-] \textbf{Positive Data:} Collect $(o,s)$ with label $y=1$.
	\item[-] \textbf{Negative Data:} Collect $(o,s^-)$ with label $y=0$, where $s^-$ is sampled from negative sampling distribution $\mathcal{P}^{-}\in\Delta(\cO)$ that has a full support over the domain of interest.
\end{itemize}
Denote $\PP_{\cC}$ as the joint distribution of data collected by the process above. The learning algorithm follows that $\hat{\gamma}={\rm argmin}_{\gamma\in\Gamma}\ \mathcal{L}_\mathrm{CT}(\gamma;\mathcal{D}_\mathtt{Rep})$, where the contrastive loss $\mathcal{L}_{\rm CT}(\gamma;\mathcal{D}_\mathtt{Rep})$ is defined as
\begin{align}
	\mathcal{L}_{\rm CT}(\gamma;\mathcal{D}_\mathtt{Rep}):=\mathbb{\hat E}_{\mathcal{D}_\mathtt{Rep}}[y\cdot\log\left(1+{f_\gamma(o, s)^{-1}}\right)+(1-y)\cdot\log\left({1+f_\gamma(o,s)}\right)].\label{eq:contrastive_loss}
\end{align}
 Consider function class $\cF_\gamma$ with finite elements with $\mathcal{F}_\gamma\subseteq(\mathcal{S}\times\mathcal{O}\mapsto\mathbb{R})$ serving as a set of candidate functions that approximates the ground-truth likelihood ratio $f^*(\cdot,\cdot)=\mathbb{O}(\cdot\wvert\cdot)/\mathcal{P}^-(\cdot)$ (see Lemma \ref{lem:targetcontrastive} for justification). Following this, the pretrained translator for the \reporter\ by the algorithm in \eqref{eq:contrastive_loss} is thus defined as $\mathbb{O}_{\hat{\gamma}}(\cdot\wvert\cdot)=f_{\hat{\gamma}}(\cdot,\cdot)\cdot\mathcal{P}^-(\cdot)$. More details are deferred to \S\ref{sec:contrastivepretrain}.
 
 \begin{remark}
 	In \eqref{eq:joint_distribution}, we assume that all pretraining data is generated from a joint distribution $\PP_\cD$, and then split for pretraining of LLM and \reporter. In practice, the pretraining dataset for the \reporter\ can consist of paired observation-state data collected from any arbitrary distribution, as long as (\romannumeral1) the LLM and \reporter\ ``speak" the same language, i.e., shared $\OO$, and (\romannumeral2) the coverage assumption can hold (see Assumption \ref{as:onlinecoverage}). 
 \end{remark}
 \begin{remark}
 	As an example, noise contrastive estimation \citep[NCE,][]{gutmann2010noise} is one of the most widely adopted objectives in contrastive representation learning. From the theoretical lens, to estimate unnormalized model $p_d$ with $x_i\overset{\rm iid}{\sim}p_d$, additional noise data is sampled from a reference distribution $p_n$ and then estimate by maximizing $\hat{\EE}[y\cdot\log(h_\gamma(x))+(1-y)\cdot\log(1-h_\gamma(x))]$ with $y=\ind(x\text{~is not noise})$ and $h^*(x)=p_d(x)/(p_d(x)+p_n(x))$. With slight modifications, we use a function class $\cF$ to approximate the ratio $p_d/p_n$ rather than the relative probability $h$ itself. In practice, the most commonly used contrastive training objectives are variations of NCE and originated from the NLP domain \citep{schiappa2023self} by sharing the same idea of \emph{minimizing the distance between the positive pair and maximizing the distance between the negative pairs}.
 \end{remark}

%% file: tex/main/perfect.tex
\section{LLM Planning via Bayesian Aggregated Imitation Learning}\label{sec:bail}
In this section, we first demonstrate that LLMs can conduct high-level planning through \ac{bail} in \S\ref{sec:BAIL}, leveraging the ICL ability of LLMs with the history-dependent prompts. However, depending solely on LLM's recommendations proves insufficient for achieving sample efficiency under the worst case (see Proposition \ref{prop:hardexample}). Following this, we propose a planning algorithm for  \planner\ in \S\ref{sec:planningalg}, leveraging LLMs for expert recommendations, in addition to an exploration strategy.

\subsection{Bayesian Aggregated Imitation Learning}\label{sec:BAIL}
In this subsection, we show that the LLM conducts high-level task planning via \ac{bail}, integrating both Bayesian model averaging \citep[BMA,][]{hoeting1999bayesian} during the online planning and imitation learning \citep[IL,][]{ross2010efficient} during the offline pretraining. Intuitively, pretrained over $\mathcal{D}_\mathtt{LLM}$, LLM approximates the conditional distribution $\mathtt{LLM}(\ell=\cdot\wvert S)=\PP_{\cD}(\ell=\cdot\wvert S)$, where $\PP_{\mathcal{D}}$ is the joint distribution in \eqref{eq:joint_distribution} and the randomness introduced by the latent variable is aggregated, i.e., $\PP_{\cD}(\ell=\cdot\wvert S)=\mathbb{E}_{z\sim\PP_{\cD}(\cdot|S)}\left[\PP_{\cD}(\ell=\cdot\wvert S,z)\right]$. Here, $\PP_{\cD}(\ell=\cdot\wvert S,z)$ can be viewed as a generating distribution with a known $z$ and is then aggregated over the posterior distribution $\PP_{\cD}(z=\cdot\wvert S)$, aligning with the form of BMA \citep{zhang2023and}. We temporarily consider the perfect setting.
\begin{definition}[Perfect Setting]
	We say the PAR system is perfectly pretrained if 
	(\romannumeral1) $\mathbb{O}_{\hat\gamma}(\cdot\wvert s)=\mathbb{O}(\cdot\wvert s)$ for all $s\in\mathcal{S}$,
	(\romannumeral2) $\mathtt{LLM}_{\hat\theta}(\cdot\wvert S_t)=\mathtt{LLM}(\cdot\wvert S_t)$ for all $S_t=(\ell_1,\dots,\ell_t)\in\mathfrak{L}^*$ with length $t\leq\bar{T}_{\rm p}$.
	\label{as:perfect}
\end{definition}
The assumption states that the \reporter\ and LLMs can report and predict with ground-truth distributions employed based on the joint distribution $\PP_\mathcal{D}$. During ICL, we invoke LLMs by history-dependent $\prompt_h^t=\mathcal{H}_t\cup\{\omega^t,\tau_{h}^t\}\in\mathfrak{L}^*$ for all $(h,t)\in[H]\times[T]$. Conditioned on latent variable $z$ and $\prompt_h^t$, the generating distribution is the optimal policy such that $\PP_{\cD}(\cdot\wvert\prompt_h^t,z)=\pi^*_{z,h}(\cdot\wvert \tau_{h}^t,\omega^t)$, which is independent of historical $\mathcal{H}_t$. In this sense, LLMs imitate expert policies during pretraining. The proposition below shows that LLMs conduct task planning via \ac{bail}.
\begin{proposition}[LLM Performs BAIL]
Assume that the pretraining data distribution is given by \eqref{eq:joint_distribution}. Under the perfect setting in Definition  \ref{as:perfect}, for all $(h,t)\in[H]\times[T]$, the LLM conducts task planning via \ac{bail}, following that
\begin{align*}
\pi_{h,\mathtt{LLM}}^t\left(\cdot\wvert\tau_{h}^t,\omega^t\right) = \sum_{z\in\mathcal{Z}}\pi^*_{z,h}\left(\cdot\wvert\tau_{h}^t,\omega^t\right) \cdot \mathbb{P}_{\cD}\left(z\wvert\prompt_h^t\right),
\end{align*}
where $\pi_{h,\mathtt{LLM}}^t$ denotes the LLM's policy and prompt is defined in \eqref{eq:prompt}.
\label{thm:BAIL}
\end{proposition}
\begin{proof}[Proof of Proposition \ref{thm:BAIL}.]
	Please refer to \S\ref{ap:bil} for a detailed proof.
\end{proof}

Proposition \ref{thm:BAIL} suggests that LLMs provide recommendations following a two-fold procedure: Firstly, LLMs compute the posterior belief of each latent variable $z\in\mathcal{Z}$ from $\prompt_h^t$. Secondly, LLMs aggregate the optimal policies over posterior probability and provide recommendations.

\subsection{LLM-Empowered Planning Algorithm}\label{sec:planningalg}
\begin{algorithm}[t] 
	\caption{Planning with PAR System - \planner} 
	\begin{algorithmic}[1]
		\renewcommand{\algorithmicrequire}{\textbf{Input:}}
		\Require Policy $\pi_{\mathtt{exp}}$ with $\eta\in(0,1)$, $c_\cZ>0$, and $|\cZ|\in\NN$.
		\renewcommand{\algorithmicrequire}{\textbf{Initialize:}}
		\Require $\mathcal{H}_0\leftarrow\{\}$, and $\epsilon\leftarrow(H\log(c_\mathcal{Z}|\mathcal{Z}|\sqrt{T})/T\eta)^{1/2}$. 
		\For{episode $t$ from $1$ to $T$}
		\State Receive the high-level task $\omega^t$ from the human user.
		\State Sample $\mathcal{I}_t\sim\text{Bernuolli}(\epsilon)$.
		\For{step $h$ from $1$ to $H$}
		\State Collect the observation $o_{h}^t$ from the \reporter.
		\State Set $\prompt_h^t\leftarrow\mathcal{H}_t\cup\{\omega^t,o_1^t,\dots,o_h^t\}$.
		\State Sample $g_{h,\mathtt{LLM}}^t\sim\mathtt{LLM}(\cdot\wvert\prompt_h^t)$ via prompting LLM.
		\State \textbf{If} $\mathcal{I}_t=1$ \textbf{then} $g_{h}^t\leftarrow g_{h,\mathtt{LLM}}^t$, \textbf{else} sample $g_{h}^t\sim\pi_{h,\mathtt{exp}}(\cdot\wvert \tau_{h}^t)$.
		\State Send the subgoal $g_h^t$ to the \actor.
		\EndFor
		\State Update $\mathcal{H}_{t+1}\leftarrow\mathcal{H}_{t}\cup\{\omega^t,\tau_{H}^t\}$.
		\EndFor	
	\end{algorithmic} 
	\label{alg:planner}
\end{algorithm}
Following the arguments above, we propose a planning algorithm for the \planner\ within a perfect PAR system. From a high level, the process of task planning is an implementation of policies from imitation learning \citep{ross2010efficient,ross2011reduction} with two key distinctions: (\romannumeral1) \planner\ collaborates with LLM, a ``nascent" expert that learns the hidden intricacies of the external world from updating prompts; (\romannumeral2) different from behavior cloning or inverse RL, \planner\ does not aim to comprehend LLM's behaviors. Instead, the imitation is accomplished during the offline pretraining, and \planner\ shall selectively adhere to LLM's suggestions during online planning. Next, we show that task planning solely guided by LLMs fails to achieve sample efficiency in the worst case.
\begin{proposition}[Hard-to-Distinguish Example] 
	\label{prop:hardexample}
	Suppose that Definition \ref{as:perfect} holds. Given any $T\in\mathbb{N}$, there exists an HMDP and specific latent variable $z\in\mathcal{Z}$ such that if \planner\ strictly follows LLM's recommended policies in Proposition  \ref{thm:BAIL}, it holds that ${\rm Reg}_z(T)\geq 0.5T\cdot(1-1/|\mathcal{Z}|)^T$.
\end{proposition}
\begin{proof}[Proof of Proposition \ref{prop:hardexample}.]
	Please refer to \S\ref{ap:hard_exp} for a detailed proof.
\end{proof}

Proposition \ref{prop:hardexample} indicates that relying solely on LLMs for task planning can result in a suboptimal $\Omega(T)$ regret in the worst case when $|Z|=T$. Thus, additional exploration is essential to discern the latent information about the external world, a parallel to the practical implementations in latent imitation learning \citep{edwards2019imitating,kidambi2021mobile} and LLM-based reasoning \citep{hao2023reasoning,nottingham2023embodied}. In practice, while the language model can guide achieving a goal, it's important to note that \emph{this guidance is not grounded in real-world observations}. Thus, as pointed out by \cite{grigsby2023amago}, the information provided in narratives might be arbitrarily wrong, which highlights the need for exploration to \emph{navigate new environments effectively}. Similar to $\epsilon$-greedy algorithms \citep{tokic2011value,dann2022guarantees}, we provide a simple but efficient algorithm for LLM-empowered task planning. Algorithm \ref{alg:planner} gives the pseudocode. In each episode, the \planner\ performs two main steps:
\begin{itemize}[leftmargin=*]
	\setlength{\itemsep}{-1pt}
	\item[-] \textbf{Policy Decision}\ ($\mathtt{Line\ 5}$): Randomly decide whether to execute the exploration policy $\pi_\mathtt{exp}$ or follow the LLM's recommendations within this episode with probability $\epsilon$.
	\item[-] \textbf{Planning with LLMs}\ ($\mathtt{Line\ 7-10}$): If \planner\ decides to follow the LLM's recommendations, the subgoal is obtained by prompting LLMs with $\prompt_h^t=\mathcal{H}_t\cup\{\omega^t,\tau_{h}^t\}$, equivalently sampling from $\mathtt{LLM}(\cdot\wvert\prompt_h^t)$. Otherwise, the \planner\ takes sub-goal from $\pi_{h,\mathtt{exp}}(\cdot\wvert \tau_h^t)$.
\end{itemize}
In conventional $\epsilon$-greedy algorithms, explorations are taken uniformly over the action space $\mathcal{G}$, i.e., $\pi_\mathtt{exp}={\rm Unif}_\mathcal{G}$. Recent work has extended it to a collection of distributions (e.g., softmax, Gaussian noise) for function approximation \citep{dann2022guarantees}. Following this, we instead consider a broader class of exploration strategies that satisfy the $\eta$-distinguishability property below.
\begin{definition}[$\eta$-distinguishability]
 \label{def:iden}
We say an exploration policy $\pi_{\mathtt{exp}}=\{\pi_{h,\mathtt{exp}}\}_{h\in[H]}$ is $\eta$-distinguishable if there exists an absolute constant $\eta>0$ such that for all $z,z'\in\mathcal{Z}$ with $z\neq z'$, it holds that $D_{\rm H}^2\left(\mathbb{P}^{\pi_\mathtt{exp}}_{z}(\tau_{H}),\mathbb{P}^{\pi_\mathtt{exp}}_{z'}\left(\tau_{H}\right)\right)\geq\eta$.
 \end{definition}
The $\eta$-distinguishability implies the existence of exploration policy  $\pi_\mathtt{exp}$ that could well-distinguish the models with an $\eta$-gap in Hellinger distance concerning the distribution of whole trajectory, which also impose condition over the model seperation. 
Next, we introduce the assumption over priori.

\begin{assumption}[Prior coverage]
	There exists a constant $c_\mathcal{Z}>0$ such that $\sup_{z,z'}\frac{\mathcal{P}_\mathcal{Z}(z')}{\mathcal{P}_\mathcal{Z}(z)}\leq c_\mathcal{Z}$.
	\label{as:coverage}
\end{assumption}
\vspace{-2mm}
The assumption asserts a bounded ratio of priors, implying that each $z \in \mathcal{Z}$ has a non-negligible prior probability. The assumption is intuitive, as a negligible priori suggests such a scenario almost surely does not occur, rendering the planning in such scenarios unnecessary. Now, we are ready to present the main theorem of the \planner\ under perfect setting.
 \begin{theorem}[Regret under Perfect Setting]
 \label{thm:regret}
Suppose that Definition \ref{as:perfect} and Assumption \ref{as:coverage} hold. Given an $\eta$-distinguishable exploration policy $\pi_{\mathtt{exp}}$ and $T\leq T_{\rm p}$, Algorithm \ref{alg:planner} ensures
\begin{align*}
{\rm Reg}_z(T)&:=\sum_{t=1}^T\EE_{\cH_t}\left[\mathcal{J}_z(\pi_z^*,\omega^t)-\mathcal{J}_z(\pi^t,\omega^t)\right]\leq\tilde{\mathcal{O}}\left(H^\frac{3}{2}\sqrt{T/\eta\cdot \log(c_\mathcal{Z}|\mathcal{Z}|\sqrt{T})}\right),
\end{align*}
for any $z\in\mathcal{Z}$ and $\{\omega^t\}_{t\in[T]}$, if the \planner\ explores with probability $\epsilon=(H\log(c_\mathcal{Z}|\mathcal{Z}|\sqrt{T})/T\eta)^{1/2}$.
 \end{theorem}
 \begin{proof}[Proof of Theorem \ref{thm:regret}.]
	Please refer to \S\ref{ap:perfectplanning} for a detailed proof.
\end{proof}
Theorem \ref{thm:regret} states that the \planner's algorithm can attain a $\tilde{\mathcal{O}}(\sqrt{T})$ regret for planning facilitated by LLMs. The multiplicative factor of the regret depends on the horizon of the interactive process $H$, the reciprocal of coverage rate $\eta$ in Definition \ref{def:iden}, and the logarithmic term $\log\left(c_\mathcal{Z}|\mathcal{Z}|\right)$ including both the cardinality of candidate models and the prior coverage in Assumption \ref{as:coverage}, which jointly characterizes the complexity of the physical world.
\begin{remark}
	\citet{lee2023supervised} has demonstrated that a perfect decision-pretrained transformer, similar to the role of LLM in ours, can attain a $\tilde{\cO}(H^\frac{3}{2}\sqrt{T})$ Bayesian regret, i.e., $\EE_{z\sim\cP_\cZ}[{\rm Reg}(T)]$, via ICL. In comparison, we focus on a more challenging setting that aims to control the frequentist regret, which is closer to applications, and attain a comparable result with additional exploration. 
\end{remark}

%% file: tex/main/pretrain.tex
\section{Performance under Practical Setting}\label{sec:pretrain}
\subsection{Pretraining Large Language Model}
\label{sec:llmpretrain}
 In this subsection, we elaborate on the pretraining of LLMs using transformer architecture. We employ a supervised learning algorithm minimizing the cross-entropy loss, i.e., $\hat{\theta}={\rm argmin}_{\theta\in\Theta}\ \mathcal{L}_\mathrm{CE}(\theta;\mathcal{D}_\mathtt{LLM})$, as detailed in \eqref{eq:contrastive_loss}. Following this, the population risk follows  that
 \begin{align}
 	&\mathcal{R}_\mathrm{CE}(\theta;\mathcal{D}_\mathtt{LLM})=\mathbb{E}_t[\mathbb{E}_{S_t}[D_{\rm KL}(\mathtt{LLM}(\cdot|S_t)\wVert\mathtt{LLM}_\theta(\cdot|S_t))+{\rm Ent}(\mathtt{LLM}(\cdot|S_t))]], \notag
 \end{align} 
where $t\sim{\rm Unif}({[\bar{T}_{\rm p}]})$, $S_t$ is distributed as the pretraining distribution, and ${\rm Ent}(\mathbb{P})=\mathbb{E}_{x\sim\mathbb{P}}[\log\mathbb{P}(x)]$ is the Shannon entropy. As the minimum is achieved at $\mathtt{LLM}_\theta(\cdot|S)=\mathtt{LLM}(\cdot|S)$, estimated $\mathtt{LLM}_{\hat{\theta}}$ and $\mathtt{LLM}$ are expected to converge under the algorithm with a sufficiently large dataset. Specifically, our design adopts a transformer function class to stay consistent with the architectural choices of language models like BERT and GPT. Specifically, a transformer model comprises $D$ sub-modules, with each sub-module incorporating a Multi-Head Attention (MHA) mechanism and a fully connected Feed-Forward (FF) layer. See \S\ref{ap:transformer} for further details, and we specify two widely adopted assumptions in the theoretical analysis of LLM pretraining \citep{wies2023learnability,zhang2023and}.

\begin{assumption}[Boundedness]
	For all $z\in\mathcal{Z}$ and $t\leq\bar{T}_{\rm p}$, there exists a constant $R>0$ such that all $S_t=(\ell_1,\dots,\ell_t)\sim\PP_{\cD}(\cdot\wvert z)$  with $S_t\in\mathfrak{L}^*$ satisfies that $\Vert S_t\Vert_{2,\infty}\leq R$ almost surely.
	\label{as:boundedness}
\end{assumption}

The boundedness assumption requires that the $\ell_2$-norm of the magnitude of each token is upper bounded by $R>0$, and such an assumption holds in most settings.

\begin{assumption}[Ambiguity] 
For all latent variable $z\in\mathcal{Z}$, there exists a constant $c_0>0$ such that for all $\ell_{t+1}\in\mathfrak{L}$ and $S_t=(\ell_1,\dots,\ell_t)\in\mathfrak{L}^*$ with length $t<\bar{T}_{\rm p}$, it holds $\PP_{\cD}(\ell_{t+1}\wvert S_t,z)\geq c_0$.	
\label{as:ambiguity}
\end{assumption}

The ambiguity assumption states that the generating distribution is lower bounded, and the assumption is grounded in reasoning as there may be multiple plausible choices for the subsequent words to convey the same meaning. Next, we present the performance of the pretrained LLMs.
\begin{theorem}[\citet{zhang2023and}] Suppose that Assumptions \ref{as:boundedness} and \ref{as:ambiguity} hold. With probability at least $1-\delta$, the pretrained model $\mathtt{LLM}_{\hat{\theta}}$ by the algorithm in \eqref{eq:crossentropyloss} satisfies that
\begin{align*}
&\bar{\mathbb{E}}_{\mathcal{D}_\mathtt{LLM}}\left[D_{\rm TV}\left(\mathtt{LLM}(\cdot\wvert S), \mathtt{LLM}_{\widehat{\theta}}(\cdot \wvert S)\right)\right] \\
&\quad{\leq \cO\biggl(\inf _{\theta^{*} \in \Theta} \sqrt{\bar{\mathbb{E}}_{\mathcal{D}_\mathtt{LLM}}\left[ D_\mathrm{KL}\left(\mathtt{LLM}(\cdot \wvert S), \mathtt{LLM}_{\theta^{*}}(\cdot \wvert S)\right)\right]}}\\
&\qquad+\frac{t_{\rm mix}^{1/4} \log \frac{1}{\delta}}{(N_{\mathrm{p}} \bar{T}_{\mathrm{p}})^{1 / 4}}+\sqrt{\frac{t_{\rm mix}}{N_{\mathrm{p}} \bar{T}_{\mathrm{p}}}}\biggl(\bar{D} \log \bigl(1+\bar{B}N_{\mathrm{p}} \bar{T}_{\mathrm{p}}\bigr)+\log\frac{1}{\delta}\biggl)\biggl),
\end{align*}
where $\bar{B}$ and $\bar{D}$ features the tranformer's architecture, $t_{\rm mix}$ denotes the mixing time of Markov chain $\{S_t\}_{t\in[T]}$\footnote{Note that $\{S^n_t\}_{t\in[T]}$ directly satisfies Markov property since $S_t^n=(\ell_1^n,\dots,\ell_t^n)$ and thus $S^n_i\subseteq S^n_{t}$ for all $i\leq t$.}, and $N_{\mathrm{p}} \bar{T}_{\mathrm{p}}$ is the size of dataset $\mathcal{D}_\mathtt{LLM}$. See \S\ref{ap:transformer} for detailed structure and definitions.
\label{thm:llmpretrain}
\end{theorem}
\begin{proof}[Proof of Theorem \ref{thm:llmpretrain}.]
	Please refer to Theorem 5.3 in \cite{zhang2023and} for a detailed proof.
\end{proof}

Theorem \ref{thm:llmpretrain} states that the total variation of the conditional distribution, with expectation taken over the average distribution of context $S$ in $\mathcal{D}_\mathtt{LLM}$ (see Table \ref{tab:notation} for definition), converges at $\mathcal{O}\left((N_{\mathrm{p}}\bar{T}_{\mathrm{p}})^{-1/2}\right)$. Note that the first two terms represent the approximation error and deep neural networks act as a universal approximator \citep{yarotsky2017error} such that the error would vanish with  increasing volume of network \citep[Proposition C.4,][]{zhang2023and}. For notational simplicity, we denote the right-hand side of theorem as $\Delta_\mathtt{LLM}(N_{\rm p},T_{\rm p},H,\delta)$.

\subsection{Pretraining Observation-to-Language Translator}
\label{sec:contrastivepretrain}
In this subsection, we focus on the pretraining of observation-to-language translators under a self-supervised learning architecture using the contrastive loss. Consider the function class 
$$
\mathcal{F}_\gamma=\{f_\gamma(\cdot,\cdot):\gamma\in\Gamma,\|f_\gamma\|_\infty\leq B_\mathcal{F},\|1/f_\gamma\|_\infty\leq B^-_\mathcal{F}\},
$$ 
with finite elements, and the contrastive loss $\mathcal{L}_{\rm CT}(\gamma;\mathcal{D}_\mathtt{Rep})$ in \eqref{eq:contrastive_loss} is then defined over $\mathcal{F}_\gamma$. Note that the contrastive loss can be equivalently written as the negative log-likelihood loss of a binary discriminator, following that $\mathcal{L}_{\rm CT}(\gamma;\mathcal{D}_\mathtt{Rep})=\mathbb{\hat E}_{\mathcal{D}_\mathtt{Rep}}\left[-\mathbb{D}_\gamma(y\wvert o,s)\right]$, where we define
\begin{align}
&\mathbb{D}_\gamma(y\wvert o,s):=\left(\frac{f_\gamma(o, s)}{1+f_\gamma(o, s)}\right)^y\left(\frac{1}{1+f_\gamma(o, s)}\right)^{1-y}.
\label{eq:implieddistribution}
\end{align}
Based on \eqref{eq:implieddistribution} and the algorithm $\hat{\gamma}={\rm argmin}_{\gamma\in\Gamma}\ \mathcal{L}_\mathrm{CT}(\gamma;\mathcal{D}_\mathtt{Rep})$, the population risk follows that
\begin{align}
&\mathcal{R}_{\rm CT}(\gamma;\mathcal{D}_\mathtt{Rep})=\mathbb{E}\left[D_{\rm KL}\left(\mathbb{D}_\gamma(\cdot| o,s)\wVert\DD(\cdot| o,s)\right)+{\rm Ent}(\DD(\cdot| o,s))\right].
\end{align}
As the minimum is attained at $\mathbb{D}_\gamma(\cdot\wvert o,s)=\DD(\cdot\wvert o,s)$, where $\DD(\cdot\wvert o,s):=\PP_{\cC}(\cdot\wvert o,s)$ is the distribution of the label conditioned on the $(o,s)$ pair in contrastive data collection, estimated $\mathbb{D}_{\hat{\gamma}}(\cdot\wvert o,s)$ and $\DD(\cdot\wvert o,s)$ are expected to converge, and thus the learning target is the ground-truth likelihood ratio $f^*(o,s)=\mathbb{O}(o\wvert s)/\mathcal{P}^-(o)$ (see Lemma \ref{lem:targetcontrastive}). Below, we assume the learning target $f^*(o,s)$ is realizable in $\cF_\gamma$, which is standard in literature \citep{qiu2022contrastive}.
\begin{assumption}[Realizability]
	Given a designated negative sampling distribution $\mathcal{P}^-\in\Delta(\cO)$, there exists $\gamma^*\in\Gamma$ such that $f_{\gamma^*}(o,s)=\mathbb{O}(o\wvert s)/\mathcal{P}^-(o)$ for all $(o,s)\in\mathcal{O}\times\mathcal{S}$.
	\label{as:realizability}
\end{assumption}
Next we present the performance of the pretrained translator.
\begin{theorem}[Pretrained Translator]
	\label{thm:translatorpretrain}
	Suppose that Assumption \ref{as:realizability} holds. With probability at least $1-\delta$, the pretrained model $\mathbb{O}_{\hat{\gamma}}$ by the algorithm in \eqref{eq:implieddistribution} satisfies that
	\begin{align*}
	&\bar{\mathbb{E}}_{\mathcal{D}_\mathtt{Rep}}\left[D_{\rm TV}\left(\mathbb{O}(\cdot\wvert s),\mathbb{O}_{\hat{\gamma}}(\cdot\wvert s)\right)\right]\leq\mathcal{O}\left(\frac{B_\mathcal{F}(B^-_\mathcal{F})^{1/2}}{(N_{\rm p}T_{\rm p}H)^{1/2}}\sqrt{\log(N_{\rm p}T_{\rm p}H|\mathcal{F}_\gamma|/\delta)}\right),
	\end{align*}
	where let $\mathbb{O}_{\hat{\gamma}}(\cdot\wvert s)=f_{\hat{\gamma}}(\cdot\wvert s)\cdot\mathcal{P}^-(\cdot)$ and $|\mathcal{F}_\gamma|$ denotes the cardinality of  the function class $\mathcal{F}_\gamma$.
\end{theorem}
 \begin{proof}[Proof of Theorem \ref{thm:translatorpretrain}.]
	Please refer to \S\ref{ap:contrastive} for a detailed proof.
\end{proof}
Theorem \ref{thm:translatorpretrain} posits that the average expectation of the total variation of the translation distribution regarding $\cD_\mathtt{Rep}$ converges at $\mathcal{O}\left((N_{\mathrm{p}}{T}_{\mathrm{p}})^{-1/2}\right)$. For notational simplicity, write the right-hand side of the theorem as $\Delta_\mathtt{Rep}(N_{\rm p},T_{\rm p},H,\delta)$. Furthermore, the algorithm also ensures a more stringent convergence guarantee concerning $\chi^2$-divergence: $\bar{\mathbb{E}}_{\mathcal{D}_\mathtt{Rep}}[\chi^2(\mathbb{O}_{\hat{\gamma}}(\cdot\wvert s)\wVert\mathbb{O}(\cdot\wvert s))]\leq\Delta_\mathtt{Rep}(N_{\rm p},T_{\rm p},H,\delta)^2$. 

\subsection{Performance with Pretrained PAR System}\label{sec:imperfectregret}
In this subsection, we delve into the performance of task planning with pretrained PAR system. We first introduce the online coverage assumption, which pertains to the distribution of online planning trajectories under practical scenarios and trajectories in pretraining datasets.
\begin{assumption}[Coverage]
\label{as:onlinecoverage}
	There exists absolute constants $\lambda_S>0$ and $\lambda_R>0$ such that for all latent variable $z\in\mathcal{Z}$, $t<\bar{T}_{\rm p}$ and policy sequence $\{\pi^{i}\}_{i\leq\lceil t/2H\rceil}$ from the \planner, it holds that (\romannumeral1) $\prod_{i=1}^{\lceil t/2H\rceil}\mathbb{\hat{P}}_z^{\pi_i}(\tilde{S_i})\leq\lambda_S\cdot\bar{\mathbb{P}}_{\mathcal{D}_\mathtt{LLM}}(S_t)$ for all ordered sequence $S_t=(\tilde{S}_i)_{i\leq\lceil t/2H\rceil}\in\mathfrak{L}^*$, where $|\tilde{S}_i|=2H$ for all $k<\lceil t/2H\rceil$, and (\romannumeral2) $\bar{\mathbb{P}}_{\mathcal{D}_\mathtt{Rep}}(s)\geq \lambda_R$ for all state $s\in\mathcal{S}$.
\end{assumption}
Here, $\mathbb{\hat{P}}_z$ denotes the distribution of the dynamic system with the pretrained translator. The assumption asserts that (\romannumeral1)  distribution of the ICL prompts induced by policy sequences $\{\pi^{i}\}_{i\leq\lceil t/2H\rceil}$ from the \planner\ under practical scenarios is covered by the pretraining data, where $\lceil t/2H\rceil$ denotes the number of  episodes described in $S_t$. (\romannumeral2) all states $s\in\mathcal{S}$ are covered by the average distribution of the \reporter's pretraining dataset. Similar conditions are  adopted in ICL analysis \citep{zhang2023and}, decision pretrained transformer \citep{lee2023supervised,lin2023transformers} and offline RL \citep{munos2005error,duan2020minimax}. Intuitively, LLM and reporter cannot precisely plan or translate beyond the support of the pretraining dataset. These conditions are achievable if an explorative behavior strategy $\pi^b$ is deployed with a sufficiently large $N_{\rm p}$ when collecting data. We then present the main theorem regarding the practical performance.

\begin{theorem}[Regret under Practical Setting]
Suppose that Assumptions \ref{as:coverage}, \ref{as:boundedness}, \ref{as:ambiguity}, \ref{as:realizability} and \ref{as:onlinecoverage}. Given an $\eta$-distinguishable exploration policy $\pi_{\mathtt{exp}}$ and $T\leq T_{\rm p}$, under the practical setting, the \planner's algorithm in Algorithm \ref{alg:planner} ensures that
\begin{align*}
{\rm Reg}_z(T)&\leq\tilde{\mathcal{O}}\Big(H^\frac{3}{2}\sqrt{T/\eta\cdot\log(c_\mathcal{Z}|\mathcal{Z}|\sqrt{T})}+H^2T\cdot\Delta_{\rm p}(N_{\rm p},T_{\rm p},H,1\sqrt{T},\xi)\Big),
\end{align*}
for any $z\in\mathcal{Z}$ and $\{\omega_t\}_{t\in[T]}$. The cumulative pretraining error of PAR system follows that
\begin{align*}
	\Delta_{\rm p}&(N_{\rm p},T_{\rm p},H,\delta,\xi)=(\eta\lambda_R)^{-1}\cdot\Delta_\mathtt{Rep}(N_{\rm p},T_{\rm p},H,\delta)^2\\
	&+2\lambda_R^{-1}\cdot\Delta_\mathtt{Rep}(N_{\rm p},T_{\rm p},H,\delta)+\lambda_S\cdot\Delta_\mathtt{LLM}(N_{\rm p},T_{\rm p},H,\delta).
\end{align*} 
where $\xi=(\eta,\lambda_S,\lambda_R)$ are defined in Definition \ref{def:iden} and Assumption \ref{as:onlinecoverage}, and pretraining errors $\Delta_\mathtt{LLM}$ and $\Delta_\mathtt{Rep}$ are  defined in Theorem \ref{thm:llmpretrain} and Theorem \ref{thm:translatorpretrain}. Under the practical setting, \planner\ should explore with probability $\epsilon=\big(H\log\big(c_\mathcal{Z}|\mathcal{Z}|\sqrt{T}\big)/T\eta\big)^{1/2}+H(\eta\lambda_{\min})^{-1}\Delta_\mathtt{Rep}(N_{\rm p},T_{\rm p},H,1/\sqrt{T})^2$.
\label{thm:practical regret}
\end{theorem}
 \begin{proof}[Proof of Theorem \ref{thm:practical regret}.]
	Please refer to \S\ref{ap:practicalregret} for a detailed proof.
\end{proof}
Theorem \ref{thm:practical regret} reveals that, in comparison to perfect scenario, the \planner\ can achieve an approximate $\tilde{\mathcal{O}}(\sqrt{T})$ regret, but incorporating an additional pretraining error term that could diminishe with an increase in the volume of pretraining data. Besides, it further underscores the necessity of exploration, where the \planner\ should explore with an additional $H(\eta\lambda_{\min})^{-1}\Delta_\mathtt{Rep}(N_{\rm p},T_{\rm p},H,\delta)^2$  to handle the mismatch between the ground-truth and the pretrained environment.
\begin{remark}
	The challenge of establishing a performance guarantee in a practical setting arises from the mismatch between the ground-truth environment and the pretrained one, leading to a distributional shift in posterior probability. Besides, \ac{bail} is realized through a pretrained LLM, which introduces its pretraining error inaddition. In response, we propose a novel regret decomposition and provide the convergence rate of posterior probability with bounded pretraining errors, distinguishing ours from the previous results in \citet{lee2023supervised,liu2023reason}.
\end{remark}
\paragraph{Extentions.} We also present two extensions. In \S\ref{sec:wm}, we discuss the design of \planner\ by taking LLMs as World Model (WM). Here, the \planner\ prompts the LLM to predict the next observation rather than subgoals, alleviating the reliance on expert knowledge. By leveraging model-based RL methods like Monte Carlo Tree Search (MCTS) and Real-Time Dynamic Programming (RTDP), the \planner\ utilizes the LLM-simulated environment to optimize its strategies based on the contextual information. As shown in Proposition \ref{thm:bawm}, the simulated world model via ICL conforms to Bayesian Aggregated World Model (BAWM). Hence, the LLM \planner\ achieves a regret at rate of ${\rm Reg}_z(T)\leq\tilde{\mathcal{O}}(H\sqrt{T/\eta})+H^2T\cdot\Delta_{\rm p,wm}$ under practical setting with regularity conditions (see Corollary \ref{cor:bawm_regret}). Besides, we extend the results in \S\ref{sec:bail} to accommodate the scenario of multi-agent collaboration, i.e., $K$ \actor s. In \S\ref{sec:multiagent}, we formulate the probelm as a cooperative hierarchical Markov Game (HMG) and establish a theoretical guarantee of ${\rm Reg}_z(T)\leq\tilde{\mathcal{O}}(\sqrt{H^3TK/\eta})$ under the perfect setting (see Corollary \ref{corol:perfectregret}). These  two extention correponds to recent works on LLM planning as world model \citep[e.g.,][]{hu2023language} and muti-agent collaboration of LLM Agents \citep[e.g.,][]{mandi2023roco}.

\section{Conclusion}
In this work, we embedded the LLM-empowered decision-making problem into a hierarchical RL framework named PAR system where at the high level, the LLM \planner\ decomposes the user-specified task into subgoals, and at the low level, the \actor (s) translate the linguistic subgoals into physical realizations while also providing feedbacks for augmenting the planning process through a trained reporter. 
Under the perfect setting, we characterize the BAIL nature of the LLM-aided planning pipeline and the nessecity of exploration even under expert guidance. We also shed light on how the training errors of both LLM and reporter enter the ICL error under practical scenarios.

%% file: tex/appendix/appendix_background.tex
\normalsize
\section{Additional Background}

In this appendix, we present the additional background knowledge that are omitted due to the space limit. We first lay out the notations used in this paper.

\paragraph{Notations.} For some $n\in\mathbb{N}^+$, let $[n]=\{1,\dots,n\}$. Denote $\Delta(\mathcal{X})$ as the probability simplex over $\mathcal{X}$. Consider two non-negative sequence $\{a_n\}_{n\geq0}$ and $\{b_n\}_{n\geq0}$, if $\limsup a_n/b_n<\infty$, we write it as $a_n=\mathcal{O}(b_n)$ and use $\tilde{\mathcal{O}}$ to omit logarithmic terms. Else if $\liminf a_n/b_n<\infty$, we write $a_n=\Omega(b_n)$. For continuum $\mathcal{S}$, denote $|\mathcal{S}|$ as the cardinality. For matrix $X\in\RR^{m\times n}$, the $\ell_{p,q}$-norm is defined as $\|X\|_{p,q}=(\sum_{i=1}^n\|X_{:,i}\|_p^q)^{1/q}$, where $X_{:,i}$ denotes the $i$-th column of $X$.

\begin{table}[!h]
\caption{Table of Notations.}
\vspace{0.2cm}
\centering
\renewcommand*{\arraystretch}{1.2}
\begin{footnotesize}
\begin{tabular}{ >{\centering\arraybackslash}m{2.5cm} | >{\centering\arraybackslash}m{13cm} } 
\toprule\hline
Notation & Meaning \\ 
\hline
$\mathcal{J}_{z}(\cdot,\cdot)$, $\pi_z^*(\cdot)$ & value function and optimal policy $\pi_z^*(\cdot):=\argmax_\pi \cJ(\pi,\cdot)$ concerning ground-truth $\OO$\\

$\hat{\cJ}_{z}(\cdot,\cdot)$, $\hat{\pi}_z^*(\cdot)$ & value function and optimal policy $\hat{\pi}_z^*(\cdot):=\argmax_\pi \hat{\cJ}(\pi,\cdot)$ concerning pretrained $\OO_{\hat{\gamma}}$\\

$\PP_\cD(\cdot)$, $\PP_{\cC}(\cdot)$ &  probability induced by the distribution of joint and contrastive data collection  \\ 

$\pi^t_{h,\mathtt{LLM}}$, $\hat{\pi}^t_{h,\mathtt{LLM}}$ & $\pi^t_{h,\mathtt{LLM}}(\cdot\wvert \tau_{h}^t,\omega^t):=\mathtt{LLM}(\cdot\wvert\prompt_h^t)$ and $\hat{\pi}^t_{h,\mathtt{LLM}}(\cdot\wvert \tau_{h}^t,\omega^t):=\mathtt{LLM}_{\hat\theta}(\cdot\wvert\prompt_h^t)$ at step $h$\\

$\PP_z(\cdot)$, $\hat{\PP}_z(\cdot)$ & probability under environment featured by $z$, ground-truth $\OO$ or pretrained $\OO_{\hat{\gamma}}$  \\

$\PP_z^{\pi}(\cdot)$, $\hat{\PP}_z^{\pi}(\cdot)$ & probability under environment featured by $z$, policy $\pi$, ground-truth $\OO$ or pretrained $\OO_{\hat{\gamma}}$  \\


$\cP_\Omega(\cdot)$, $\cP_\cZ(\cdot)$ & prior distributions of high-level tasks and latent variables\\
\hline
$\breve{\tau}_{h/t}^i$ & $\breve{\tau}_{h/t}^i={\tau}_H$ for all $i<t$ and $\breve{\tau}_{h/t}^{t}={\tau}_{h}$\\

$\PP_{z}(\cdot|\cdot,\tdo\cdot)$ & $\PP_{z}(\cdot\wvert o_1,\tdo g_{1:h-1}):=\int_{o_{2:h-1}}\prod_{h'=1}^{h-1}\PP_{z}\left(o_{h'+1}\wvert(o,g)_{1:h'}\right)\rd o_{2:h-1}$\\

$\PP_{\mathtt{LLM}}^t(\cdot|\cdot,\tdo\cdot)$ & $\PP_{\mathtt{LLM}}^t(\cdot\wvert o_1,\tdo g_{1:h-1}):=\int_{o_{2:h-1}}\prod_{h'=1}^{h-1}\PP_{\cD}\left(o_{h'+1}\wvert(o,g)_{1:h'},\cH_t\right)\rd o_{2:h-1}$\\

$\mathcal{\hat{J}}_{t,\mathtt{LLM}}(\cdot,\cdot),\hat{\pi}_{\mathtt{LLM}}^{t,*}(\cdot)$ & value function of environment simulated by $\mathtt{LLM}_{\hat{\theta}}$ and $\hat{\pi}_{\mathtt{LLM}}^{t,*}(\cdot):=\argmax_\pi \mathcal{\hat{J}}_{t,\mathtt{LLM}}(\pi,\cdot)$\\

$\mathcal{J}_{t,\mathtt{LLM}}(\cdot,\cdot),\pi_{\mathtt{LLM}}^{t,*}(\cdot)$ & value function of environment simulated by perfect $\mathtt{LLM}$ and ${\pi}_{\mathtt{LLM}}^{t,*}(\cdot):=\argmax_\pi \mathcal{{J}}_{t,\mathtt{LLM}}(\pi,\cdot)$\\

$\PP_{\mathtt{LLM}}^t(\cdot)$, $\hat{\PP}_{\mathtt{LLM}}^t(\cdot)$ & probability of environment simulated by perfect $\mathtt{LLM}$ or pretrained $\mathtt{LLM}_{\hat{\theta}}$ with $\cH_t$\\

\hline

$D_{\rm TV}(P,Q)$ & total variation distance, $D_{\rm TV}(P,Q):=1/2\cdot\mathbb{E}_{x\sim P}[|{\rm d}Q(x)/{\rm d}P(x)-1|]$\\

$D_{\rm H}^2(P,Q)$ & Helliger distance, $D^2_{\rm H}(P,Q):=1/2\cdot\mathbb{E}_{x\sim P}\big[\big(\sqrt{{\rm d}Q(x)/{\rm d}P(x)}-1\big)^2\big]$\\

$D_{\rm KL}(P,Q)$ & KL divergence, $D_{\rm KL}(P\|Q):=\mathbb{E}_{x\sim P}[\log \rd P(x)/\rd Q(x)]$\\

$\chi^2(P,Q)$ & $\chi^2$-divergence, $\chi^2(P\|Q):=\mathbb{E}_{x\sim P}[(\rd Q(x)/\rd P(x)-1)^2]$\\

$\hat{\mathbb{E}}_\mathcal{D}[f]$ & $\bar{\EE}[f]:=1/n\cdot\sum_{t=1}^nf(\ell_t)$ given dataset $\cD=\{\ell_t\}_{t\in[n]}$\\

$\bar{\PP}_\mathcal{D}(\cdot)$, $\bar{\mathbb{E}}_\mathcal{D}[f]$ & $\bar{\PP}_{\mathcal{D}}(\cdot):=\sum_{n=1}^{N}\sum_{t=0}^{T-1}\PP_\cD(\cdot|\ell_{1:t}^n)/NT$ and $\bar{\EE}[f]:=\EE_{\ell\sim\bar{\PP}_{\mathcal{D}}}[f(\ell)]$ given $\cD=\{\ell_{1:T}^n\}_{n\in[N]}$\\

\hline \bottomrule
\end{tabular}
\end{footnotesize}
\label{tab:notation}
\end{table}

\subsection{Hierarchical Markov Decision Process} \label{sec:HMDP}
In this subsection, we present a formalized definition of the HMDP model introduced in \S\ref{sec:PAR}.
\paragraph{Low-level \ac{mdp}.}
Define $\mathcal{G}$ as the space of high-level actions. For fixed $g\in\mathcal{G}$ and high-level step $h\in[H]$, the low-level \ac{mdp} is defined as $\mathcal{M}_h(g) = (\mathcal{S}, \mathcal{A}, H_a,\TT_h, \bar{r}_g)$, where $\mathcal{S}$ is the state space, $\mathcal{A}$ is the low-level action space, $H_a$ is the number of steps, $\TT_h=\{\TT_{h,\bar{h}}\}_{\bar{h}\in[H_a]}$ is the transition kernel, and $\bar{r}=\{\bar{r}_{\bar{h}}\}_{\bar{h}\in[H_a]}$ is the reward function with $\bar{r}_{\bar{h}}:\cS\times\cA\times\cG\mapsto\RR$. The low-level agent follows policy $\mu=\{\mu_g\}_{g\in\cG}$, where $\mu_g=\{\mu_{\bar{h}}\}_{\bar{h}\in[H_a]}$ and $\mu_{\bar{h}}:\cS\times\cG\mapsto\Delta(\cA)$.
\paragraph{High-level \ac{pomdp}.} 	
Define $\Omega$ be the space of disclosed variables, and we write $z=(\TT,\mu)$ to feature the low-level environment. Each low-level episode corresponds to a single high-level action. Given fixed pair $(z,\omega) \in\mathcal{Z}\times\Omega$, the \ac{pomdp} is characterized by $\mathcal{W}(z,\omega)=(\mathcal{S},\mathcal{O},\mathcal{G},H,\mathbb{P}_z,r_\omega)$, where $\mathcal{O}$ is the observation space, $\mathbb{O}=\{\mathbb{O}_h\}_{h\in[H]}$ is the emission distribution with $\OO_h:\cO\mapsto\Delta(\cS)$, $r= \{r_{h}\}_{h\in[H]}$ is the reward function with $r_h:\cO\times\Omega\mapsto\RR$, and $\mathbb{P}_{z}=\{\mathbb{P}_{z,h}\}_{h\in[H]}$ is the high-level transition kernel following that
\begin{align*}
\PP_{z,h}(s'\wvert s,g)=\PP\big(\bar{s}_{h,H_a+1}=s'\wvert\bar{s}_{h,1}=s,a_{h,1:\bar{h}}\sim\mu_g,\bar{s}_{h,2:\bar{h}+1}\sim\TT_{h}\big),
\end{align*} 
for all $h\in[H]$. The space of state $\mathcal{S}$ and latent variable $z$ are inherited from the low-level \ac{mdp}.	

\vspace{10pt}
Please refer to Figure \ref{fig:HMDP} for the interactive protocol of \ac{hmdp}. Furthermore, for the high-level POMDP, the state value function of policy $\pi$, the state value function is defined as
\begin{equation}
V_{z,h}^\pi(s,\tau,\omega)=\mathbb{E}_\pi\left[\sum_{h'=h}^H r_{h'}(o_{h'},\omega)\Big\vert s_h=s,\tau_{h}=\tau\right],
\end{equation}
where trajectory $\tau_h\in(\mathcal{O}\times\mathcal{G})^{h-1}\times\mathcal{O}$, and similarly we define the state-action value function as
\begin{equation}
Q_{z,h}^\pi(s,\tau,g,\omega)=\mathbb{E}_\pi\left[\sum_{h'=h}^H r_{h'}(o_{h'},\omega)\Big\vert s_h=s,\tau_{h}=\tau,g_h=g\right],
\end{equation}
where expectation is taken concerning the policy $\pi$, transition kernel $\mathbb{P}_z$, and emission distribution $\mathbb{O}$. Besides, for all $h\in[H]$, denote the probability of observing trajectory $\tau_h$ under policy $\pi$ as
\begin{equation}
\begin{aligned}
&\mathbb{P}_{z}^\pi(\tau_h)=\pi(\tau_h)\cdot\mathbb{P}_{z}(\tau_{h}),\quad \mathbb{P}_{z}(\tau_{h})=\prod_{h'=1}^{h-1}\mathbb{P}(o_{h'+1}\wvert\tau_{h'},g_{h'}),\quad \pi(\tau_h)=\prod_{h'=1}^{h-1}\pi_h(g_{h'}\wvert \tau_{h'}),
\end{aligned}
\label{eq:defpomdp}
\end{equation}
where $\mathbb{P}_z(\tau_h)$ denotes the part of the probability of $\tau_h$ that is incurred by the dynamic environment independent of policies, $\pi(\tau_h)$ denotes the part that can be attributed to the randomness of policy.

\subsection{LLM Pretraining under Transformer Architecture}\label{ap:transformer}
\paragraph{Transformer and Attention Mechanism.}Consider a sequence of $N$ input vectors $\{\mathbf{h}_i\}_{i=1}^n\subset\mathbb{R}^{d}$, written as an input matrix $\mathbf{H}=[\mathbf{h}_1,\dots,\mathbf{h}_n]^\top\in\mathbb{R}^{n\times d}$, where each $\mathbf{h}_i$ is a row of $\mathbf{H}$ (also a token). Consider $\Kb\in\mathbb{R}^{n_s\times d}$ and $\Vb\in\mathbb{R}^{n_s\times d_s}$, then the (softmax) attention mechanism maps these input vectors using the function $\attn(\mathbf{H},\Kb,\Vb)=\Softmax(\Hb\Kb^\top)\Vb\in\mathbb{R}^{n\times d_s}$, where softmax function is applied row-wisely and normalize each vector via the exponential function such that $[\Softmax(\hb)]_i=\exp(\hb_i)/\sum_{j=1}^d\exp(\hb_j)$ for all $i\in[d]$. To approximate sophisticated functions, practitioners use Multi-head Attention (MHA) instead, which forwards the input vectors into $h$ attention modules in parallel with $h\in\mathbb{N}$ as a hyperparameter and outputs the sum of these sub-modules. Denote $\Wb=\{(\Wb^H_i,\Wb^K_i,\Wb^V_i)\}_{i=1}^h$ as the set of weight matrices, the MHA outputs $\mha(\Hb,\Wb)=\sum_{i=1}^h\attn(\Hb\Wb^H_i,\Hb\Wb^K_i,\Hb\Wb^V_i)$, where $\Wb^H_i\in\mathbb{R}^{d\times d_h}$, $\Wb^K_i\in\mathbb{R}^{d\times d_h}$ and $\Wb^V_i\in\mathbb{R}^{d\times d}$ for all $i\in[h]$, and $d_h$ is usually set to $d/h$ \citep{michel2019sixteen}. Based on the definitions above, we are ready to present the transformer architecture employed in LLMs like BERT and GPT \citep{devlin2018bert,brown2020language}. Detailedly, the transformer network has $D$ sub-modules, consisting of an MHA and Feed-Forward (FF) fully-connected layer. Given input matrix $\mathbf{H}^{(0)}=\Hb\in\mathbb{R}^{n\times d}$, in the $j$-th layer for $j\in[D]$, it first takes the output from the $(t-1)$-th layer $\mathbf{H}^{(t-1)}$ as the input matrix, and forwards it to the MHA module with a projection function $\proj[\cdot]$ and a residual link. After receiving intermediate $\overline{\Hb}^{(t)}\in\mathbb{R}^{n\times d}$, the FF module maps each row through a same single-hidden layer neural network with $d_F$ neurons such that $\relu(\overline{\Hb}^{(t)}\Ab_1^{(t)})\Ab_2^{(t)}$, where $\Ab_1^{(t)}\in\mathbb{R}^{d\times d_F}$, $\Ab_2^{(t)}\in\mathbb{R}^{d_F\times d}$, and $[\relu(\Xb)]_{i,j}=\max\{\Xb_{i,j},0\}$. Specifically, the output of the $t$-th layer with $t\in[D]$ can be summarized as below:
\begin{align*}
\overline{\Hb}^{(t)}=\proj\left[\mha\left({\Hb}^{(t-1)},\Wb^{(t)}\right)+\gamma_1^{(t)}{\Hb}^{(t-1)}\right],\quad{\Hb}^{(t)}=\proj\left[\relu(\overline{\Hb}^{(t)}\Ab_1^{(t)})\Ab_2^{(t)}+\gamma_2^{(t)}\overline{\Hb}^{(t)}\right],
\end{align*}
where $\gamma_1^{(t)}$ and $\gamma_2^{(t)}$ features the allocation of residual link. The final output of the transformer is the probability of the next token via a softmax distribution such that 
$$
{\Hb}^{(D+1)}=\Softmax\left(\mathbf{1}^\top{\Hb}^{(D)}\Ab^{(D+1)}/N\gamma^{(D+1)}\right),
$$
where $\Ab^{(D+1)}\in\mathbb{R}^{d\times d_E}$ denotes the weight matrix with dimension $d_E\in\mathbb{N}$ and $\gamma^{(D+1)}\in(0,1]$ is the fixed temperature parameter. Let $\boldsymbol\theta^{(t)}=\left(\Wb^{(t)},\Ab^{(t)},\boldsymbol\gamma^{(t)}\right)$ for all $t\in[D]$, where $\Ab^{(t)}=(\Ab^{(t)}_1,\Ab^{(t)}_2)$ and $\boldsymbol\gamma^{(t)}=(\gamma_1^{(t)},\gamma_2^{(t)})$, and denote $\boldsymbol\theta^{(D+1)}=(\Ab^{(D+1)},\gamma)$. Hence, the parameter of the whole transformer architecture is the concatenation of parameters in each layer such that $\boldsymbol\theta=(\boldsymbol\theta^{(1)},\dots,\boldsymbol\theta^{(D+1)})$, and we consider a bounded parameter space, which is defined as below
\begin{align*}
	\boldsymbol\Theta:=\{\boldsymbol{\theta}\wvert&\|\Ab^{(t)}_1\|_F\leq B_{A,1},\|\Ab^{(t)}_2\|_F\leq B_{A,2},\|\Ab^{(D+1),\top}\|_{1,2}\leq B_{A},|\gamma_1^{(t)}|\leq1,|\gamma_2^{(t)}|\leq1,\\
	&|\gamma^{(D+1)}|\leq1,\|\Wb_i^{H,(t)}\|\leq B_H,\|\Wb_i^{K,(t)}\|\leq B_K,\|\Wb_i^{V,(t)}\|\leq B_V,\forall(i,t)\in[h]\times[D]\}.
\end{align*}
To facilitate the expression of Theorem \ref{thm:llmpretrain}, we further define $\bar{D}=D^2d\cdot(d_h+d_F+d)+d_E\cdot d$ and $\bar{B}=\gamma^{-1}RhB_{A,1}B_{A,2}B_{A}B_HB_KB_V$, where $R$ is (almost surely) the upper bound of the magnitude of each token $\ell\in\mathfrak{L}$ in token sequence $S_t\in\mathfrak{L}^*$, which is defined in Assumption \ref{as:boundedness}.
\vspace{-0.1cm}
\paragraph{Markov Chains.}
We follow the notations used in \citet{paulin2015concentration,zhang2023and}. Let $\Omega$ be a Polish space. The transition kernel for a time-homogeneous Markov chain $\{X_{i}\}_{i=1}^{\infty}$ supported on $\Omega$ is a probability distribution $\PP(x,y)$ for every $x\in\Omega$. Given $X_{1}=x_{1},\cdots,X_{t-1}=x_{t-1}$, the conditional distribution of $X_{t}$ equals $\PP(x_{t-1},y)$. A distribution $\pi$ is said to be a stationary distribution of this Markov chain if $\int_{x\in\Omega}\PP(x,y)\cdot\pi(x)=\pi(y)$. We adopt $\PP_t(x,\cdot)$ to denote the distribution of $X_{t}$ conditioned on $X_{1}=x$. The mixing time of the chain is defined by
\begin{align}
    d(t)=\sup_{x\in\Omega}D_{\rm TV}\big(\PP_{t}(x,\cdot),\pi\big), \quad t_{\rm mix}(\varepsilon)=\min\{t\wvert d(t)\leq\varepsilon\},\quad t_{\rm mix}=t_{\rm mix}(1/4).
\end{align}

%% file: tex/main/extention.tex
\section{Extentions}\label{sec:extention}
\subsection{LLM Planning via Bayesian Aggregated World Model}\label{sec:wm}
\begin{algorithm}[t] 
	\caption{Planning with PAR System - \planner\ with LLM as World Model} 
	\begin{algorithmic}[1]
		\renewcommand{\algorithmicrequire}{\textbf{Input:}}
		\Require Policy $\pi_{\mathtt{exp}}$ with $\eta\in(0,1)$, parameter $c_\cZ>0$, $N_{\rm s}\in\NN$, and $|\cZ|\in\NN$, 
		\renewcommand{\algorithmicrequire}{}
		\Require \qquad and reward function $r=\{r_h\}_{h\in[H]}$ specified by the human user.
		\renewcommand{\algorithmicrequire}{\textbf{Initialize:}}
		\Require $\mathcal{H}_0\leftarrow\{\}$, $\cD_t^{\rm s}\leftarrow\{\}, \forall t\in[T]$, and $\epsilon\leftarrow(\log(c_\mathcal{Z}|\mathcal{Z}|\sqrt{T})/T\eta)^{1/2}$. 
		\For{episode $t$ from $1$ to $T$}
		\State Receive the high-level task $\omega^t$ from the human user.
		\State Sample $\mathcal{I}_t\sim\text{Bernuolli}(\epsilon)$.
		\For{stimulation $n$ from 1 to $N_{\rm s}$}
		\State Sample $g_{h,n}^{t,{\rm s}}\sim{\rm Unif}(\cG)$ for all $h\in[H]$ and set $\prompt_{1,n}^t\leftarrow\cH_t\cup\{o_1^t,g_{1,n}^{t,{\rm s}}\}$.
		\For{step $h$ from $1$ to $H$}
		\State Update $\prompt_{h,n}^t\leftarrow\mathcal{H}_t\cup\big\{o_{1,n}^t,g_{1,n}^{t,{\rm s}},\dots,o_{h,n}^{t,{\rm s}},g_{h,n}^{t,{\rm s}}\big\}$.
		\State Predict $o_{h+1,n}^{t,{\rm s}}\sim\mathtt{LLM}(\cdot\wvert\prompt_{h,n}^t)$ via prompting LLM.
		\EndFor
		\State Update $\cD_t^{\rm s}\leftarrow\cD_t^{\rm s}\cup\big\{o_{1,n}^t,g_{1,n}^{t,{\rm s}},\dots,o_{H-1,n}^{t,{\rm s}},g_{H-1,n}^{t,{\rm s}},o_{H,n}^{t,{\rm s}}\big\}$.
		\EndFor
		\For{step $h$ from $1$ to $H$}
		\State Collect the observation $o_{h}^t$ from the \reporter.
		\State Calculate $\pi_\texttt{LLM}^t\leftarrow\textsc{Optimal-planning}(\omega^t,\cD_t^{\rm s},r)$
		\State Sample $g_{h}^t\sim(1-\cI_t)\cdot\pi_{h,\texttt{LLM}}^t(\cdot\wvert \omega^t,\tau_{h}^t)+\cI_t\cdot\pi_{h,\mathtt{exp}}^t(\cdot\wvert \tau_{h}^t)$.
		\State Send the subgoal $g_h^t$ to the \actor.
		\EndFor
		\State Update $\mathcal{H}_{t+1}\leftarrow\mathcal{H}_{t}\cup\left\{\omega^t,\tau_{H}^t\right\}$.
		\EndFor	
	\end{algorithmic} 
	\label{alg:WM_planner}
\end{algorithm}

Recall that the pretraining algorithm in \S\ref{sec:metric} also equips LLM with the capability to predict observation generation, i.e., $\PP_h(o_h\wvert (o,g)_{1:h-1})$. Existing literature has shown the benefits of augmenting the reasoning process with predicted world states, as it endows LLMs with a more grounded inference without reliance on expert knowledge \citep{hu2023language}. Specifically, 
the \planner\ interactively prompts LLM to internally simulate entire trajectories grounded on historical feedback. By leveraging model-based RL methods such as Monte Carlo Tree Search \citep{browne2012survey} and Real-Time Dynamic Programming \citep{barto1995learning}, the \planner\ utilizes the LLM-simulated environment to optimize its strategies. The planning protocol is as follows: at the beginning of $t$-th episode, \planner\ iteratively prompts LLM with initial observation $o_1$, history $\cH_t$, and subgoals $g_{1:H}$ sequentially to predict observations $o_{1:H}$. Subsequently, a simulation dataset $\cD_t^{\rm s}$ is collected, allowing the \planner\ to compute the optimal policy with rewards specified by the human users, using methods such as MCTS. We first show that the LLM-simulated environment conforms to a Bayesian Aggregated World Model (BAWM), and is formalized as follows.
\begin{proposition}[LLM as BAWM]
Assume that the distribution of pretraining data is given by \eqref{eq:joint_distribution}. Under the perfect setting in Definition  \ref{as:perfect}, for each $(h,t)\in[H]\times[T]$, the LLM serves as a Bayesian aggregated world model, following that
\begin{align}
&\PP_{\mathtt{LLM}}^t(\cdot\wvert o_1,\tdo g_{1:h-1})=\sum_{z\in\mathcal{Z}}\PP_{z}\left(\cdot\wvert o_1,\tdo g_{1:h-1}\right)\cdot\PP_{\cD}\left(z\wvert\cH_t\right),
\end{align}
with marginal distributions defined as $\PP_{z}(\cdot\wvert o_1,\tdo g_{1:h-1})=\int_{o_{2:h-1}}\prod_{h'=1}^{h-1}\PP_{z}\left(o_{h'+1}\wvert(o,g)_{1:h'}\right)\rd o_{2:h-1}$ and $\PP_{\mathtt{LLM}}^t(\cdot\wvert o_1,\tdo g_{1:h-1})=\int_{o_{2:h-1}}\prod_{h'=1}^{h-1}\PP_{\cD}\left(o_{h'+1}\wvert(o,g)_{1:h'},\cH_t\right)\rd o_{2:h-1}$.
\label{thm:bawm}
\end{proposition}
 \begin{proof}[Proof of Propoition \ref{thm:bawm}.]
	Please refer to \S\ref{ap:bawm} for a detailed proof.
\end{proof}
Note that the generation distribution $\PP_{\mathtt{LLM}}^t(\cdot\wvert(o,g)_{1:h})=\texttt{LLM}(\cdot\wvert(o,g)_{1:h},\cH_t)$ is non-stationary, since $\PP_{\cD}(z\wvert(o,g)_{1:h},\cH_t)$ fluctuates with simulated part $(o,g)_{1:h}$ due to the autoregressive manner of LLMs. Instead, Proposition \ref{thm:bawm} posits that the marginal distribution has a stationary expression based on posterior aggregation. Akin to Assumption \ref{as:onlinecoverage}, we introduce the coverage assumption.
\begin{assumption}[Strong Coverage]
\label{as:onlinefullcoverage}
	There exists absolute constants $\lambda_{S,1},\lambda_{S,2}$ and $\lambda_R$ such that for all $z\in\mathcal{Z}$, length $t<\bar{T}_{\rm p}$ and policy sequence $\{\pi^{i}\}_{i\leq\lfloor t/2H\rfloor}$ from the \planner, it holds that (\romannumeral1) $\prod_{i=1}^{\lfloor t/2H\rfloor}\mathbb{\hat{P}}_z^{\pi_i}(\tilde{S}_i)\leq\lambda_{S,1}\cdot\bar{\mathbb{P}}_{\mathcal{D}_\mathtt{LLM}}((\tilde{S}_i)_{i\leq\lfloor t/2H\rfloor})$ and $\bar{\mathbb{P}}_{\mathcal{D}_\mathtt{LLM}}(\tilde{S}_{\lceil t/2H\rceil}|(\tilde{S}_i)_{i\leq\lfloor t/2H\rfloor})\geq\lambda_{S,2}$ for all ordered $S_t=(\tilde{S}_i)_{i\leq\lceil t/2H\rceil}\in\mathfrak{L}^*$, where $|\tilde{S}_i|=2H$ for all $k<\lceil t/2H\rceil$, (\romannumeral2) $\bar{\mathbb{P}}_{\mathcal{D}_\mathtt{Rep}}(s)\geq\lambda_R$ for all $s\in\mathcal{S}$.
\end{assumption}
We remark that Assumption \ref{as:onlinefullcoverage} imposes a stronger condition over the coverage, particularly on the in-episode trajectory $\tilde{S}_{\lceil t/2H\rceil}$, Here, $\lceil t/2H\rceil$ denotes the number of episodes described in $S_t$. The demand of the stronger assumption arises from LLM now serving as a WM, necessitating more extensive information across all kinds of scenarios. Suppose that the \planner\ can learn optimal policy $\hat{\pi}^{t,*}_\mathtt{LLM}={\rm argmax}_{\pi\in\Pi}\ \mathcal{\hat{J}}_\mathtt{LLM}^t(\pi,\omega)$ with sufficiently large simulation steps $|\cD_t^{\rm s}|$, where $\mathcal{\hat{J}}_\mathtt{LLM}^t$ denotes the value function concerning $\mathtt{LLM}_{\hat{\theta}}$ and history $\cH_t$. Akin to Algorithm \ref{alg:planner}, the planning algorithm by taking LLM as WM includes an $\epsilon$-greedy exploration with $\eta$-distinguishable $\pi_\mathtt{exp}$. The pseudocode is in Algorithm \ref{alg:WM_planner}. The following corollary presents the performance under practical settings. 
\begin{corollary}[Regret under Practical Setting with LLM as World Model]
	Suppose that Assumptions \ref{as:coverage}, \ref{as:boundedness}, \ref{as:ambiguity}, \ref{as:realizability} and \ref{as:onlinecoverage}. Given an $\eta$-distinguishable exploration policy $\pi_{\mathtt{exp}}$ and $T\leq T_{\rm p}$, under the practical setting, the \planner's algorithm in Algorithm \ref{alg:WM_planner} ensures that
\begin{align*}
{\rm Reg}_z(T)&\leq\tilde{\mathcal{O}}\Big(H\sqrt{T/\eta\cdot\log(c_\mathcal{Z}|\mathcal{Z}|\sqrt{T})}+H^2T\cdot\Delta_{\rm p,wm}(N_{\rm p},T_{\rm p},H,1/\sqrt{T},\xi)\Big),
\end{align*}
for any $z\in\mathcal{Z}$ and $\{\omega_t\}_{t\in[T]}$. The cumulative pretraining error of the PAR system follows 
\begin{align*}
	&\Delta_{\rm p,wm}(N_{\rm p},T_{\rm p},H,\delta,\xi)=2(\eta\lambda_R)^{-1}\cdot\Delta_\mathtt{Rep}(N_{\rm p},T_{\rm p},H,\delta)^2\\
	&\quad+2\lambda_R^{-1}\cdot\Delta_\mathtt{Rep}(N_{\rm p},T_{\rm p},H,\delta)+2\lambda_{S,1}\lambda_{S,2}^{-1}\cdot\Delta_\mathtt{LLM}(N_{\rm p},T_{\rm p},H,\delta).
\end{align*}
where $\xi=(\eta,\lambda_{S,1},\lambda_{S,2},\lambda_R)$ are defined in Definition \ref{def:iden} and Assumption \ref{as:onlinecoverage}, and errors $\Delta_\mathtt{LLM}$ and $\Delta_\mathtt{Rep}$ are  defined in Theorem \ref{thm:llmpretrain} and Theorem \ref{thm:translatorpretrain}. Under practical setting, \planner\ should explore with probability $\epsilon=(\log(c_\mathcal{Z}|\mathcal{Z}|\sqrt{T})/T\eta)^{1/2}+H(\eta\lambda_{\min})^{-1}\Delta_\mathtt{Rep}(N_{\rm p},T_{\rm p},H,1/\sqrt{T})^2$.
\label{cor:bawm_regret}
\end{corollary}
 \begin{proof}[Proof of Corollary \ref{cor:bawm_regret}.]
	Please refer to \S\ref{ap:bawm_regret} for a detailed proof.
\end{proof}
 
\subsection{LLM-Empowered Multi-Agent Collaboration}\label{sec:multiagent}
\begin{algorithm}[t] 
	\caption{Multi-Agent Planning with PAR System - \planner} 
	\begin{algorithmic}[1]
		\renewcommand{\algorithmicrequire}{\textbf{Input:}}
		\Require Policy $\pi_{\mathtt{exp}}$ with $\eta\in(0,1)$, parameter $c_\cZ>0$, and $|\cZ|\in\NN$.
		\renewcommand{\algorithmicrequire}{\textbf{Initialize:}}
		\Require $\mathcal{H}_0\leftarrow\emptyset$, and $\epsilon\leftarrow(HK\log(c_\mathcal{Z}|\mathcal{Z}|\sqrt{T})/T\eta)^{1/2}$.	
		\For{episode $t$ from $1$ to $T$}
		\State Receive the high-level task $\omega^t$ from the human user.
		\State Sample $\mathcal{I}_t\sim\text{Bernuolli}(\epsilon)$.
		\For{step $h$ from $1$ to $H$}
		\State Collect the observation $o_{h}^t$ from \reporter.
		\For{\actor\ $k$ from $1$ to $K$}
		\State Set $\prompt_{h,k}^t\leftarrow\mathcal{H}_t\cup\left\{\omega^t,o_1^t,g_1^t,\dots,o_{h}^t,k\right\}$.
		\State Sample $g_{h,k,\mathtt{LLM}}^t\sim\mathtt{LLM}(\cdot\wvert\prompt_{h,k}^t)$ via prompting LLM.
		\EndFor
		\State \textbf{If} $\mathcal{I}_t=1$ \textbf{then} $g_{h}^t\leftarrow g_{h,\mathtt{LLM}}^t$, \textbf{else} sample $g_{h}^t\sim\pi_{h,\mathtt{exp}}(\cdot\wvert \tau_{h}^t)$.
		\EndFor
		\State Send the subgoal $g_h^t$ to the \actor s.
		\State Update $\mathcal{H}_{t+1}\leftarrow\mathcal{H}_{t}\cup\left\{\omega^t,\tau_{H}^t\right\}$.
		\EndFor	
	\end{algorithmic} 
	\label{alg:multi_planner}
\end{algorithm}
To characterize the multi-agent interactive process, i.e., several \actor s, of task planning, we consider a turn-based \emph{cooperative}  hierarchical Markov Game (HMG), corresponding to HMDP in \S\ref{sec:PAR}. Instead, HMG consists of a low-level language-conditioned Markov Game (MG) and a high-level language-conditioned cooperative Partially Observable Markov Game (POMG). To extend this framework, we introduce the following modifications: (\romannumeral 1) low-level MG: let $\cK=[K]$ be the set of \actor s, and $\mathcal{G}=\mathcal{G}_1\times\dots\times\mathcal{G}_K$ and $\mathcal{A}=\mathcal{A}_1\times\dots\times\mathcal{A}_K$ be the space of subgoals and low-level actions. Low-level \actor s conduct planning following a joint policy $\mu=\{\mu_{h}\}_{h\in[H]}$ with $\mu_{h}:\mathcal{S}\times\mathcal{G}\mapsto\Delta(\mathcal{A})$, where $\{\mu_{h,k}\}_{k\in\cK}$ can be correlated, e.g., within zero-sum game, Stackelberg game \citep{bacsar1998dynamic}. (\romannumeral 2) high-level POMG: under cooperation, assume that policies can be factorized as 
$$
\pi_h(\mathbf{g}_h\wvert\tau_{h-1},\omega)=\prod_{k=1}^K\pi_{h,k}(g_{h,k}\wvert\tau_{h-1},\omega),\quad\forall h\in[H].
$$ 
The remaining concepts are consistent with HMDP. Here, the \planner\ assumes the role of \emph{central controller} and solves a fully-cooperative POMG that aims to maximize a shared value function. Thus, the \planner\ should infer both the \actor s' intentions, i.e., joint policy $\mu$, and the environment, i.e., transition kernel $\TT$, from the historical context, and then assign subgoal for each \actor. 

Specifically, the LLM's recommendations are obtained by invoking the ICL ability of LLMs with the history-dependent prompt akin to \eqref{eq:prompt} sequentially for each \actor. For the $k$-th \actor, prompt LLM with $\prompt_{h,k}^t=\mathcal{H}_t\cup\{\omega^t,\tau_{h}^t,k\}$, where denote $\mathcal{H}_t=\bigcup_{i=1}^{t-1}\{\omega^i,\tau_{H}^i\}$ and $\tau_h^t=\{o_h^1,\mathbf{g}_h^1,\dots,o_h^t\}$. Under the perfect setting (see Definition \ref{as:perfect}), LLM's joint policy for recommendations follows:	
\begin{align}
		\pi_{h,\mathtt{LLM}}^t\big(\mathbf{g}_h^t\wvert  \tau_{h}^t,\omega^t\big)&=\prod_{k\in\cK}\left(\sum_{z\in\mathcal{Z}}\pi^*_{z,h,k}\left(g_{h,k}^t\wvert \tau_{h}^t,\omega^t\right)\cdot\PP_{\cD}\left(z\wvert\prompt_h^t\right)\right),
		\label{eq:multiagent_BAIL}
	\end{align}
	which is akin to Proposition \ref{thm:BAIL} and the proof of the statement is provided in \S\ref{ap:multiagent_regret}. The pseudocode is presented in Algorithm \ref{alg:multi_planner}. Then, we give the performance guarantee under multi-agent scenarios with the perfect PAR system.
\begin{corollary}[Multi-agent Collaboration Regret under Perfect Setting]
Suppose that Assumptions \ref{as:perfect} and \ref{as:coverage} hold. Given an $\eta$-distinguishable exploration policy $\pi_{\mathtt{exp}}$ and $T\leq T_{\rm p}$, the \planner's algorithm in Algorithm \ref{alg:multi_planner} guarantees that
\begin{align*}
{\rm Reg}_z(T)&\leq\tilde{\mathcal{O}}\left(H^\frac{3}{2}\sqrt{TK/\eta\cdot \log\left(c_\mathcal{Z}|\mathcal{Z}|\sqrt{T}\right)}\right),
\end{align*}
for any $z\in\mathcal{Z}$ and $\{\omega^t\}_{t\in[T]}$, if \planner\ explores with $\epsilon=(HK\log\left(c_\mathcal{Z}|\mathcal{Z}|\sqrt{T}\right)/T\eta)^{1/2}$.
\label{corol:perfectregret}
 \end{corollary}
\begin{proof}[Proof of Corollary \ref{corol:perfectregret}.]
	Please refer to \S\ref{ap:multiagent_regret} for a detailed proof.
\end{proof}
Corollary \ref{corol:perfectregret} is akin to Theorem \ref{thm:regret} with an additional $\sqrt{K}$ in regret. Besides, the multi-agent space of latent variable $|\mathcal{Z}|=|\cZ_\TT|\times|\cZ_{\mu,\rm m}|$, where $\cZ_{\mu,\rm m}$ is the space of joint policy, is generally larger than the single-agent space. Specifically, if responses are uncorrelated, then we have $\log|\cZ_{\mu,\rm m}|=K\log|\cZ_{\mu,\rm s}|$, resulting in a $\sqrt{K}$ times larger regret. The proof of extension to practical setting is akin to Corollary \ref{corol:perfectregret} based on derivations in Theorem \ref{thm:practical regret}, and is omitted. 

%% file: tex/appendix/appendix_perfect.tex
\section{Proof for Section \ref{sec:bail}: Perfect Setting}
\subsection{Proof of Proposition \ref{thm:BAIL}}
\label{ap:bil}
\emph{Proof of Proposition \ref{thm:BAIL}.} Note that for all $h\in[H]$ and $t\in[T]$, we have
\begin{align}
	\pi_{h,\mathtt{LLM}}^t\left(g_{h}^t\wvert \tau_{h}^t,\omega^t\right)&=\sum_{z\in\mathcal{Z}}\PP_{\cD}\left(g_{h}^t\wvert\prompt_h^t,z\right)\cdot\PP_{\cD}\left(z\wvert\prompt_h^t\right)\nonumber\\
	&=\sum_{z\in\mathcal{Z}}\PP_{\cD}\left(g_{h}^t\wvert\mathcal{H}_t,\tau_{h}^t,\omega^t,z\right)\cdot\PP_{\cD}\left(z\wvert\prompt_h^t\right)\nonumber\\
	&=\sum_{z\in\mathcal{Z}}\pi^*_{z,h}\left(\cdot\wvert \tau_{h}^t,\omega^t\right)\cdot\PP_{\cD}\left(z\wvert\prompt_h^t\right),
\end{align}
where the second equation results from the law of total probability, the third equation follows the definition of prompts in \eqref{eq:prompt}, and the last equation results from the generation distribution.\hfill$\Box$

\subsection{Proof of Theorem \ref{thm:regret}}\label{ap:perfectplanning}
\emph{Proof of Thereom \ref{thm:regret}.} Recall that the \planner\ takes a mixture policy of $\pi_\mathtt{exp}$ and $\pi_\mathtt{LLM}$ such that
\begin{equation}
\pi_h^t(\cdot\wvert \tau_{h}^t,\omega^t)\sim(1-\epsilon)\cdot\pi_{h,\mathtt{LLM}}^t(\cdot\wvert \tau_{h}^t,\omega^t)+\epsilon\cdot\pi_{h,\mathtt{exp}}(\cdot\wvert\tau_h^t),
\label{eq:mix_policy}
\end{equation}
and  Proposition \ref{thm:BAIL} indicates that LLM's recommended policies take the form:
\begin{align}
&\pi_{h,\mathtt{LLM}}^t\left(\cdot\wvert \tau_{h}^t,\omega^t\right)=\sum_{z\in\mathcal{Z}}\pi^*_{z,h}\left(\cdot\wvert \tau_{h}^t,\omega^t\right)\cdot\PP_{\cD}\left(z\wvert\prompt_h^t\right),\text{~where~}\prompt_h^t=\cH_t\cup\tau_h^t, \mathcal{H}_t=\left\{\omega^i,\tau_{H}^i\right\}_{i\in[t-1]},
\label{eq:perfect_llmpolicy}
\end{align}
for all $(h,t)\in[H]\times[T]$. Following \eqref{eq:mix_policy}, given $z\in\cZ$ and $\{\omega^t\}_{t\in[T]}$, the regret is decomposed as
\begin{align}
\text{Reg}(T)&=\underbrace{\sum_{t=1}^T\sum_{h=1}^H\EE_{\cH_t\sim\bigotimes_{i=1}^{t-1}\mathbb{P}_{z}^{\pi_i}}\mathbb{E}_{(s_h^t,\tau_{h}^t)\sim\mathbb{P}_{z}^{\pi^t}}\left[\left(\pi_{z,h}^*-\pi_{h,\mathtt{exp}}\right)Q_{z,h}^*(s_h^t,\tau_{h}^t,\omega^t)\right]\cdot\epsilon}_{\textbf{(\romannumeral1)}}\nonumber\\
&\hspace{0.4cm}+\underbrace{\sum_{t=1}^T\sum_{h=1}^H\EE_{\cH_t\sim\bigotimes_{i=1}^{t-1}\mathbb{P}_{z}^{\pi_i}}\mathbb{E}_{(s_h^t,\tau_{h}^t)\sim\mathbb{P}_{z}^{\pi^t}}\left[\left(\pi_{z,h}^*-{\pi^t_{h,\mathtt{LLM}}}\right)Q_{z,h}^*(s_h^t,\tau_{h}^t,\omega^t)\right]\cdot(1-\epsilon)}_{\textbf{(\romannumeral2)}}\nonumber\\
&\leq\sum_{t=1}^T\sum_{h=1}^H\EE_{\cH_t\sim\bigotimes_{i=1}^{t-1}\mathbb{P}_{z}^{\pi_i}}\mathbb{E}_{(s_h^t,\tau_{h}^t)\sim\mathbb{P}_{z}^{\pi^t}}\left[\left(\pi_{z,h}^*-{\pi^t_{h,\mathtt{LLM}}}\right)Q_{z,h}^*(s_h^t,\tau_{h}^t,\omega^t)\right]+HT\epsilon ,
\label{eq:regretdecom}
\end{align}
where the second equation results from performance difference lemma (PDL, see Lemma \ref{lem:pdl}), and we write $\pi_hQ_{h}(s_h,\tau_{h},\omega)=\langle \pi_h(\cdot|\tau_h,\omega),Q_{h}(s_h,\tau_{h},\cdot,\omega)\rangle_\cG$, and $\mathbb{P}_{z}^\pi(\tau_h)$ is defined in \eqref{eq:defpomdp}. Based on Lemma \ref{online guarantee}, with probability at least $1-\delta$, the following event $\cE_1$ holds: for all $(h,t)\in[H]\times[T]$, 
\begin{equation}
	\sum_{z'\in\mathcal{Z}}\sum_{i\in[t]}D_{\rm H}^2\big(\mathbb{P}_{z}^{\pi^i}(\breve{\tau}_{h/t}^{i}),\mathbb{P}_{z'}^{\pi^i}(\breve{\tau}_{h/t}^{i})\big)\cdot\PP_{\cD}(z'\wvert\prompt_h^t)\leq 2\log\left(c_\mathcal{Z}|\mathcal{Z}|/\delta\right),
	\label{eq:online}
\end{equation}
where the randomness is incurred by $\prompt_h^t$ and define $\breve{\tau}_{h/t}^{i}={\tau}_H$ for all $i\in[t-1]$ and $\breve{\tau}_{h/t}^{t}={\tau}_{h}$ for notational simplicity. Suppose that event $\cE_1$ in \eqref{eq:online} holds, and denote $\mathcal{X}^t_{\mathtt{exp}}=\{i\in[t]:\pi^i=\pi_{\mathtt{exp}}\}$ as the set of exploration episodes. Note that for all $(h,t,z')\in[H]\times[T]\times\cZ$, it holds that
\begin{align}
\sum_{i\in[t]}D_{\rm H}^2\big(\mathbb{P}_{z}^{\pi^i}(\breve{\tau}_{h/t}^{i}),\mathbb{P}_{z'}^{\pi^i}(\breve{\tau}_{h/t}^{i})\big)\geq\sum_{i\in\mathcal{X}^{t-1}_{\mathtt{exp}}}D_{\rm H}^2\left(\mathbb{P}_{z}^{\pi_\mathtt{exp}}({\tau}_H),\mathbb{P}_{z'}^{\pi_\mathtt{exp}}({\tau}_H)\right)\geq{\eta\cdot|\mathcal{X}^{t-1}_{\mathtt{exp}}|},
\label{eq:cumerrorlb}
\end{align}
 where the last inequality results from $\pi_{\mathtt{exp}}$ is $\eta$-distinguishable (see Definition \ref{def:iden}) and the fact that $D_{\rm H}^2(P,Q)\leq1$ for all $P,Q\in\Delta(\mathcal{X})$. Combine \eqref{eq:online} and \eqref{eq:cumerrorlb}, we can get
\begin{equation}
\sum_{z'\ne z}\PP_{\cD}(z'\wvert\prompt_h^t)\leq\min\left\{2\log\left(c_\mathcal{Z}|\mathcal{Z}|/\delta\right)\eta^{-1}/|\mathcal{X}^{t-1}_{\mathtt{exp}}|,1\right\},
\label{eq:nonoptimal}
\end{equation}
for all $(h,t)\in[H]\times[T]$. Recall that \eqref{eq:perfect_llmpolicy} indicates that for all $(h,t)\in[H]\times[T]$, we have
$$
\left(\pi_{z,h}^*-{\pi^t_{h,\mathtt{LLM}}}\right)(\cdot\wvert \tau_h,\omega)=\sum_{z'\neq z}\left(\pi_{z,h}^*-\pi_{z',h}^*\right)(\cdot\wvert \tau_h,\omega)\cdot\PP_{\cD}(z'\wvert\prompt_h^t).
$$
Based on Proposition \ref{thm:BAIL} and conditioned on $\cE_1$, it holds that
\begin{align}
	\sum_{t=1}^T&\sum_{h=1}^H\EE_{\cH_t\sim\bigotimes_{i=1}^{t-1}\mathbb{P}_{z}^{\pi_i}}\mathbb{E}_{(s_h^t,\tau_{h}^t)\sim\mathbb{P}_{z}^{\pi^t}}\left[\left(\pi_{z,h}^*-{\pi^t_{h,\mathtt{LLM}}}\right)Q_{z,h}^*(s_h^t,\tau_{h}^t,\omega^t)\right]\nonumber\\
	&\leq H\cdot\sum_{t=1}^T\sum_{h=1}^H\sum_{z'\neq z}\EE_{\cH_t\sim\bigotimes_{i=1}^{t-1}\mathbb{P}_{z}^{\pi_i}}\mathbb{E}_{\tau_{h}^t\sim\mathbb{P}_{z}^{\pi^t}}\Biggl[\PP_{\cD}(z'\wvert\prompt_h^t)\Biggl]\nonumber\\
	&\leq2\log(c_\mathcal{Z}|\mathcal{Z}|/\delta)H\eta^{-1}\cdot\sum_{t=1}^T\sum_{h=1}^H\EE\left[\min\left\{1/|\mathcal{X}^{t-1}_{\mathtt{exp}}|,1\right\}\right],
	\label{eq:llmpdl}
\end{align}
Note that $\mathds{1}(\pi^t=\pi_\mathtt{exp})\overset{\rm iid}{\sim}\text{Bernuolli}(\epsilon)$ for all $t\in[T]$. Besides,the following event $\cE_2$ holds:
\begin{align}
	\sum_{t=1}^{T}\min\left\{1/|\mathcal{X}^{t-1}_{\mathtt{exp}}|,1\right\}\leq\cO(\epsilon^{-1}\log(T\log T/\delta)).
	\label{eq:planningerr}
\end{align}
with probability at least $1-\delta$ based on Lemma \ref{lem:epsilon_concen}. Combine \eqref{eq:regretdecom}, \eqref{eq:llmpdl} and \eqref{eq:planningerr}, we have
\begin{align*}
	{\rm Reg}_z(T)&\leq \sum_{t=1}^T\sum_{h=1}^H\EE_{\cH_t\sim\bigotimes_{i=1}^{t-1}\mathbb{P}_{z}^{\pi_i}}\mathbb{E}_{(s_h^t,\tau_{h}^t)\sim\mathbb{P}_{z}^{\pi^t}}\left[\left(\pi_{z,h}^*-{\pi^t_{h,\mathtt{LLM}}}\right)Q_{z,h}^*(s_h,\tau_{h},\omega^t)\ind\left(\cE_1\cap\cE_2\text{~holds}\right)\right]\\
	&\quad+ \sum_{t=1}^T\sum_{h=1}^H\EE_{\cH_t\sim\bigotimes_{i=1}^{t-1}\mathbb{P}_{z}^{\pi_i}}\mathbb{E}_{(s_h^t,\tau_{h}^t)\sim\mathbb{P}_{z}^{\pi^t}}\left[\left(\pi_{z,h}^*-{\pi^t_{h,\mathtt{LLM}}}\right)Q_{z,h}^*(s_h,\tau_{h},\omega^t)\ind\left(\cE_1\cap\cE_2\text{~fails}\right)\right]+ HT\epsilon\\
	&\leq\mathcal{O}\Big(\log(c_\mathcal{Z}|\mathcal{Z}|/\delta)H^2\log(T\log T/\delta)\cdot(\eta\epsilon)^{-1}+ HT\epsilon+2HT\delta\Big)\\
	&\leq\tilde{\mathcal{O}}\left(H^\frac{3}{2}\sqrt{\log\left(c_\mathcal{Z}|\mathcal{Z}|/\delta\right)T/\eta}\right),
\end{align*}  
where we choose to expolre with probability $\epsilon=(H\log\left(c_\mathcal{Z}|\mathcal{Z}|/\delta\right)/T\eta)^{1/2}$. If we take $\delta=1/\sqrt{T}$ in the arguments above, then we can conclude the proof of Theorem \ref{thm:regret}.\hfill$\Box$

\subsection{Proof of Lemma \ref{online guarantee}}\label{ap:online}
\begin{lemma}\label{online guarantee}
Suppose that Assumptions \ref{as:perfect} and \ref{as:coverage} hold. Given $\delta\in(0,1)$ and ground-truth $z\in\mathcal{Z}$, for all $(h,t)\in[H]\times[T]$, with probability at least $1-\delta$, it holds that
$$
\sum_{z'\in\mathcal{Z}}\sum_{i\in[t]}D_{\rm H}^2\big(\mathbb{P}_{z}^{\pi^i}(\breve{\tau}_{h/t}^{i}),\mathbb{P}_{z'}^{\pi^i}(\breve{\tau}_{h/t}^{i})\big)\cdot\PP_{\cD}(z'\wvert\prompt_h^t)\leq 2\log\left(c_\mathcal{Z}|\mathcal{Z}|/\delta\right),
$$
where denote $\breve{\tau}_{h/t}^i={\tau}_H$ for all $i<t$ and $\breve{\tau}_{h/t}^{t}={\tau}_{h}$, and $\mathbb{P}_{z}^\pi(\tau_h)$ is defined in \eqref{eq:defpomdp}.
\end{lemma}
\noindent\emph{Proof of Lemma \ref{online guarantee}.} The proof is rather standard (e.g., see \cite{geer2000empirical}). Let $\mathfrak{F}_{t}$ be the filtration induced by $\{\omega^i,\tau_H^i\}_{i<t}\cup\{\ind(\pi^i=\pi_\texttt{exp})\}_{i\in[t]}$. For all $(h,t,z')\in[H]\times[T]\times\mathcal{Z}$, with probability at least $1-\delta$, the information gain concerning $z'$ satisfies that
\begin{align}
	L_{h,t}(z')=\sum_{i=1}^t \log\left(\frac{\mathbb{P}_{z'}(\breve{\tau}_{h/t}^{i})}{\mathbb{P}_{z}(\breve{\tau}_{h/t}^{i})}\right)&\leq2\log\mathbb{E}_{\mathfrak{F}_{1:t}}\left[\exp\left(\frac{1}{2}\sum_{i=1}^t \log\frac{\mathbb{P}_{z'}(\breve{\tau}_{h/t}^{i})}{\mathbb{P}_{z}(\breve{\tau}_{h/t}^{i})}\right)\right]+2\log(|\mathcal{Z}|/\delta),
	\label{eq:info_gain_concen}
\end{align}
where the inequality follows Lemma \ref{lem:exptrick} with $\lambda=1/2$ and a union bound taken over $\mathcal{Z}$. Besides,
\begin{align}
\mathbb{E}_{\mathfrak{F}_{1:t}}\left[\exp\left(\frac{1}{2}\sum_{i=1}^t \log\frac{\mathbb{P}_{z'}(\breve{\tau}_{h/t}^{i})}{\mathbb{P}_{z}(\breve{\tau}_{h/t}^{i})}\right)\right]=\prod_{i=1}^t\left(1-D_{\rm H}^2\big(\mathbb{P}_{z}^{\pi^i}(\breve{\tau}_{h/t}^{i}),\mathbb{P}_{z'}^{\pi^i}(\breve{\tau}_{h/t}^{i})\big)\right).
	\label{eq:onlinehellbound}
\end{align}
Combine \eqref{eq:info_gain_concen}, \eqref{eq:onlinehellbound} and fact that $\log(1-x)\leq -x$ for all $x\leq1$, it holds that
\begin{equation}
	L_{h,t}(z')\leq-2\sum_{i=1}^tD_{\rm H}^2\big(\mathbb{P}_{z}^{\pi^i}(\breve{\tau}_{h/t}^{i}),\mathbb{P}_{z'}^{\pi^i}(\breve{\tau}_{h/t}^{i})\big)+2\log(|\mathcal{Z}|/\delta),
\end{equation}
with probability greater than $1-\delta$. Based on the Donsker-Varadhan representation in Lemma \ref{lem:dvrepresent} and duality principle, we have
$\log\mathbb{E}_{Q}[e^f]=\sup_{P\in\Delta(\cX)}\left\{\mathbb{E}_{P}\left[f\right]-D_{\rm KL}(P\wVert Q)\right\}$, where the supremum is taken at $P(x)\propto\exp(f(x))\cdot Q(x)$. Please refer to Lemma 4.10 in \cite{van2014probability} for detailed proof. Based on the arguments above, for all $(h,t,P)\in[H]\times[T]\times\Delta(\cZ)$, it holds 
\begin{align}
	&\sum_{z'\in\mathcal{Z}}L_{h,t}(z')\cdot P(z')-D_{\rm KL}\big{(}P\wVert\mathcal{P}_\mathcal{Z}\big{)}\leq\sum_{z'\in\mathcal{Z}}L_{h,t}(z')\cdot\PP_{\cD}(z'\wvert\prompt_h^t)-D_{\rm KL}\big{(}\PP_{\cD}(\cdot\wvert\prompt_h^t)\wVert\mathcal{P}_\mathcal{Z}\big{)}.
	\label{eq:posterioroptimizer}
\end{align}
since $\PP_{\cD}(z'\wvert\prompt_h^t)\propto\exp\left(L_{h,t}(z')\right)\cdot\mathcal{P}_\mathcal{Z}(z')$ for all $(h,t)\in[H]\times[T]$. Let $\delta_z(\cdot)$ bs the Dirac distribution over the singleton $z$. Following this, by taking $P=\delta_z$ in \eqref{eq:posterioroptimizer}, we have 
\begin{equation}
	\sum_{z'\in\mathcal{Z}}L_{h,t}(z')\cdot\PP_{\cD}(z'\wvert\prompt_h^t)\geq D_{\rm KL}\big{(}\PP_{\cD}(\cdot\wvert\prompt_h^t)\wVert\mathcal{P}_\mathcal{Z}\big)+\log\cP_\mathcal{Z}(z)\geq\log\cP_\mathcal{Z}(z),
	\label{eq:dv_cor}
\end{equation}
where the first inequality uses $D_{\rm KL}(\delta_z(\cdot)\wVert\mathcal{P}_\mathcal{Z}(\cdot))=-\log\cP_\mathcal{Z}(z)$ based on the definitions. Therefore, for all $(h,t)\in[H]\times[T]$, with probability at least $1-\delta$, it holds that
\begin{align}
	\sum_{z'\in\mathcal{Z}}\sum_{i\in[t]}D_{\rm H}^2\big(\mathbb{P}_{z}^{\pi^i}(\breve{\tau}_{h/t}^{i}),\mathbb{P}_{z'}^{\pi^i}(\breve{\tau}_{h/t}^{i})\big)\cdot\PP_{\cD}(z'\wvert\prompt_h^t)&\leq-\sum_{z'\in\mathcal{Z}}L_{h,t}(z')/2\cdot\PP_{\cD}(z'\wvert\prompt_h^t)+\log\left(|\mathcal{Z}|/\delta\right)\nonumber\\
	&\leq2\log\left(c_\mathcal{Z}|\mathcal{Z}|/\delta\right),
	\label{eq:onlinedecom}
\end{align}
where the first inequality results from \eqref{eq:onlinehellbound}, and the last inequality follows \eqref{eq:dv_cor} and Assumption \ref{as:coverage}, which indicates that $1/\mathcal{P}_\mathcal{Z}(z)\leq c_\mathcal{Z}|\mathcal{Z}|$. Thus, we conclude the proof of Lemma \ref{online guarantee}.\hfill$\Box$

\subsection{Proof of Proposition \ref{prop:hardexample}}\label{ap:hard_exp}
Our construction of the hard-to-distinguish example is a natural extension to the hard instance for the contextual bandit problem in Proposition 1 \citep{zhang2022feel}.

\begin{proof}[Proof of Proposition \ref{prop:hardexample}.] Suppose that the high-level POMDP is fully observable, i.e., $\mathbb{O}(s)=s$, with $H=2$ and $|\Omega|$=1. Consider $\mathcal{S}=\{s_1,s_2,s_3\}$ with rewards $r(s_1)=0.5$, $ r(s_2)=1$, $r(s_3)=0$, $\mathcal{G}=\{g_1,g_2\}$, and $\mathcal{Z}=\{z_1,\dots,z_N\}$. Starting from initial state $s_1$, the transition kernel follows
$$
\left\{
\begin{aligned}
	&\mathbb{P}_{z_i}(s_1\wvert s_1,g_1)=1,\quad\mathbb{P}_{z_i}(s_2\wvert s_1,g_1)=0,\quad\mathbb{P}_{z_i}(s_3\wvert s_1,g_1)=0,\hspace{1.2cm}\forall i\in [N],\\
	&\mathbb{P}_{z_1}(s_1\wvert s_1,g_2)=0,\quad\mathbb{P}_{z_1}(s_1\wvert s_1,g_2)=1,\quad\mathbb{P}_{z_1}(s_3\wvert s_1,g_2)=0,\hspace{1.15cm}\text{if~}i=1,\\
	&\mathbb{P}_{z_i}(s_1\wvert s_1,g_2)=0,\quad\mathbb{P}_{z_i}(s_2\wvert s_1,g_2)=p_i,\quad\mathbb{P}_{z_i}(s_3\wvert s_1,g_2)=1-p_i,\quad\text{if~}i\neq1,
\end{aligned}\right.
$$
where $p_i=0.5(1-\frac{i}{N})$ for all $i\in[N]$. 
For latent environment $z_1$, the optimal policy is $\pi^*_{z_1,1}(s_1)=g_2$ and $\pi^*_{z_i,1}(s_1)=g_1$ if $i\neq1$. Suppose that prior distribution $\mathcal{P}_\mathcal{Z}$ is uniform. At $t=1$, without any information, the posterior $\mathbb{P}(\cdot\wvert\prompt_1)$ degenerates to prior $\mathcal{P}_\mathcal{Z}(\cdot)={\rm Unif}_\mathcal{Z}(\cdot)$. Hence, the LLM's policy at first step follows that
$
\pi_{\mathtt{LLM}}(\cdot\wvert s_1)=\left(1-\frac{1}{N}\right)\cdot\delta_{g_1}(\cdot)+\frac{1}{N}\cdot\delta_{g_2}(\cdot).
$
Since $\mathbb{P}_{z_i}(s_1\wvert s_1,g_1)=1$ and $\mathbb{P}_{z_i}(s_2\wvert s_1,g_1)=\mathbb{P}_{z_i}(s_3\wvert s_1,g_1)=0$ for all $i\in[N]$, taking subgoal $g_1$ provides no information to differentiate $z_i$ from others, and the posterior remains uniform. 
Such situation, i.e., $\mathbb{P}(\cdot\wvert\prompt_t)={\rm Unif}_\mathcal{Z}(\cdot)$, ends only if the LLM suggests taking $g_2$ at some epsiode $t$. Consider the hard trajectory $\tau_{\rm hard}=\{s_1,g_1,s_1\}_{t\in[T]}$, where LLM consistently adheres to the initial $\pi_{\mathtt{LLM}}$ and keeps recommending subgoal $g_1$. Thus, we have $\mathbb{P}_{z_1}(\tau_{\rm hard})=(1-1/N)^T$, indicating ${\rm Reg}_{z_1}(T)\geq 0.5T\cdot(1-1/N)^T$. 
\end{proof}

%% file: tex/appendix/appendix_imperfect.tex
\section{Proof for Section~\ref{sec:pretrain}: Practical Setting}
\subsection{Proof of Theorem \ref{thm:translatorpretrain}}\label{ap:contrastive}
\emph{Proof of Theorem \ref{thm:translatorpretrain}.} Recall that the binary discriminator for label $y\in\{0,1\}$ is defined as
$$
\mathbb{D}_\gamma(y\wvert o,s):=\left(\frac{f_\gamma(o, s)}{1+f_\gamma(o, s)}\right)^y\left(\frac{1}{1+f_\gamma(o, s)}\right)^{1-y},
$$
and the contrastive learning algorithm in \eqref{eq:contrastive_loss} follows $\hat{\gamma}={\rm argmax}_{\gamma\in\Gamma}\ \mathbb{\hat E}_{\mathcal{D}_\mathtt{Rep}}\big[\log \mathbb{D}_\gamma(y\wvert o,s)\big]$, and thus $f_{\hat{\gamma}}$ is the maximum likelihood estimator (MLE) concerning the dataset $\mathcal{D}_\mathtt{Rep}$. Based on Lemma \ref{lem:mle}, the MLE-type algorithm ensures that, with probability at least $1-\delta$, it holds that
\begin{equation}
	\bar{\mathbb{E}}_{(o,s)\sim\mathcal{D}_\mathtt{Rep}}\left[D_{\rm TV}^2\left(\mathbb{D}_{\hat{\gamma}}(\cdot\wvert o,s),\mathbb{D}(\cdot\wvert o,s)\right)\right]\leq {2\log(N_{\rm p}T_{\rm p}H|\mathcal{F}_\gamma|/\delta)}/{N_{\rm p}T_{\rm p}H},
	\label{eq:MLEguarantee}
\end{equation}
where $\mathbb{D}(\cdot\wvert o,s)=\mathbb{D}_{\gamma^*}(\cdot\wvert o,s)$  with $f_{\gamma^*}=f^*\in\mathcal{F}_\gamma$ denotes the ground-truth discriminator based on the realizability in Assumption \ref{as:realizability}. Based on the definition of total variation, it holds that
\begin{align}
	&D_{\rm TV}^2\left(\mathbb{D}_{\hat{\gamma}}(\cdot\wvert o,s),\mathbb{D}(\cdot\wvert o,s)\right)\nonumber\\
	&\quad=\left(\frac{f_{\hat{\gamma}}(o, s)-f^*(o, s)}{(1+f_{\hat{\gamma}}(o, s))(1+f^*(o, s))}\right)^2\leq\frac{1}{(1+R_\mathcal{F})^2}\left(\frac{f_{\hat{\gamma}}(o, s)-f^*(o, s)}{1+f^*(o, s)}\right)^2\nonumber\\
	&\quad=\frac{1}{(1+R_\mathcal{F})^2}\left(\frac{\mathbb{O}_{\hat{\gamma}}(o\wvert s)-\mathbb{O}(o\wvert s)}{\mathcal{P}^-(o)+\mathbb{O}(o\wvert s)}\right)^2=\frac{1}{(1+R_\mathcal{F})^2}\left(\frac{\mathbb{\bar{O}}_{\hat{\gamma}}(o\wvert s)-\mathbb{\bar{O}}(o\wvert s)}{\mathbb{\bar{O}}(o\wvert s)}\right)^2,
	\label{eq:tvreporter}
\end{align}
where the first inequality results from $\|f\|_\infty\leq R_\mathcal{F}$ for all $f\in\mathcal{F}_\gamma$, the third equation arise from the definition that $\mathbb{O}_\gamma(\cdot|s)=f_\gamma(\cdot,s)\cdot\mathcal{P}^-(\cdot)$, and we write $\mathbb{\bar{O}}(\cdot\wvert s)=\frac{1}{2}\left(\mathbb{O}(\cdot\wvert s)+\mathcal{P}^-(\cdot)\right),\mathbb{\bar{O}}_{{\gamma}}(\cdot\wvert s)=\frac{1}{2}\left(\mathbb{O}_{{\gamma}}(\cdot\wvert s)+\mathcal{P}^-(\cdot)\right)$. Moreover, $\mathbb{\bar O}(\cdot\vert s)$ represents the marginal distribution derived from the joint distribution $\PP_\cC$ of collected dataset $\mathcal{D}_\mathtt{Rep}$ (see data collection process in \S\ref{sec:metric}), as follows:
\begin{align}
	\mathbb{P}_{\cC}(o\wvert s)&=\mathbb{P}_{\cC}(o\wvert s,y=0)\cdot \mathbb{P}_{\cC}(y=0\wvert s)+\mathbb{P}_{\cC}(o\wvert s,y=1)\cdot \mathbb{P}_{\cC}(y=1\wvert s)\nonumber\\
	&=\mathbb{P}_{\cC}(o\wvert s,y=0)\cdot \mathbb{P}_{\cC}(y=0)+\mathbb{P}_{\cC}(o\wvert s,y=1)\cdot \mathbb{P}_{\cC}(y=1):=\mathbb{\bar{O}}(o\wvert s),
	\label{eq:bartranslation}
\end{align}
where the second equation results from the fact that contrastive data are labeled independent of data itself such that $\mathbb{P}_\cC(s\wvert y)=\mathbb{P}_\cC(s)$ for all $y\in\{0,1\}$. Based on \eqref{eq:bartranslation},  we can get
\begin{align}
	\bar{\mathbb{E}}_{(o,s)\sim\mathcal{D}_\mathtt{Rep}}\left[\left(\frac{\mathbb{\bar{O}}_{\hat{\gamma}}(o\wvert s)-\mathbb{\bar{O}}(o\wvert s)}{\mathbb{\bar{O}}(o\wvert s)}\right)^2\right]=\bar{\mathbb{E}}_{s\sim\mathcal{D}_\mathtt{Rep}}\left[\mathbb{E}_{o\sim\mathbb{\bar O}(\cdot\wvert s)}\left[\left(\frac{\mathbb{\bar{O}}_{\hat{\gamma}}(\cdot\wvert s)-\mathbb{\bar{O}}(\cdot\wvert s)}{\mathbb{\bar{O}}(\cdot\wvert s)}\right)^2\right]\right],
	\label{eq:repmargin}
\end{align}
where equations results from the fact that $\PP_{\cC}(o,s)=\mathbb{\bar{O}}(o\wvert s)\cdot\PP_{\cC}(s)$ and definition of $\chi^2$-divergence. Therefore, combine \eqref{eq:tvreporter} and \eqref{eq:repmargin}, it holds that
\begin{align}
\bar{\mathbb{E}}_{(o,s)\sim\mathcal{D}_\mathtt{Rep}}\left[D_{\rm TV}^2\left(\mathbb{D}_{\hat{\gamma}}(\cdot\wvert o,s),\mathbb{D}(\cdot\wvert o,s)\right)\right]\leq\frac{1}{(1+R_\mathcal{F})^2}\cdot\bar{\mathbb{E}}_{s\sim\mathcal{D}_\mathtt{Rep}}\left[\chi^2\left(\mathbb{\bar{O}}_{\hat{\gamma}}(\cdot\wvert s)\wVert\mathbb{\bar{O}}(\cdot\wvert s)\right)\right].
\label{eq:tvchisquare}
\end{align}
Based on the variational representation of $f$-divergenve \citep[\S 7.13,][]{polyanskiy2022information}, we have
\begin{align}
	\chi^2\left(\mathbb{\bar{O}}_{\hat{\gamma}}(\cdot\wvert s)\wVert\mathbb{\bar{O}}(\cdot\wvert s)\right)&=\sup_{g:\mathcal{O}\mapsto\mathbb{R}}\left\{\frac{\left(\mathbb{E}_{\mathbb{\bar{O}}_{\hat\gamma}}[g(o)\wvert s]-\mathbb{E}_{\mathbb{\bar{O}}}[g(o)\wvert s]\right)^2}{{\rm Var}_{\mathbb{\bar{O}}}[g(o)\wvert s]}\right\}\nonumber\\
	&=\sup_{g:\mathcal{O}\mapsto\mathbb{R}}\left\{\frac{\left(\mathbb{E}_{\mathbb{O}_{\hat\gamma}}[g(o)\wvert s]-\mathbb{E}_{\mathbb{O}}[g(o)\wvert s]\right)^2}{4\cdot{\rm Var}_{\mathbb{O}}[g(o)\wvert s]}\cdot\frac{{\rm Var}_{\mathbb{O}}[g(o)\wvert s]}{{\rm Var}_{\mathbb{\bar{O}}}[g(o)\wvert s]}\right\}\nonumber\\
	&\geq\sup_{\substack{g:\mathcal{O}\mapsto\mathbb{R},\\ \mathbb{E}_{\mathbb{O}}[g(o)|s]=0}}\left\{\frac{\left(\mathbb{E}_{\mathbb{O}_{\hat\gamma}}[g(o)|s]-\mathbb{E}_\mathbb{O}[g(o)\wvert s]\right)^2}{4\cdot{\rm Var}_{\mathbb{O}}[g(o)\wvert s]}\cdot\frac{\mathbb{E}_{\mathbb{O}}[g(o)^2\wvert s]}{\mathbb{E}_{\mathbb{\bar{O}}}[g(o)^2\wvert s]}\right\},
	\label{eq:chidivvariation}
\end{align}
where the second equation follows the defintions of $\mathbb{\bar{O}}(\cdot\wvert s)$ and $\mathbb{\bar{O}}_{\hat{\gamma}}(\cdot\wvert s)$, and the inequality results from ${\rm Var}_{\mathbb{\bar{O}}}[g(o)|s]=\mathbb{E}_{\mathbb{\bar{O}}}[g(o)^2|s]$ if $\mathbb{E}_{\mathbb{O}_{\hat\gamma}}[g(o)|s]$=0. Furthermore, note that
\begin{align}
\frac{\mathbb{E}_{\mathbb{O}}[g(o)^2\wvert s]}{\mathbb{E}_{\mathbb{\bar{O}}}[g(o)^2\wvert s]}&=2\left(1+\frac{\mathbb{E}_{\mathcal{P}^{-}}[g(o)^2\wvert s]}{\mathbb{E}_{\mathbb{O}}[g(o)^2\wvert s]}\right)^{-1}\leq2\left(1+\left\|\frac{\mathcal{P}^-(\cdot)}{\mathbb{O}(\cdot\wvert s)}\right\|_\infty\right)^{-1}\leq2(1+B^-_\mathcal{F})^{-1},
\label{eq:inverseratiobound}
\end{align}
as $\mathcal{P}^-(\cdot)/\mathbb{P}(\cdot|s)=f^*\in\mathcal{F}_\gamma$ and $\|1/f\|_\infty\leq B^-_\mathcal{F}$ for all $f\in\mathcal{F}$ under the realizability in Assumption \ref{as:realizability}. Besides, it holds that
\begin{align}
	\sup_{\substack{g:\mathcal{O}\mapsto\mathbb{R},\\ \mathbb{E}_{\mathbb{O}}[g(o)|s]=0}}\left\{\frac{\left(\mathbb{E}_{\mathbb{O}_{\hat\gamma}}[g(o)\wvert s]-\mathbb{E}_\mathbb{O}[g(o)\wvert s]\right)^2}{{\rm Var}_{\mathbb{O}}[g(o)\wvert s]}\right\}&=\sup_{g:\mathcal{O}\mapsto\mathbb{R}}\left\{\frac{\left(\mathbb{E}_{\mathbb{O}_{\hat\gamma}}[g(o)\wvert s]-\mathbb{E}_\mathbb{O}[g(o)\wvert s]\right)^2}{{\rm Var}_{\mathbb{O}}[g(o)\wvert s]}\right\}\nonumber\\
	&=\chi^2\left(\mathbb{O}_{\hat{\gamma}}(\cdot\wvert s)\wVert\mathbb{O}(\cdot\wvert s)\right),
	\label{eq:chisquarediv}
\end{align} 
Based on \eqref{eq:MLEguarantee}, \eqref{eq:tvchisquare}, \eqref{eq:chidivvariation}, \eqref{eq:inverseratiobound} and \eqref{eq:chisquarediv}, then we have
 \begin{equation}
 	\bar{\mathbb{E}}_{\mathcal{D}_\mathtt{Rep}}\left[\chi^2\left(\mathbb{O}_{\hat{\gamma}}(\cdot\wvert s)\wVert\mathbb{O}(\cdot\wvert s)\right)\right]\leq\mathcal{O}\left(\frac{(1+B^-_\mathcal{F})(1+B_\mathcal{F})^2}{N_{\rm p}T_{\rm p}H}\cdot\log(N_{\rm p}T_{\rm p}H|\mathcal{F}|/\delta)\right).
 	\label{eq:chisquareerror}
 \end{equation}
Combine \eqref{eq:chisquareerror} and the divergence inequalities \citep[\S 7.6,][]{polyanskiy2022information}, we have
\begin{align*}
 &\bar{\mathbb{E}}_{\mathcal{D}_\mathtt{Rep}}\left[D_{\rm TV}\left(\mathbb{O}_{\hat{\gamma}}(\cdot\wvert s)\wVert\mathbb{O}(\cdot\wvert s)\right)\right]\leq\frac{1}{2}\cdot\bar{\mathbb{E}}_{\mathcal{D}_\mathtt{Rep}}\left[\sqrt{\chi^2\left(\mathbb{O}_{\hat{\gamma}}(\cdot\wvert s)\wVert\mathbb{O}(\cdot\wvert s)\right)}\right]\\
 &\hspace{0.5cm}\leq\frac{1}{2}\cdot\sqrt{\bar{\mathbb{E}}_{\mathcal{D}_\mathtt{Rep}}\left[\chi^2\left(\mathbb{O}_{\hat{\gamma}}(\cdot\wvert s)\wVert\mathbb{O}(\cdot\wvert s)\right)\right]}\leq\mathcal{O}\left(\frac{B_\mathcal{F}(B^-_\mathcal{F})^{1/2}}{(N_{\rm p}T_{\rm p}H)^{1/2}}\sqrt{\log(N_{\rm p}T_{\rm p}H|\mathcal{F}_\gamma|/\delta)}\right),
\end{align*}
where the second inequality follows $\mathbb{E}[X]\leq\sqrt{\mathbb{E}[X^2]}$ and we finish the proof of Theorem \ref{thm:translatorpretrain}.\hfill$\Box$

\subsection{Proof of Theorem \ref{thm:practical regret}}\label{ap:practicalregret}
\paragraph{Notations.} Denote $(\mathcal{J},\mathcal{\hat{J}})$, $({\pi}_z^*,\hat{\pi}_z^*)$, and $(\mathbb{P}_{z,h},\mathbb{\hat{P}}_{z,h})$ as the value functions, optimal policies, and probability distributions under the environment concerning the ground-truth $\mathbb{O}$ and the pretrained $\mathbb{O}_{\hat{\gamma}}$. Furthermore, $(\pi^t,\hat{\pi}^t)$ are the \planner's policy empowered by perfect $\mathtt{LLM}$ or pretrained $\mathtt{LLM}_{\hat{\theta}}$.  \\[10pt]
\emph{Proof of Theorem \ref{thm:practical regret}.} Conditioned on the event $\cE_1$ that both Theorem \ref{thm:llmpretrain} and \ref{thm:translatorpretrain} hold, the regret under the practical setting can be decomposed as
\begin{align}
	{\rm Reg}_z(T)&\leq\underbrace{\sum_{t=1}^T\mathcal{\hat{J}}_z(\hat{\pi}_z^*,\omega^t)-\mathcal{{J}}_z(\hat{\pi}_z^*,\omega^t)}_{\textbf{(\romannumeral1})}+\underbrace{\sum_{t=1}^T\mathcal{{J}}_z(\hat{\pi}_z^*,\omega^t)-\mathcal{{J}}_z({\pi}_z^*,\omega^t)}_{\textbf{(\romannumeral2})}\nonumber\\
	&\quad+\underbrace{\sum_{t=1}^T\mathcal{{J}}_z({\pi}_z^*,\omega^t)-\mathcal{\hat{J}}_z({\pi}_z^*,\omega^t)}_{\textbf{(\romannumeral3})}+\underbrace{\sum_{t=1}^T\EE_{\cH_t}\left[\mathcal{\hat{J}}_z({\pi}_z^*,\omega^t)-\mathcal{\hat{J}}_z(\hat{\pi}^t,\omega^t)\right]}_{\textbf{(\romannumeral4})},
	\label{eq:regretdecomposition}
\end{align} 
and $\textbf{(\romannumeral2})\leq0$ results from the optimality such that $\mathcal{{J}}_z(\hat{\pi}_z^*,\omega^t)\leq\mathcal{{J}}_z({\pi}_z^*,\omega^t)$ for all $t\in[T]$.
\vspace{10pt}

\noindent\textbf{Step 1. Bound (\romannumeral1) and (\romannumeral3) with Translator's Pretraining Error.}\\[5pt]
For any policy sequence $\{\pi_t\}_{t\leq T}\subseteq\Pi$ and length $T\in\mathbb{N}$, based on PDL in Lemma \ref{lem:pdl}, we have
\begin{align}
	&\sum_{t=1}^T\mathcal{\hat{J}}_z(\pi_t,\omega^t)-\mathcal{{J}}_z(\pi_t,\omega^t)\nonumber\\
	&\quad=\sum_{t=1}^T\sum_{h=1}^H\mathbb{E}_{(s_h^t,\tau_{h}^t,g_h^t)\sim\mathbb{P}_{z}^{\pi_t}}\left[(\mathbb{P}_{z,h}\hat{V}_{h}^{\pi_t}-\mathbb{\hat{P}}_{z,h}\hat{V}_{h}^{\pi_t})(s_h^t,\tau_{h}^t,g_h^t,\omega^t)\right]\nonumber\\
	&\quad\leq H\sum_{t=1}^T\sum_{h=1}^H\mathbb{E}_{(s_h^t,\tau_{h}^t,g_h^t)\sim\mathbb{P}_{z}^{\pi_t}}\left[D_{\rm TV}\left(\mathbb{P}_{z,h}(\cdot,\cdot\wvert s_h^t,\tau_{h}^t,g_h^t),\mathbb{\hat{P}}_{z,h}(\cdot,\cdot\wvert s_h^t,g_h^t,\tau_{h}^t)\right)\right]\nonumber\\
	&\quad\leq H\sum_{t=1}^T\sum_{h=1}^H\mathbb{E}_{(s_h^t,g_h^t)\sim\mathbb{P}_{z}^{\pi_t}}\EE_{s_{h+1}^t\sim\mathbb{P}_{z,h}(\cdot\wvert s_h^t,g_h^t)}\left[D_{\rm TV}\left(\OO(\cdot|s_{h+1}^t),\mathbb{O}_{\hat{\gamma}}(\cdot|s_{h+1}^t)\right)\right],
	\label{eq:practical13}
\end{align}
where the last inequality results from the fact that for any $f$-divergence, it holds that
$$
D_f(\PP_{Y|X}\otimes\PP_X,\QQ_{Y|X}\otimes\PP_X)=\EE_{X\sim \PP_X}[D_f(\PP_{Y|X},\QQ_{Y|X})].
$$ 
Based on \eqref{eq:practical13}, by taking policies $\pi=\hat{\pi}_z^*$ and $\pi={\pi}_z^*$ respectively, we have
\begin{align}
	\textbf{\small(\romannumeral1})+\textbf{\small(\romannumeral3})&=\sum_{t=1}^T\mathcal{\hat{J}}_z(\hat{\pi}_z^*,\omega^t)-\mathcal{{J}}_z(\hat{\pi}_z^*,\omega^t)+\sum_{t=1}^T\mathcal{{J}}_z({\pi}_z^*,\omega^t)-\mathcal{\hat{J}}_z({\pi}_z^*,\omega^t)\nonumber\\
	&\leq2H^2T\cdot\max_{s\in\mathcal{S}}\left\{ D_{\rm TV}\left(\mathbb{O}(\cdot\wvert s),\mathbb{O}_{\hat{\gamma}}(\cdot\wvert s)\right)\right\}\leq2H^2T\lambda_R^{-1}\cdot\Delta_\mathtt{Rep}(N_{\rm p},T_{\rm p},H,\delta),
	\label{eq:term13}
\end{align}
where the last inequality results from Assumption \ref{as:onlinecoverage} and Theorem \ref{thm:translatorpretrain}.\\[10pt]
\textbf{Step 2. Bound (\romannumeral4) with LLM's and Translator's Pretraining Errors}\\[5pt]
Recall that the \planner\ follows a mixture policy of $\pi_\mathtt{exp}$ and $\hat{\pi}_\mathtt{LLM}$ as
\begin{equation}
\pi_h^t(\cdot\wvert \tau_{h}^t,\omega^t)\sim(1-\epsilon)\cdot\hat{\pi}_{h,\mathtt{LLM}}^t(\cdot\wvert \tau_{h}^t,\omega^t)+\epsilon\cdot\pi_{h,\mathtt{exp}}(\cdot|\tau_h^t).
\end{equation}
Based on PDL in Lemma \ref{lem:pdl}, the performance difference in term \textbf{\small(\romannumeral4}) can be decomposed as
\begin{align}
	\textbf{\small(\romannumeral4})&=\sum_{t=1}^T\sum_{h=1}^H\EE_{\cH_t\sim\bigotimes_{i=1}^{t-1}\mathbb{\hat{P}}_{z}^{\hat\pi_i}}\mathbb{E}_{(s_h^t,\tau_{h}^t)\sim\mathbb{\hat{P}}_{z}^{\hat{\pi}^t}}\left[(\pi^*_{z,h}-\hat{\pi}_h^t)\hat{Q}_h^{\pi_z^*}(s_h^t,\tau_{h}^t,\omega^t)\right]\nonumber \\
	&=\sum_{t=1}^T\sum_{h=1}^H\EE_{\cH_t\sim\bigotimes_{i=1}^{t-1}\mathbb{\hat{P}}_{z}^{\hat\pi_i}}\mathbb{E}_{(s_h^t,\tau_{h}^t)\sim\mathbb{\hat{P}}_{z}^{\hat{\pi}^t}}\left[\left(\pi^*_{z,h}-{\hat{\pi}^t_{h,\mathtt{LLM}}}\right)\hat{Q}_h^{\pi_z^*}(s_h^t,\tau_{h}^t,\omega^t)\right]\cdot(1-\epsilon)\nonumber\\
	&\quad+\sum_{t=1}^T\sum_{h=1}^H\EE_{\cH_t\sim\bigotimes_{i=1}^{t-1}\mathbb{\hat{P}}_{z}^{\hat\pi_i}}\mathbb{E}_{(s_h^t,\tau_{h}^t)\sim\mathbb{\hat{P}}_{z}^{\hat{\pi}^t}}\left[\left(\pi^*_{z,h}-\pi_{h,\mathtt{exp}}\right)\hat{Q}_h^{\pi_z^*}(s_h^t,\tau_{h}^t,\omega^t)\right]\cdot\epsilon\nonumber\\
	&\leq H\sum_{t=1}^T\sum_{h=1}^H\EE_{\cH_t\sim\bigotimes_{i=1}^{t-1}\mathbb{\hat{P}}_{z}^{\hat\pi_i}}\mathbb{E}_{\tau_{h}^t\sim\mathbb{\hat{P}}_{z}^{\hat{\pi}^t}}\left[ D_{\rm TV}\left(\pi_{z,h}^*(\cdot\wvert \tau_{h}^t,\omega^t),\mathtt{LLM}_{\hat{\theta}}(\cdot\wvert\prompt_h^t)\right)\right]+HT\epsilon
	\label{eq:term4bound}
\end{align}
where we write $\pi_hQ_{h}(s_h,\tau_{h},\omega)=\langle \pi_h(\cdot|\tau_h,\omega),Q_{h}(s_h,\tau_{h},\cdot,\omega)\rangle_\cG$ for all $h\in[H]$, and $\hat{Q}_h^{\pi}$ denotes the action value function under the practical setting. Furthermore, we have
\begin{align}
	\sum_{t=1}^T&\sum_{h=1}^H\EE_{\cH_t\sim\bigotimes_{i=1}^{t-1}\mathbb{\hat{P}}_{z}^{\hat\pi_i}}\mathbb{E}_{\tau_{h}^t\sim\mathbb{\hat{P}}_{z}^{\hat{\pi}^t}}\left[D_{\rm TV}\left(\pi_{z,h}^*(\cdot\wvert \tau_{h}^t,\omega^t),\mathtt{LLM}_{\hat{\theta}}(\cdot\wvert\prompt_h^t)\right)\right]\nonumber\\
	&\leq\sum_{t=1}^T\sum_{h=1}^H\EE_{\cH_t\sim\bigotimes_{i=1}^{t-1}\mathbb{\hat{P}}_{z}^{\hat\pi_i}}\mathbb{E}_{\tau_{h}^t\sim\mathbb{\hat{P}}_{z}^{\hat{\pi}^t}}\left[D_{\rm TV}\left(\mathtt{LLM}_{\hat{\theta}}(\cdot\wvert\prompt_h^t),\mathtt{LLM}(\cdot\wvert\prompt_h^t)\right)\right]\nonumber\\
	&\quad+\sum_{t=1}^T\sum_{h=1}^H\EE_{\cH_t\sim\bigotimes_{i=1}^{t-1}\mathbb{\hat{P}}_{z}^{\hat\pi_i}}\mathbb{E}_{\tau_{h}^t\sim\mathbb{\hat{P}}_{z}^{\hat{\pi}^t}}\left[D_{\rm TV}\left(\pi_{z,h}^*(\cdot\wvert \tau_{h}^t,\omega^t),\mathtt{LLM}(\cdot\wvert\prompt_h^t)\right)\right]\nonumber\\
	&\leq\sum_{t=1}^T\sum_{h=1}^H\EE_{\cH_t\sim\bigotimes_{i=1}^{t-1}\mathbb{\hat{P}}_{z}^{\hat\pi_i}}\mathbb{E}_{\tau_{h}^t\sim\mathbb{\hat{P}}_{z}^{\hat{\pi}^t}}\left[D_{\rm TV}\left(\mathtt{LLM}_{\hat{\theta}}(\cdot\wvert\prompt_h^t),\mathtt{LLM}(\cdot\wvert\prompt_h^t)\right)\right]\nonumber\\
	&\quad+\sum_{t=1}^T\sum_{h=1}^H\EE_{\cH_t\sim\bigotimes_{i=1}^{t-1}\mathbb{\hat{P}}_{z}^{\hat\pi_i}}\mathbb{E}_{\tau_{h}^t\sim\mathbb{\hat{P}}_{z}^{\hat{\pi}^t}}\left[\sum_{z'\ne z}\PP_{\cD}(z'\wvert\prompt_h^t)\right],
	\label{eq:term4concenbound}
	\end{align}
where the first inequality arises from the triangle inequality, and the second inequality results from Thoerem \ref{thm:BAIL}. Furthermore, the first term can be bounded by the pretraining error, following 
	\begin{align}
	&\sum_{t=1}^T\sum_{h=1}^H\EE_{\cH_t\sim\bigotimes_{i=1}^{t-1}\mathbb{\hat{P}}_{z}^{\hat\pi_i}}\mathbb{E}_{\tau_{h}^t\sim\mathbb{\hat{P}}_{z}^{\hat{\pi}^t}}\left[D_{\rm TV}\left(\mathtt{LLM}_{\hat{\theta}}(\cdot\wvert\prompt_h^t),\mathtt{LLM}(\cdot\wvert\prompt_h^t)\right)\right]\nonumber\\
	&\quad\leq\lambda_S\cdot\sum_{t=1}^T\sum_{h=1}^H\bar{\mathbb{E}}_{\prompt_h^t\sim\mathcal{D}_\mathtt{LLM}}\left[D_{\rm TV}\left(\mathtt{LLM}_{\hat{\theta}}(\cdot\wvert\prompt_h^t),\mathtt{LLM}(\cdot\wvert\prompt_h^t)\right)\right],\nonumber\\
	&\quad=\lambda_SHT\cdot\Delta_\mathtt{LLM}(N_{\rm p},T_{\rm p},H,\delta),
	\label{eq:term4llmbound}
\end{align}
where the last inequality follows  Theorem \ref{thm:llmpretrain} and Assumption \ref{as:onlinecoverage}. Under practical setting, $\prompt_h^t$  is generated from practical transition $\mathbb{\hat{P}}_{z}$, mismatching $\PP_{\cD}(z\wvert\prompt_h^t)$ in pretraining. Let $\mathcal{X}^t_{\mathtt{exp}}=\{i\in[t]:\hat{\pi}^i=\pi_{\mathtt{exp}}\}$ and write $\breve{\tau}_{h/t}^{i}={\tau}_H^i$ for all $i<t$ and $\breve{\tau}_{h/t}^{t}={\tau}_{h}^t$. Define the information gains as
\begin{equation}
	L_{h,t}^\mathtt{exp}(z')=\sum_{i\in\mathcal{X}^t_{\mathtt{exp}}}\log\left(\frac{\mathbb{P}_{z'}(\breve{\tau}_{h/t}^{i})}{\mathbb{P}_{z}(\breve{\tau}_{h/t}^{i})}\right),\quad L_{h,t}^\mathtt{LLM}(z')=\sum_{i\in[t]\backslash\mathcal{X}^t_{\mathtt{exp}}}\log\left(\frac{\mathbb{P}_{z'}(\breve{\tau}_{h/t}^{i})}{\mathbb{P}_{z}(\breve{\tau}_{h/t}^{i})}\right),
	\label{eq:logprob}
\end{equation}
where $\mathbb{P}_{z}(\tau_h)$ is defined in \eqref{eq:defpomdp}. Based on the law of total probability, we have
\begin{equation}
	\PP_{\cD}(z'\wvert\prompt_h^t)=\frac{\mathbb{P}_{z'}(\prompt_h^t)\cdot\mathcal{P}_\mathcal{Z}(z')}{\sum_{\tilde{z}\in\cZ}\mathbb{P}_{\tilde{z}}(\prompt_h^t)\cdot\mathcal{P}_\mathcal{Z}(\tilde{z})}\leq\frac{\mathbb{P}_{z'}(\prompt_h^t)}{\mathbb{P}_z(\prompt_h^t)}\cdot\frac{\mathcal{P}_\mathcal{Z}(z')}{\mathcal{P}_\mathcal{Z}(z)}.
	\label{eq:bayes_uppperbound}
\end{equation}
Let $\cE_2$ be the event that Lemma \ref{lem:practicalconcentration} holds. Based on \eqref{eq:bayes_uppperbound}, \eqref{eq:logprob} and conditioned on event $\cE_2$, it holds that
\begin{align}
 	&\sum_{z'\ne z}\PP_{\cD}(z'\wvert\prompt_h^t)\leq\min\left\{\sum_{z'\ne z}\frac{\mathbb{P}_{z'}(\prompt_h^t)}{\mathbb{P}_z(\prompt_h^t)}\cdot\frac{\mathcal{P}_\mathcal{Z}(z')}{\mathcal{P}_\mathcal{Z}(z)},1\right\}\nonumber\\
 	&\leq\min\left\{c_\mathcal{Z}\sum_{z'\ne z}\exp\left(L_{h,t}^\mathtt{exp}(z')+L_{h,t}^\mathtt{LLM}(z')\right),1\right\}\nonumber\\
 	&\leq\min\left\{c_\mathcal{Z}\sum_{z'\ne z}\exp\Big(t\cdot H\lambda^{-1}_{R}\Delta_\mathtt{Rep}(N_{\rm p},T_{\rm p},H,\delta)^2-2\eta|\mathcal{X}^t_{\mathtt{exp}}|+8\log(|\mathcal{Z}|/\delta)+2\eta\Big),1\right\}\nonumber\\
 	&\leq\min\left\{c_\mathcal{Z}\sum_{z'\ne z}\exp\Big(-\left(\eta\epsilon-H\lambda^{-1}_{R}\Delta_\mathtt{Rep}(N_{\rm p},T_{\rm p},H,\delta)^2\right)t+8\log(|\mathcal{Z}|/\delta)+2\eta\Big),1\right\}\nonumber\\
 	&\leq\min\left\{c_\mathcal{Z}\cdot\exp\Big(-\left(\eta\epsilon-H\lambda^{-1}_{R}\Delta_\mathtt{Rep}(N_{\rm p},T_{\rm p},H,\delta)^2\right)t+9\log(|\mathcal{Z}|/\delta)+2\eta\Big),1\right\}
 	\label{eq:term4concentration}
\end{align}
for all $(h,t)\in[H]\times[T]$, where the second inequality follows Assumption \ref{as:coverage}. Here, we suppose that $|\mathcal{X}^t_{\mathtt{exp}}|/t=\epsilon$ for simplicity, which is attainable if we explore at a fixed fraction during episodes. Assume that $\eta\epsilon\geq H\lambda^{-1}_{R}\Delta_\mathtt{Rep}(N_{\rm p},T_{\rm p},H,\delta)^2$ holds temporarily. Following \eqref{eq:term4concentration} and condition on event $\cE_2$, there exists a large constant $c_0>0$ such that
\begin{align}
\sum_{t=1}^T&\sum_{h=1}^H\sum_{z'\ne z}\PP_{\cD}(z'\wvert\prompt_h^t)\leq c_0\cdot H\log(c_\mathcal{Z}|\mathcal{Z}|/\delta)\cdot\left(\eta\epsilon-H\lambda^{-1}_{R}\Delta_\mathtt{Rep}(N_{\rm p},T_{\rm p},H,\delta)^2\right)^{-1},
	\label{eq:term4concentrationsum}
\end{align}
where we use the fact that there exists constant $c_0>0$ such that $\sum_{t=1}^T \min\{c_3\exp(-c_1t+c_2),1\}\leq c_0\cdot c_1^{-1}(c_2+\log c_3)$ for $c_1\leq1$. Furthermore, based on \eqref{eq:term4concentrationsum}, we can show that
\begin{align}
	\sum_{t=1}^T&\sum_{h=1}^H\EE_{\cH_t\sim\bigotimes_{i=1}^{t-1}\mathbb{\hat{P}}_{z}^{\hat\pi_i}}\mathbb{E}_{\tau_{h}^t\sim\mathbb{\hat{P}}_{z}^{\hat{\pi}^t}}\left[\sum_{z'\ne z}\PP_{\cD}(z'\wvert\prompt_h^t)\right]\nonumber\\
	&\leq\sum_{t=1}^T\sum_{h=1}^H\sum_{z'\ne z}\EE_{\cH_t\sim\bigotimes_{i=1}^{t-1}\mathbb{\hat{P}}_{z}^{\hat\pi_i}}\mathbb{E}_{\tau_{h}^t\sim\mathbb{\hat{P}}_{z}^{\hat{\pi}^t}}\left[\PP_{\cD}(z'\wvert\prompt_h^t)\ind\left(\cE_2\text{~holds}\right)\right]+2HT\delta\nonumber\\
	&\leq c_0\cdot H\log(c_\mathcal{Z}|\mathcal{Z}|/\delta)\cdot\left(\eta\epsilon-H\lambda^{-1}_{R}\Delta_\mathtt{Rep}(N_{\rm p},T_{\rm p},H,\delta)^2\right)^{-1}+2HT\delta.
	\label{eq:term4conditionconcen}
\end{align} 
Combine \eqref{eq:term4bound}, \eqref{eq:term4concentration}, \eqref{eq:term4llmbound} and \eqref{eq:term4conditionconcen}, it holds that
\begin{align}
	\textbf{\small(\romannumeral4})&\leq \underbrace{c_0\cdot H^2\log(c_\mathcal{Z}|\mathcal{Z}|/\delta)\cdot\left(\eta\epsilon-H\lambda^{-1}_{R}\cdot\Delta_\mathtt{Rep}(N_{\rm p},T_{\rm p},H,\delta)^2\right)^{-1}}_{\textbf{(\romannumeral5})}\nonumber\\
	&\qquad+\underbrace{HT\eta^{-1}\left(\eta\epsilon-H\lambda^{-1}_{R}\cdot\Delta_\mathtt{Rep}(N_{\rm p},T_{\rm p},H,\delta)^2\right)}_{\textbf{(\romannumeral6})}+\lambda_SH^2T\cdot\Delta_\mathtt{LLM}(N_{\rm p},T_{\rm p},H,\delta)\nonumber\\
	&\qquad+H^2T(\eta\lambda_R)^{-1}\cdot\Delta_\mathtt{Rep}(N_{\rm p},T_{\rm p},H,\delta)^2+2HT\delta,
	\label{eq:term4}
\end{align} 
If we explore with probability $\epsilon=H(\eta\lambda_R)^{-1}\cdot\Delta_\mathtt{Rep}(N_{\rm p},T_{\rm p},H,\delta)^2+(H\log\left(c_\mathcal{Z}|\mathcal{Z}|/\delta\right)/T\eta)^{1/2}$, which satisfies the condition that $\eta\epsilon\geq H\lambda^{-1}_{R}\Delta_\mathtt{Rep}(N_{\rm p},T_{\rm p},H,\delta)^2$ assumed in \eqref{eq:term4concentration}, then we have
\begin{align}
	\textbf{\small(\romannumeral5})+\textbf{\small(\romannumeral6})\leq\cO\left(H^\frac{3}{2}\sqrt{\log(c_\mathcal{Z}|\mathcal{Z}|/\delta)\cdot T/\eta}\right).
	\label{eq:term56}
\end{align}
\noindent\textbf{Step 3. Conclude the Proof based on Step 1 and Step 2.}\\[5pt]
Combine \eqref{eq:regretdecomposition}, \eqref{eq:term13}, \eqref{eq:term4} and \eqref{eq:term56}, the regret under the practical setting follows
\begin{align}
	{\rm Reg}_z(T)&\leq\textbf{\small(\romannumeral1})+\textbf{\small(\romannumeral3})+\textbf{\small(\romannumeral4})+HT\cdot\PP(\cE_1\text{~fails})\nonumber\\
	&=\mathcal{O}\underbrace{\Big(H^\frac{3}{2}\sqrt{\log(c_\mathcal{Z}|\mathcal{Z}|/\delta)\cdot T/\eta}}_{\displaystyle\text{Planning error}}+\underbrace{H^2T\cdot\Delta_{\rm p}(N_{\rm p},T_{\rm p},H,\delta,\xi)\Big)}_{\displaystyle\text{Pretraining error}}+4HT\delta,
	\label{eq:practicalregret}
\end{align}
where the cumulative pretraining error of the imperfectly pretrained PAR system follows 
\begin{align*}
	\Delta_{\rm p}&(N_{\rm p},T_{\rm p},H,\delta,\xi)=(\eta\lambda_R)^{-1}\cdot\Delta_\mathtt{Rep}(N_{\rm p},T_{\rm p},H,\delta)^2\\
	&+2\lambda_R^{-1}\cdot\Delta_\mathtt{Rep}(N_{\rm p},T_{\rm p},H,\delta)+\lambda_S\cdot\Delta_\mathtt{LLM}(N_{\rm p},T_{\rm p},H,\delta).
\end{align*} 
Here, $\xi=(\eta,\lambda_S,\lambda_R)$ denotes the set of distinguishability and coverage coefficients in Definition \ref{def:iden} and Assumption \ref{as:onlinecoverage}, and $\Delta_\mathtt{LLM}(N_{\rm p},T_{\rm p},H,\delta)$ and $\Delta_\mathtt{Rep}(N_{\rm p},T_{\rm p},H,\delta)$ are  pretraining errors defined in Theorem \ref{thm:llmpretrain} and Theorem \ref{thm:translatorpretrain}. By taking $\delta=1/\sqrt{T}$, we complete the proof of Theorem \ref{thm:practical regret}. \hfill$\Box$

\subsection{Proof of Lemma \ref{lem:practicalconcentration}}
In this subsection, we provide a detailed examination of posterior concentration when there exists a mismatch between the ground-truth environment and the pretrained environment. 
\begin{lemma}\label{lem:practicalconcentration}
Suppose that Assumption \ref{as:coverage} and Theorem \ref{thm:translatorpretrain} hold. For all $(z',h,t)\in\mathcal{Z}\times[H]\times[T]$, with probability at least $1-2\delta$, it holds that
\vspace{2pt}
\begin{flalign*}
	&\qquad\text{(\romannumeral1).}\ L_{h,t}^\mathtt{LLM}(z')\leq\left(t-|\mathcal{X}^t_{\mathtt{exp}}|\right)H\lambda^{-1}_{R}\cdot\Delta_\mathtt{Rep}(N_{\rm p},T_{\rm p},H,\delta)^2+4\log(|\mathcal{Z}|/\delta),&\\[2pt]
	&\qquad\text{(\romannumeral2).}\ L_{h,t}^\mathtt{exp}(z')\leq|\mathcal{X}^t_{\mathtt{exp}}|H\lambda^{-1}_{R}\cdot\Delta_\mathtt{Rep}(N_{\rm p},T_{\rm p},H,\delta)^2+4\log(|\mathcal{Z}|/\delta)-2\eta\cdot|\mathcal{X}^t_{\mathtt{exp}}|+2\eta,
\end{flalign*}
\vspace{2pt}
where $L_{h,t}^\mathtt{LLM}(z')$ and $L_{h,t}^\mathtt{exp}(z')$ are the information gain defined in \eqref{eq:logprob}.
\end{lemma} 
\noindent\emph{Proof of Lemma \ref{lem:practicalconcentration}}. Let $\mathfrak{F}_{t}$ be the filtration induced by $\{\omega^i,\tau_H^i\}_{i<t}\cup\{\ind(\pi^i=\pi_\texttt{exp})\}_{i\in[t]}$. Consider a fixed tuple $(z',h,t)\in\mathcal{Z}\times[H]\times[T]$, it holds that\
\begin{align*}
	&\mathbb{\hat{P}}_z(L_{h,t}^\mathtt{LLM}(z')\geq\beta_{h,t}^\mathtt{LLM})\leq\underset{\lambda\geq0}{\inf}\ \mathbb{E}_{\mathfrak{F}_{1:t}}\left[\exp(\lambda\cdot(L_{h,t}^\mathtt{LLM}(z')-\beta_{h,t}^\mathtt{LLM}))\right]\nonumber\\
	&\quad=\underset{\lambda\geq0}{\inf}\ \mathbb{E}_{\bigotimes_{i\in[t]\backslash\mathcal{X}^t_{\mathtt{exp}}}\mathbb{\hat{P}}_z^{\hat{\pi}_i}}\left[\exp\left(\sum_{i\in[t]\backslash\mathcal{X}^t_{\mathtt{exp}}}\lambda\cdot\log\left(\frac{\mathbb{P}_{z'}^{\hat{\pi}^i}(\breve{\tau}_{h/t}^{i})}{\mathbb{P}_{z}^{\hat{\pi}^i}(\breve{\tau}_{h/t}^{i})}\right)-\lambda\cdot\beta_{h,t}^\mathtt{LLM}\right)\right]\nonumber\\
	&\quad=\underset{\lambda\geq0}{\inf}\ \prod_{i\in[t]\backslash\mathcal{X}^t_{\mathtt{exp}}}\mathbb{E}_{\mathbb{{P}}_z^{\hat{\pi}^i}}\left[\left(\frac{\mathbb{P}_{z'}^{\hat{\pi}^i}(\breve{\tau}_{h/t}^{i})}{\mathbb{{P}}_{z}^{\hat{\pi}^i}(\breve{\tau}_{h/t}^{i})}\right)^\lambda\cdot\frac{\mathbb{\hat{P}}_{z'}^{\hat{\pi}^i}(\breve{\tau}_{h/t}^{i})}{\mathbb{P}_{z}^{\hat{\pi}^i}(\breve{\tau}_{h/t}^{i})}\right]\cdot\exp\left(-\lambda\cdot\beta_{h,t}^\mathtt{LLM}\right)\nonumber\\
	&\quad\leq\underset{\lambda\geq0}{\inf}\ \prod_{i\in[t]\backslash\mathcal{X}^t_{\mathtt{exp}}}\mathbb{E}_{\mathbb{{P}}_z^{\hat{\pi}^i}}\left[\left(\frac{\mathbb{P}_{z'}^{\hat{\pi}^i}(\breve{\tau}_{h/t}^{i})}{\mathbb{{P}}_{z}^{\hat{\pi}^i}(\breve{\tau}_{h/t}^{i})}\right)^{2\lambda}\right]^{1/2}\mathbb{E}_{\mathbb{{P}}_z^{\hat{\pi}^i}}\left[\left(\frac{\mathbb{\hat{P}}_{z'}^{\hat{\pi}^i}(\breve{\tau}_{h/t}^{i})}{\mathbb{P}_{z}^{\hat{\pi}^i}(\breve{\tau}_{h/t}^{i})}\right)^2\right]^{1/2}\cdot\exp\left(-\lambda\cdot\beta_{h,t}^\mathtt{LLM}\right),
\end{align*}
where the first inequality is a natural corollary to Lemma \ref{lem:exptrick}, and the last inequality follows the Cauchy-Swartz inequality. By taking $\lambda=\frac{1}{4}$, for all $(h,t)\in[H]\times[T]$, we have
\begin{align}
&\mathbb{E}_{\mathbb{{P}}_z^{\hat{\pi}^i}}\left[\left(\frac{\mathbb{P}_{z'}^{\hat{\pi}^i}(\breve{\tau}_{h/t}^{i})}{\mathbb{{P}}_{z}^{\hat{\pi}^i}(\breve{\tau}_{h/t}^{i})}\right)^{1/2}\right]^{1/2}\mathbb{E}_{\mathbb{{P}}_z^{\hat{\pi}^i}}\left[\left(\frac{\mathbb{\hat{P}}_{z'}^{\hat{\pi}^i}(\breve{\tau}_{h/t}^{i})}{\mathbb{P}_{z}^{\hat{\pi}^i}(\breve{\tau}_{h/t}^{i})}\right)^2\right]^{1/2}\leq\sqrt{1+\chi^2\big(\mathbb{P}_{z'}^{\hat{\pi}^i}(\breve{\tau}_{h/t}^{i})\wVert\mathbb{\hat{P}}_{z}^{\hat{\pi}^i}(\breve{\tau}_{h/t}^{i})\big)}.
\label{eq:concentrationbound}
\end{align}
Based on Theorem \ref{thm:translatorpretrain} and Assumption \ref{as:coverage}, for any policy $\pi\in\Pi$, it holds that
\begin{align}
	1&+\chi^2\big(\mathbb{P}_{z'}^\pi(\tau_h)\wVert\mathbb{\hat{P}}_{z}^{\pi}(\tau_h)\big)\leq 1+\chi^2\big(\mathbb{P}_{z'}^\pi(\tau_h,s_{1:h})\wVert\mathbb{\hat{P}}_{z}^{\pi}(\tau_h,s_{1:h})\big)\nonumber\\
	&\leq 1+\chi^2\left(\prod_{h'=1}^h\mathbb{P}_{z'}^\pi(g_h,s_{h+1}\wvert\tau_h,s_h)\cdot\OO(o_h\wvert s_h)\wVert\prod_{h'=1}^h\mathbb{P}_{z'}^\pi(g_h,s_{h+1}\wvert\tau_h,s_h)\cdot\mathbb{O}_{\hat{\gamma}}(o_h\wvert s_h)\right)\nonumber\\
	&\leq\big(1+\max_{s\in\mathcal{S}}\big\{\chi^2\big(\mathbb{O}(\cdot\wvert s)\wVert\mathbb{O}_{\hat{\gamma}}(\cdot\wvert s)\big)\big\}\big)^H\leq\left(1+\lambda^{-1}_{R}\cdot\Delta_\mathtt{Rep}(N_{\rm p},T_{\rm p},H,\delta)^2\right)^H,
	\label{eq:concentrationboundllmchi}
\end{align}
where the first inequality follows data processing inequality and the second inequality arises from the tensorization \cite[Theorem 7.32 and \S7.12,][]{polyanskiy2022information}. To ensure that $L_{h,t}^\mathtt{LLM}(z')\leq\beta_{h,t}^\mathtt{LLM}$ holds for all $(z',h,t)\in\mathcal{Z}\times[H]\times[T]$ with probability at least $1-\delta$, we let 
$$
\prod_{i\in[t]\backslash\mathcal{X}^t_{\mathtt{exp}}}\sqrt{1+\chi^2\big(\mathbb{P}_{z'}^{\hat{\pi}^i}(\breve{\tau}_{h/t}^{i})\wVert\mathbb{\hat{P}}_{z}^{\hat{\pi}^i}(\breve{\tau}_{h/t}^{i})\big)}\cdot\exp\left(-\frac{\beta_{h,t}^\mathtt{LLM}}{4}\right)=\frac{\delta}{|\cZ|},
$$
with a union bound taken over $\mathcal{Z}$, since Lemma \ref{lem:exptrick} has ensured the inequality holds for all $(h,t)\in[H]\times[T]$. Thus, the constant $\beta_{h,t}^\mathtt{LLM}$ is then chosen as
\begin{align*}
	\beta_{h,t}^\mathtt{LLM}&=2\sum_{i\in[t]\backslash\mathcal{X}^t_{\mathtt{exp}}}\log\left(1+\chi^2\big(\mathbb{P}_{z'}^{\hat{\pi}^i}(\breve{\tau}_{h/t}^{i})\wVert\mathbb{\hat{P}}_{z}^{\hat{\pi}^i}(\breve{\tau}_{h/t}^{i})\big)\right)+4\log(|\mathcal{Z}|/\delta)\\
	&\leq\left(t-|\mathcal{X}^t_{\mathtt{exp}}|\right)\cdot H\log\left(1+\lambda^{-1}_{R}\cdot\Delta_\mathtt{Rep}(N_{\rm p},T_{\rm p},H,\delta)^2\right)+4\log(|\mathcal{Z}|/\delta)\nonumber\\
	&\leq\left(t-|\mathcal{X}^t_{\mathtt{exp}}|\right)\cdot H\lambda^{-1}_{R}\cdot\Delta_\mathtt{Rep}(N_{\rm p},T_{\rm p},H,\delta)^2+4\log(|\mathcal{Z}|/\delta),
\end{align*}
which is based on \eqref{eq:concentrationbound}, \eqref{eq:concentrationboundllmchi} by taking a union bound over $\mathcal{Z}$, and the last inequality results from $\log(1+x)\leq x$ for all $x\geq0$. Similarly, for the exploration episodes, we let
\begin{align*}
	\mathbb{\hat{P}}_z&(L_{h,t}^\mathtt{exp}(z')\geq\beta_{h,t}^\mathtt{exp})\leq\underset{\lambda\geq0}{\inf}\ \mathbb{E}\left[\exp(\lambda\cdot(L_{h,t}^\mathtt{exp}-\beta_{h,t}^\mathtt{exp}))\right]\nonumber\\
	&\quad\leq\prod_{i\in\mathcal{X}^t_{\mathtt{exp}}}\sqrt{1-D^2_{\rm H}\big(\mathbb{P}^{\hat{\pi}^i}_{z'}(\breve{\tau}_{h/t}^{i}),\mathbb{{P}}_{z}^{\hat{\pi}^i}(\breve{\tau}_{h/t}^{i})\big)}\cdot\sqrt{1+\chi^2\big(\mathbb{P}_{z'}^{\hat{\pi}^i}(\breve{\tau}_{h/t}^{i})\wVert\mathbb{\hat{P}}_{z}^{\hat{\pi}^i}(\breve{\tau}_{h/t}^{i})\big)}\cdot\exp\left(-\frac{1
	}{4}\beta_{h,t}^\mathtt{exp}\right).
\end{align*}
Furthermore, based on Definition \ref{def:iden}, the expolration episodes satisfies that
\begin{align}
	\sum_{i\in\mathcal{X}^t_{\mathtt{exp}}}&D^2_{\rm H}\big(\mathbb{P}_{z'}^{\hat{\pi}^i}(\breve{\tau}_{h/t}^{i}),\mathbb{P}_{z}^{\hat{\pi}^i}(\breve{\tau}_{h/t}^{i})\big)\geq\sum_{i\in\mathcal{X}^{t-1}_{\mathtt{exp}}}D^2_{\rm H}\big(\mathbb{P}_{z'}^{\hat{\pi}^i}(\tau_H),\mathbb{P}_{z}^{\hat{\pi}^i}(\tau_H)\big)\geq{\eta}\cdot|\mathcal{X}^{t-1}_{\mathtt{exp}}|.
	\label{eq:concentrationboundexphell}
\end{align}
To ensure that $L_{h,t}^\mathtt{exp}(z')\leq\beta_{h,t}^\mathtt{exp}$ holds for all $(z',h,t)\in\mathcal{Z}\times[H]\times[T]$ with high probability, we take 
$$
\prod_{i\in[t]\backslash\mathcal{X}^t_{\mathtt{exp}}}\sqrt{1-D^2_{\rm H}\big(\mathbb{P}^{\hat{\pi}^i}_{z'}(\breve{\tau}_{h/t}^{i}),\mathbb{{P}}_{z}^{\hat{\pi}^i}(\breve{\tau}_{h/t}^{i})\big)}\cdot\sqrt{1+\chi^2\big(\mathbb{P}_{z'}^{\hat{\pi}^i}(\breve{\tau}_{h/t}^{i})\wVert\mathbb{\hat{P}}_{z}^{\hat{\pi}^i}(\breve{\tau}_{h/t}^{i})\big)}\cdot\exp\left(-\frac{\beta_{h,t}^\mathtt{exp}}{4}\right)=\frac{\delta}{|\cZ|},
$$
with a union bound taken over $\mathcal{Z}$, and thus the constant $\beta_{h,t}^\mathtt{exp}$ is chosen as
\begin{align*}
	\beta_{h,t}^\mathtt{exp}&=2\sum_{i\in\mathcal{X}^t_{\mathtt{exp}}}\log\left(1-D^2_{\rm H}\left(\mathbb{P}_{z'}^{\hat{\pi}^i}(\breve{\tau}_{h/t}^{i}),\mathbb{{P}}_{z}^{\hat{\pi}^i}(\breve{\tau}_{h/t}^{i})\right)\right)\nonumber\\
	&\quad+2\sum_{i\in\mathcal{X}^t_{\mathtt{exp}}}\log\left(1+\chi^2\big(\mathbb{P}_{z'}^{\hat{\pi}^i}(\breve{\tau}_{h/t}^{i})\wVert\mathbb{\hat{P}}_{z}^{\hat{\pi}^i}(\breve{\tau}_{h/t}^{i})\big)\right)+4\log(|\mathcal{Z}|/\delta)\\
	&\leq|\mathcal{X}^t_{\mathtt{exp}}|\cdot H\log\left(1+\lambda^{-1}_{R}\cdot\Delta_\mathtt{Rep}(N_{\rm p},T_{\rm p},H,\delta)^2\right)+4\log(|\mathcal{Z}|/\delta)-2\eta\cdot|\mathcal{X}^{t-1}_{\mathtt{exp}}|\\
	&\leq|\mathcal{X}^t_{\mathtt{exp}}|\cdot H\lambda^{-1}_{R}\cdot\Delta_\mathtt{Rep}(N_{\rm p},T_{\rm p},H,\delta)^2+4\log(|\mathcal{Z}|/\delta)-2\eta\cdot(|\mathcal{X}^t_{\mathtt{exp}}|-1),
\end{align*}
where the first inequality results from \eqref{eq:concentrationboundllmchi}, \eqref{eq:concentrationboundexphell} and facts that $\log(1-x)\leq -x$ for all $x\leq1$ and $\log(1+x)\leq x$ for all $x\geq0$, and then we complete the proof of Lemma \ref{lem:practicalconcentration}.\hfill$\Box$

\subsection{Proof of Lemma \ref{lem:targetcontrastive}}
\begin{lemma}[Learning Target of Contrastive Loss] 
\label{lem:targetcontrastive}
For any observation-state pair $(o,s)\in\mathcal{O}\times\mathcal{S}$ sampled from the contrastive collection process, the learning target is 
$
f^*(o,s)={\mathbb{O}(o\wvert s)}/{\mathcal{P}^-(o)}.
$
\end{lemma}
\noindent\emph{Proof of Lemma \ref{lem:targetcontrastive}.} For any $(o,s)\in\mathcal{O}\times\mathcal{S}$, the posterior probability of label $y$ follows that
\begin{align*}
	&\DD(y\wvert o,s):=\mathbb{P}_{\cC}(y\wvert o,s)=\frac{\mathbb{P}_{\cC}(o\wvert s,y)\cdot\mathbb{P}_{\cC}(s\wvert y)}{\sum_{y\in\{0,1\}}\mathbb{P}_{\cC}(o\wvert s,y)\cdot\mathbb{P}_{\cC}(s\wvert y)},
\end{align*}
where the equation follows the Baye's Theorem and $\mathbb{P}_{\cC}(y=0)=\mathbb{P}_{\cC}(y=1)=1/2$. Moreover, the contrastive data collection process in \S\ref{sec:metric} indicates that
\begin{equation}
\mathbb{P}_{\cC}(\cdot\wvert s,y=0)=\mathbb{O}(\cdot\wvert s),\quad\mathbb{P}_{\cC}(\cdot\wvert s,y=1)=\mathcal{P}^-(\cdot),
\end{equation}
and data are labeled independent of data itself, such that $\mathbb{P}_{\cC}(s\wvert y)=\mathbb{P}_{\cC}(s)$. Thus,
$\mathbb{P}_{\cC}(y\wvert o,s)=\mathbb{P}_{\cC}(o\wvert s,y)/(\mathcal{P}^-(o)+\mathbb{O}(o\wvert s))$.
Recall that the population risk is 
$$
\mathcal{R}_{\rm CT}(\gamma;\mathcal{D}_\mathtt{Rep})=\mathbb{E}\left[D_{\rm KL}\left(\mathbb{D}_\gamma(\cdot| o,s)\wVert\DD(\cdot| o,s)\right)+\mathrm{Ent}(\DD(\cdot| o,s))\right].$$ 
As the minimum is attained at $\mathbb{D}_\gamma(\cdot\wvert o,s)=\DD(\cdot\wvert o,s)$. Following \eqref{eq:implieddistribution}, the learning target follows
\begin{equation}
\frac{\mathbb{P}_{\cC}(o\wvert s,y)}{\mathcal{P}^-(o)+\mathbb{O}(o\wvert s)}=\left(\frac{f^*(o, s)}{1+f^*(o, s)}\right)^y\left(\frac{1}{1+f^*(o, s)}\right)^{1-y}.
\label{eq:learningtarget}
\end{equation}
By solving the equation in \eqref{eq:learningtarget}, the learning target follows that $f^*(o,s)=\mathbb{O}(o\wvert s)/\mathcal{P}^-(o)$ for the contrastive loss in \eqref{eq:contrastive_loss},
and then we conclude the proof of Lemma \ref{lem:targetcontrastive}.\hfill$\Box$

%% file: tex/appendix/appendix_extention.tex
\section{Proof for Section \ref{sec:extention}: Extentions}

\subsection{Proof of Proposition \ref{thm:bawm}}\label{ap:bawm}
\begin{proof}[Proof of Proposition \ref{thm:bawm}] Based on the law of total probability, it holds that
\begin{align}
\PP_{\cD}\left(o_{h}\wvert(o,g)_{1:h-1},\cH_t\right)&=\sum_{z\in\cZ}\PP_z\left(o_h\wvert (o,g)_{1:h-1}\right)\cdot\PP_{\cD}\left(z\wvert (o,g)_{1:h-1},\cH_t\right)
	\label{eq:dogo_equiv}
\end{align}
Furthermore, based on Baye's theorem, we have
\begin{align}
\PP_{\cD}\left(z\wvert (o,g)_{1:h-1},\cH_t\right)=\frac{\prod_{h'=1}^{h-2}\PP_z\left(o_{h'+1}\wvert(o,g)_{1:h'}\right)}{\prod_{h'=1}^{h-2}\PP_{\cD}\left(o_{h'+1}\wvert(o,g)_{1:h'},\cH_t\right)}\cdot\PP_{\cD}(z\wvert\cH_t),
	\label{eq:dogo_equiv_2}
\end{align}
Hence, \eqref{eq:dogo_equiv} and \eqref{eq:dogo_equiv_2} jointly indicates that
\begin{align}
\prod_{h'=1}^{h-1}\PP_{\cD}\left(o_{h'+1}\wvert(o,g)_{1:h'},\cH_t\right)&=\PP_{\cD}\left(o_{h}\wvert(o,g)_{1:h-1},\cH_t\right)\cdot\prod_{h'=1}^{h-2}\PP_{\cD}\left(o_{h'+1}\wvert(o,g)_{1:h'},\cH_t\right)\nonumber\\
	&=\sum_{z\in\cZ}\PP_z\left(o_h\wvert (o,g)_{1:h-1}\right)\cdot\prod_{h'=1}^{h-2}\PP_z\left(o_{h'+1}\wvert(o,g)_{1:h'}\right)\cdot\PP_{\cD}(z\wvert\cH_t)\nonumber\\
	&=\sum_{z\in\cZ}\left(\prod_{h'=1}^{h-1}\PP_z\left(o_{h'+1}\wvert(o,g)_{1:h'}\right)\right)\cdot\PP_{\cD}(z\wvert \cH_t).
	\label{eq:world_bayes_decom}
\end{align}
Following the definition of marginal distributions, it holds that
\begin{align*}
	\PP_{\mathtt{LLM}}^t\left(o_h\wvert  o_1,\tdo g_{1:h-1}\right)
	&=\int_{o_{2:h-1}}\prod_{h'=1}^{h-1}\PP_{\cD}\left(o_{h'+1}\wvert(o,g)_{1:h'},\cH_t\right)\rd o_{2:h-1}\\
	&=\sum_{z\in\cZ}\left(\int_{o_{2:h-1}}\prod_{h'=1}^{h-1}\PP_z\left(o_{h'+1}\wvert(o,g)_{1:h'}\right)\rd o_{2:h-1}\right)\cdot\PP_{\cD}(z\wvert \cH_t)\\
	&=\sum_{z\in\cZ}\PP_{z}\left(o_h\wvert o_{1},\tdo g_{1:h-1}\right)\cdot\PP_{\cD}(z\wvert \cH_t),
\end{align*}
where the second equation follows \eqref{eq:world_bayes_decom} and then we complete the proof of Proposition \ref{thm:bawm}.
\end{proof}

\subsection{Proof of Corollary \ref{cor:bawm_regret}}\label{ap:bawm_regret}
\paragraph{Notations.} Denote $(\mathcal{J},\mathcal{\hat{J}})$ and $({\pi}_z^*,\hat{\pi}_z^*)$, and $(\mathbb{P}_{z,h},\mathbb{\hat{P}}_{z,h})$ as the value functions, optimal policies, and probability under the environment concerning the ground-truth $\mathbb{O}$ and the pretrained $\mathbb{O}_{\hat{\gamma}}$. Let $(\mathcal{\hat{J}}_{t,\mathtt{LLM}},\hat{\pi}_{\mathtt{LLM}}^{t,*})$ be the value function of the environment simulated by pretrained $\mathtt{LLM}_{\hat{\theta}}$ and its optimal policy; $\mathcal{J}_{t,\mathtt{LLM}}$ denote the value function of the environment simulated by perfect $\mathtt{LLM}$; $(\PP_{\mathtt{LLM}}^t,\hat{\PP}_{\mathtt{LLM}}^t)$ are the probability under environment simulated by perfect $\mathtt{LLM}$ or pretrained $\mathtt{LLM}_{\hat{\theta}}$.
\begin{proof}[Proof of Corollary \ref{cor:bawm_regret}]
Condition on the event $\cE_1$ that both Theorem \ref{thm:llmpretrain} and \ref{thm:translatorpretrain} hold, the regret under the practical setting can be decomposed as

\begin{align}
	{\rm Reg}_z(T)&\leq\underbrace{\sum_{t=1}^T\mathcal{\hat{J}}_z(\hat{\pi}_z^*,\omega^t)-\mathcal{{J}}_z(\hat{\pi}_z^*,\omega^t)}_{\textbf{(\romannumeral1})}+\underbrace{\sum_{t=1}^T\EE_{\cH_t}\left[\mathcal{{J}}_z(\hat{\pi}_z^*,\omega^t)-\mathcal{\hat{J}}_{t,\mathtt{LLM}}(\hat{\pi}_z^*,\omega^t)\right]}_{\textbf{(\romannumeral2})}\nonumber\\
	&\quad+\underbrace{\sum_{t=1}^T\EE_{\cH_t}\left[\mathcal{\hat{J}}_{t,\mathtt{LLM}}(\hat{\pi}_z^*,\omega^t)-\mathcal{\hat{J}}_{t,\mathtt{LLM}}(\hat{\pi}^t,\omega^t)\right]}_{\textbf{(\romannumeral3})}\nonumber\\
	&\quad+\underbrace{\sum_{t=1}^T\EE_{\cH_t}\left[\mathcal{\hat{J}}_{t,\mathtt{LLM}}(\hat{\pi}^t,\omega^t)-\mathcal{J}_z(\hat{\pi}^t,\omega^t)\right]}_{\textbf{(\romannumeral4})}+\underbrace{\sum_{t=1}^T\EE_{\cH_t}\left[\mathcal{J}_z(\hat{\pi}^t,\omega^t)-\mathcal{\hat{J}}_z(\hat{\pi}^t,\omega^t)\right]}_{\textbf{(\romannumeral5})}.
	\label{eq:wm_term15bound}
\end{align}
\noindent\textbf{Step 1. Bound (\romannumeral1) and (\romannumeral5) with Translator's Pretraining Error.}\\[5pt]
Similar to \eqref{eq:practical13} in the proof of Theorem \ref{thm:practical regret}, it holds that
\begin{align}
	\textbf{\small(\romannumeral1})+\textbf{\small(\romannumeral6})\leq2H^2T\lambda_R^{-1}\cdot\Delta_\mathtt{Rep}(N_{\rm p},T_{\rm p},H,\delta),
	\label{eq:bawm_term16}
\end{align}
following the pretraining error in Theorem \ref{thm:translatorpretrain}.\\[10pt]
\noindent\textbf{Step 2. Bound (\romannumeral3) via Optimality in \planner's Algorithm.}\\[5pt] Recall that \planner\ conducts task planning via the mixture policy:
\begin{equation}
\pi_h^t(\cdot\wvert \tau_{h}^t,\omega^t)\sim(1-\epsilon)\cdot\hat{\pi}_{h,\mathtt{LLM}}^{t,*}(\cdot\wvert \tau_{h}^t,\omega^t)+\epsilon\cdot\pi_{h,\mathtt{exp}}(\cdot|\tau_h^t),
\end{equation} 
Following this, it holds that
 \begin{align}
	\textbf{\small(\romannumeral3})&=\sum_{t=1}^T\EE_{\cH_t}\left[\mathcal{\hat{J}}_{t,\mathtt{LLM}}(\hat{\pi}_z^*,\omega^t)-\mathcal{\hat{J}}_{t,\mathtt{LLM}}(\hat{\pi}_{\mathtt{LLM}}^{t,*},\omega^t)\right]+\sum_{t=1}^T\EE_{\cH_t}\left[\mathcal{\hat{J}}_{t,\mathtt{LLM}}(\hat{\pi}_{\mathtt{LLM}}^{t,*},\omega^t)-\mathcal{\hat{J}}_{t,\mathtt{LLM}}(\hat{\pi}^t,\omega^t)\right]\nonumber\\
	&\leq\sum_{t=1}^T\EE_{\cH_t}\left[\mathcal{\hat{J}}_{t,\mathtt{LLM}}(\hat{\pi}_{\mathtt{LLM}}^{t,*},\omega^t)-(1-\epsilon)\cdot\mathcal{\hat{J}}_{t,\mathtt{LLM}}(\hat{\pi}_{\mathtt{LLM}}^{t,*},\omega^t)-\epsilon\cdot\mathcal{\hat{J}}_{t,\mathtt{LLM}}(\pi_\mathtt{exp},\omega^t)\right]\leq2HT\epsilon,
	\label{eq:bawm_term3bound}
\end{align}
where the the first inequality results from the optimality of $\hat{\pi}_{\mathtt{LLM}}^{t,*}$ under simulated environment.\\[10pt]
\noindent\textbf{Step 3. Bound (\romannumeral2) and (\romannumeral4) with LLM's Pretraining Error.}\\[5pt]
For any policy $\pi\in\Pi$, given history $\cH_t$, the performance difference follows
\begin{align}
	\mathcal{\hat{J}}_{t,\mathtt{LLM}}(\pi,\omega^t)-\mathcal{J}_z(\pi,\omega^t)&=\mathcal{\hat{J}}_{t,\mathtt{LLM}}(\pi,\omega^t)-\mathcal{J}_{t,\mathtt{LLM}}(\pi,\omega^t)+\mathcal{J}_{t,\mathtt{LLM}}(\pi,\omega^t)-\mathcal{J}_z(\pi,\omega^t)\nonumber\\
	&\leq\underbrace{\EE\left[\sum_{h=1}^H\int_{o_h}\left(\hat{\PP}_{\mathtt{LLM}}^t\left(o_h\wvert  o_1,\tdo g_{1:h-1}\right)-\PP_{\mathtt{LLM}}^t\left(o_h\wvert  o_1,\tdo g_{1:h-1}\right)\right)\rd o_{h}\right]}_{\textbf{(\romannumeral6})}\nonumber\\
	&\quad+\underbrace{\sup_{g_{1:H-1}}\sum_{h=1}^H\int_{o_h}\left(\PP_{\mathtt{LLM}}^t\left(o_h\wvert  o_1,\tdo g_{1:h-1}\right)-\PP_z\left(o_h\wvert  o_1,\tdo g_{1:h-1}\right)\right)\rd o_{h}}_{\textbf{(\romannumeral7})},\nonumber
\end{align}
where the inequality arises from $\|r_h\|_\infty\leq1$ depending solely on $o_h$. Furthermore, we have
\begin{align}
	&\int_{o_h}\hat{\PP}_{\mathtt{LLM}}^t\left(o_h\wvert  o_1,\tdo g_{1:h-1}\right)-\PP_{\mathtt{LLM}}^t\left(o_h\wvert  o_1,\tdo g_{1:h-1}\right)\rd o_{h}\nonumber\\
	&\quad=\int_{o_{2:h}}\left(\prod_{h'=1}^{h-1}\hat{\PP}_{\mathtt{LLM}}^t\left(o_{h'+1}\wvert  (o,g)_{1:h'}\right)-\prod_{h'=1}^{h-1}\PP_{\mathtt{LLM}}^t\left(o_{h'+1}\wvert  (o,g)_{1:h'}\right)\right)\rd o_{2:h}.
	\label{eq:bawm_llmerror}
	\end{align}
Following the arguments above, the difference can be decomposed as
\begin{align}
	&\prod_{h'=1}^{h-1}\hat{\PP}_{\mathtt{LLM}}^t\left(o_{h'+1}\wvert  (o,g)_{1:h'}\right)-\prod_{h'=1}^{h-1}\PP_{\mathtt{LLM}}^t\left(o_{h'+1}\wvert  (o,g)_{1:h'}\right)\nonumber\\
	&\quad=\sum_{h'=1}^{h-1}\left(\hat{\PP}_{\mathtt{LLM}}^t\left(o_{h'+1}\wvert  (o,g)_{1:h'}\right)-\PP_{\mathtt{LLM}}^t\left(o_{h'+1}\wvert  (o,g)_{1:h'}\right)\right)\nonumber\\
	&\qquad\cdot\prod_{k=h'+1}^{h-1}\hat{\PP}_{\mathtt{LLM}}^t\left(o_{k+1}\wvert  (o,g)_{1:k}\right)\cdot\prod_{k=1}^{h'-1}\PP_{\mathtt{LLM}}^t\left(o_{k+1}\wvert  (o,g)_{1:k}\right)\nonumber\\
	&\quad=\sum_{h'=1}^{h-1}\left(\mathtt{LLM}_{\hat{\theta}}\left(o_{h'+1}\wvert  (o,g)_{1:h'},\cH_t\right)-\mathtt{LLM}\left(o_{h'+1}\wvert  (o,g)_{1:h'},\cH_t\right)\right)\nonumber\\
	&\qquad\cdot\prod_{k=h'+1}^{h-1}\hat{\PP}_{\mathtt{LLM}}^t\left(o_{k+1}\wvert  (o,g)_{1:k}\right)\cdot\prod_{k=1}^{h'-1}\PP_{\mathtt{LLM}}^t\left(o_{k+1}\wvert  (o,g)_{1:k}\right).
	\label{eq:bawm_llmerror_decom}
\end{align}
Combine \eqref{eq:bawm_llmerror} and \eqref{eq:bawm_llmerror_decom}, it holds that
\begin{align}
	\textbf{(\romannumeral6})&\leq\sum_{h=1}^H\int_{o_{2:h}}\sum_{h'=1}^{h-1}\left(\mathtt{LLM}_{\hat{\theta}}\left(o_{h'+1}\wvert  (o,g)_{1:h'},\cH_t\right)-\mathtt{LLM}\left(o_{h'+1}\wvert  (o,g)_{1:h'},\cH_t\right)\right)\nonumber\\
	&\qquad\cdot\prod_{k=h'+1}^{h-1}\hat{\PP}_{\mathtt{LLM}}^t\left(o_{k+1}\wvert  (o,g)_{1:k}\right)\cdot\prod_{k=1}^{h'-1}\PP_{\mathtt{LLM}}^t\left(o_{k+1}\wvert  (o,g)_{1:k}\right)\rd o_{2:h}\nonumber\\
	&\leq\sum_{h=1}^H\sum_{h'=1}^{h-1}\EE_{o_{1:h'}|\cH_t}\left[D_{\rm TV}\left(\mathtt{LLM}_{\hat{\theta}}\left(o_{h'+1}\wvert(o,g)_{1:h'},\cH_t\right),\mathtt{LLM}\left(o_{h'+1}\wvert(o,g)_{1:h'},\cH_t\right)\right)\right].
	\label{eq:bawm_llmerror_all}
\end{align}
Following \eqref{eq:bawm_llmerror_all}, for any policy $\pi\in\Pi$, we have
\begin{align}
	&\sum_{t=1}^T\EE_{\cH_t}\left[\mathcal{\hat{J}}_{t,\mathtt{LLM}}(\pi,\omega^t)-\mathcal{J}_{t,\mathtt{LLM}}(\pi,\omega^t)\right]\nonumber\\
	&\quad\leq\sum_{t=1}^T\sum_{h=1}^H\sum_{h'=1}^{h-1}\EE_{\cH_t}\EE_{(o,g)_{1:h'}|\cH_t}\left[D_{\rm TV}\left(\mathtt{LLM}_{\hat{\theta}}\left(o_{h'+1}\wvert(o,g)_{1:h'},\cH_t\right),\mathtt{LLM}\left(o_{h'+1}\wvert(o,g)_{1:h'},\cH_t\right)\right)\right]\nonumber\\
	&\quad\leq\sum_{t=1}^T\sum_{h=1}^H\sum_{h'=1}^{h-1}\lambda_{S,1}\lambda_{S,2}^{-1}\cdot\bar{\EE}_{\cD_\mathtt{LLM}}\left[D_{\rm TV}\left(\mathtt{LLM}_{\hat{\theta}}\left(o_{h'+1}\wvert(o,g)_{1:h'},\cH_t\right),\mathtt{LLM}\left(o_{h'+1}\wvert(o,g)_{1:h'},\cH_t\right)\right)\right]\nonumber\\
	&\quad\leq H^2T\lambda_{S,1}\lambda_{S,2}^{-1}\cdot\Delta_\mathtt{LLM}(N_{\rm p},T_{\rm p},H,\delta)
	\label{eq:bawm_performance_diff_llm}
\end{align}
where the first inequality follows Theorem \ref{thm:llmpretrain} and Assumption \ref{as:onlinefullcoverage}. Based on Proposition \ref{thm:bawm}, the term \textbf{(\romannumeral7}) can be upper bounded using the Bayesian aggregated arguments such that
\begin{align*}
	\textbf{(\romannumeral7})&=\sup_{g_{1:H-1}}\sum_{z'\neq z}\sum_{h=1}^H\int_{o_h}(\PP_{z'}\left(o_h\wvert  o_1,\tdo g_{1:h-1}\right)-\PP_z\left(o_h\wvert  o_1,\tdo g_{1:h-1}\right))\cdot\PP_{\cD}(z'\wvert\cH_t)\rd o_{h}\leq H\sum_{z'\neq z}\PP_{\cD}(z'\wvert\cH_t).
\end{align*}
 Following the arguments above, for any policy $\pi\in\Pi$, it holds that
\begin{align}
	&\sum_{t=1}^T\EE_{\cH_t}\left[\mathcal{\hat{J}}_{t,\mathtt{LLM}}(\pi,\omega^t)-\mathcal{J}_z(\pi,\omega^t)\right]\leq H\sum_{t=1}^T\sum_{z'\neq z}\EE_{\cH_t}\left[\PP_{\cD}(z'\wvert\cH_t)\right],
	\label{eq:bawm_performance_diff_concen}
\end{align}
Combine \eqref{eq:bawm_performance_diff_llm}, \eqref{eq:bawm_performance_diff_concen} and the similar concentration arguments of posterior probability in \eqref{eq:term4concentrationsum}, denoted by event $\cE_2$ (see proof of Theorem \ref{thm:practical regret} in \S\ref{ap:practicalregret}), it holds that
\begin{align}
	\textbf{(\romannumeral2)}+\textbf{(\romannumeral4)}&\leq\sum_{t=1}^T\EE_{\cH_t}\left[\left(\mathcal{{J}}_z(\hat{\pi}_z^*,\omega^t)-\mathcal{\hat{J}}_{t,\mathtt{LLM}}(\hat{\pi}_z^*,\omega^t)\right)\cdot\ind\left(\cE_2\text{~holds}\right)\right]\nonumber\\
	&\quad+\sum_{t=1}^T\EE_{\cH_t}\left[\left(\mathcal{\hat{J}}_{t,\mathtt{LLM}}(\hat{\pi}^t,\omega^t)-\mathcal{J}_z(\hat{\pi}^t,\omega^t)\right)\cdot\ind\left(\cE_2\text{~holds}\right)\right]+2HT\delta\nonumber\\
	&\leq2H^2T\lambda_{S,1}\lambda_{S,2}^{-1}\cdot\Delta_\mathtt{LLM}(N_{\rm p},T_{\rm p},H,\delta)+2HT\delta\nonumber\\
	&\qquad+c_0\cdot2H\log(c_\mathcal{Z}|\mathcal{Z}|/\delta)\cdot\left(\eta\epsilon-H\lambda^{-1}_{R}\cdot\Delta_\mathtt{Rep}(N_{\rm p},T_{\rm p},H,\delta)^2\right)^{-1}
	\label{eq:bawm_term24bound}
\end{align}
\noindent\textbf{Step 4. Conclude the Proof based on Step 1, Step 2, and Step 3.}\\[5pt]
Combine \eqref{eq:bawm_term16}, \eqref{eq:bawm_term3bound} and \eqref{eq:bawm_term24bound}, we have 
\begin{align}
	{\rm Reg}_z(T)&\leq \underbrace{c_0\cdot2H\log(c_\mathcal{Z}|\mathcal{Z}|/\delta)\cdot\left(\eta\epsilon-H\lambda^{-1}_{R}\cdot\Delta_\mathtt{Rep}(N_{\rm p},T_{\rm p},H,\delta)^2\right)^{-1}}_{\textbf{(\romannumeral8})}+4HT\delta\nonumber\\
	&\qquad+\underbrace{2HT\eta^{-1}\left(\eta\epsilon-H\lambda^{-1}_{R}\cdot\Delta_\mathtt{Rep}(N_{\rm p},T_{\rm p},H,\delta)^2\right)}_{\textbf{(\romannumeral9})}+2H^2T\lambda_{S,1}\lambda_{S,2}^{-1}\cdot\Delta_\mathtt{LLM}(N_{\rm p},T_{\rm p},H,\delta)\nonumber\\
	&\qquad+2H^2T(\eta\lambda_R)^{-1}\cdot\Delta_\mathtt{Rep}(N_{\rm p},T_{\rm p},H,\delta)^2+2H^2T\lambda_R^{-1}\cdot\Delta_\mathtt{Rep}(N_{\rm p},T_{\rm p},H,\delta)\nonumber\\
	&\leq\mathcal{O}\Big(H\sqrt{\log(c_\mathcal{Z}|\mathcal{Z}|/\delta)\cdot T/\eta}+H^2T\cdot\Delta_{\rm p,wm}(N_{\rm p},T_{\rm p},H,\delta,\xi)\Big)+4HT\delta,
\end{align}
if we choose $\epsilon=(\log(c_\mathcal{Z}|\mathcal{Z}|\sqrt{T})/T\eta)^{1/2}+H(\eta\lambda_{\min})^{-1}\cdot\Delta_\mathtt{Rep}(N_{\rm p},T_{\rm p},H,\delta)^2$ to strike an exploration-exploitation balance between \textbf{(\romannumeral8}) and \textbf{(\romannumeral9}). Thus, the cumulative pretraining error follows 
\begin{align*}
	\Delta_{\rm p,wm}&(N_{\rm p},T_{\rm p},H,\delta,\xi)=2(\eta\lambda_R)^{-1}\cdot\Delta_\mathtt{Rep}(N_{\rm p},T_{\rm p},H,\delta)^2\\
	&+2\lambda_R^{-1}\cdot\Delta_\mathtt{Rep}(N_{\rm p},T_{\rm p},H,\delta)+2\lambda_{S,1}\lambda_{S,2}^{-1}\cdot\Delta_\mathtt{LLM}(N_{\rm p},T_{\rm p},H,\delta).
\end{align*}
Here, $\xi=(\eta,\lambda_{S,1},\lambda_{S,2},\lambda_R)$ denotes the set of distinguishability and coverage coefficients in Definition \ref{def:iden} and Assumption \ref{as:onlinecoverage}, and $\Delta_\mathtt{LLM}(N_{\rm p},T_{\rm p},H,\delta)$ and $\Delta_\mathtt{Rep}(N_{\rm p},T_{\rm p},H,\delta)$ are  pretraining errors defined in Theorem \ref{thm:llmpretrain} and Theorem \ref{thm:translatorpretrain}. By taking $\delta=1/\sqrt{T}$, we complete the entire proof. 
\end{proof}

\subsection{Proof of Corollary \ref{corol:perfectregret}}\label{ap:multiagent_regret}
The proof is similar to that in \S\ref{ap:perfectplanning}.\\[10pt]
\emph{Proof Sketch of Corollary \ref{corol:perfectregret}.} We first verify the claim in \eqref{eq:multiagent_BAIL}, which is akin to Proposition \ref{thm:BAIL}. Note that for all $(h,t)\in[H]\times[T]$, based on the law of total probability, it holds that
\begin{align}
	\pi_{h,\mathtt{LLM}}^t\big(\mathbf{g}_h^t\wvert  \tau_{h}^t,\omega^t\big)&=\prod_{k\in\cK}\mathtt{LLM}\big(g_{h,k}^t\wvert \prompt_{h,k}^t\big)\nonumber\\
	&=\prod_{k\in\cK}\left(\sum_{z\in\mathcal{Z}}\mathbb{P}\left(g_{h,k}^t\wvert\prompt_{h,k}^t,z\right)\cdot\PP_{\cD}\left(z\wvert\prompt_{h,k}^t\right)\right)\nonumber\\
	&=\prod_{k\in\cK}\left(\sum_{z\in\mathcal{Z}}\pi^*_{z,h,k}\left(g_{h,k}^t\wvert \tau_{h}^t,\omega^t\right)\cdot\PP_{\cD}\left(z\wvert\prompt_h^t\right)\right),
	\label{eq:multiagent_bail}
\end{align} 
where the first equation arises from the autoregressive manner of LLM, and the last equation follows the generating distribution. The \planner\ takes a mixture policy of $\pi_\mathtt{exp}$ and $\pi_\mathtt{LLM}$ such that
\begin{equation}
\pi_h^t(\mathbf{g}_h^t\wvert \tau_{h}^t,\omega^t)\sim(1-\epsilon)\cdot\pi_{h,\mathtt{LLM}}^t(\mathbf{g}_h^t\wvert \tau_{h}^t,\omega^t)+\epsilon\cdot\pi_{h,\mathtt{exp}}(\mathbf{g}_h^t\wvert\tau_h^t),
\end{equation}
for any $(h,t)\in[H]\times[T]$ given an $\eta$-distinguishable policy $\pi_{\mathtt{exp}}$ (see Definition \ref{def:iden}). Given a sequence of high-level tasks $\{\omega^t\}_{t\in[T]}$, the regret can be decomposed as
\begin{align}
\text{Reg}(T)&\leq\sum_{t=1}^T\sum_{h=1}^H\EE_{\cH_t\sim\bigotimes_{i=1}^{t-1}\mathbb{P}_{z}^{\pi_i}}\mathbb{E}_{(s_h^t,\tau_{h}^t)\sim\mathbb{P}_{z}^{\pi^t}}\left[\left(\pi_{z,h}^*-{\pi^t_{h,\mathtt{LLM}}}\right)Q_{z,h}^*(s_h^t,\tau_{h}^t,\omega^t)\right]+HT\epsilon,
\label{eq:multiregretdecom}
\end{align}
 Recall that \eqref{eq:perfect_llmpolicy} indicates that for all $(h,t)\in[H]\times[T]$, we have
\begin{align*}
&\left(\pi_{z,h}^*-{\pi^t_{h,\mathtt{LLM}}}\right)(\gb_h\wvert \tau_h,\omega)\\
&\quad=\prod_{k\in\cK}\left(\sum_{z'\in\mathcal{Z}}\pi^*_{z',h,k}\left(g_{h,k}\wvert \tau_{h},\omega\right)\cdot\PP_{\cD}\left(z'\wvert\prompt_h^t\right)\right)-\prod_{k\in\cK}\pi_{z,h,k}^*(g_{h,k}\wvert \tau_{h},\omega)\\
&\quad\leq H\sum_{k\in\cK}\left(\sum_{z'\neq z}(\pi^*_{z',h,k}-\pi^*_{z,h,k})\left(g_{h,k}\wvert \tau_{h},\omega\right)\cdot\PP_{\cD}\left(z'\wvert\prompt_h^t\right)\right)\\
	&\quad\qquad\cdot\prod_{k'=1}^{k-1}\left(\sum_{z''\in\mathcal{Z}}\pi^*_{z'',h,k'}\left(g_{h,k,k'}\wvert \tau_{h},\omega\right)\cdot\PP_{\cD}\left(z'\wvert\prompt_h^t\right)\right)\cdot\prod_{k'=k+1}^{K}\pi_{z,h}^*(g_{h,k}\wvert \tau_{h},\omega).
\end{align*}
Following this, we have
\begin{align}
	\left(\pi_{z,h}^*-{\pi^t_{h,\mathtt{LLM}}}\right)Q_{z,h}^*(s_h^t,\tau_{h}^t,\omega^t)&\leq HK\cdot\sum_{z'\ne z}\PP_{\cD}\left(z\wvert\prompt_h^t\right),
	\label{eq:multi_policy_diff}
\end{align} 
for all $(h,t)\in[H]\times[T]$. Based on Lemma \ref{online guarantee} and the similar arguments in the proof Theorem \ref{thm:regret} in \S\ref{ap:perfectplanning}, with probability at least $1-\delta$, the following event $\cE_1$ holds: for all $(h,t)\in[H]\times[T]$, \begin{equation}
\sum_{z'\ne z}\PP_{\cD}(z'\wvert\prompt_h^t)\leq\mathcal{O}\left(\min\left\{\log\left(c_\mathcal{Z}|\mathcal{Z}|/\delta\right)\eta^{-1}/|\mathcal{X}^{t-1}_{\mathtt{exp}}|,1\right\}\right),
\label{eq:multi_nonoptimal}
\end{equation}
where $\mathcal{X}^t_{\mathtt{exp}}=\{i\in[t]:\pi^i=\pi_{\mathtt{exp}}\}$ denotes the set of exploration episodes. Based on \eqref{eq:multiagent_bail}, \eqref{eq:multi_policy_diff} and conditioned on $\cE_1$, it holds that
\begin{align}
	\sum_{t=1}^T&\sum_{h=1}^H\EE_{\cH_t\sim\bigotimes_{i=1}^{t-1}\mathbb{P}_{z}^{\pi_i}}\mathbb{E}_{(s_h^t,\tau_{h}^t)\sim\mathbb{P}_{z}^{\pi^t}}\left[\left(\pi_{z,h}^*-{\pi^t_{h,\mathtt{LLM}}}\right)Q_{z,h}^*(s_h^t,\tau_{h}^t,\omega^t)\right]\nonumber\\
	&\leq2\log(c_\mathcal{Z}|\mathcal{Z}|/\delta)HK\eta^{-1}\cdot\sum_{t=1}^T\sum_{h=1}^H\EE_{\cH_t\sim\bigotimes_{i=1}^{t-1}\mathbb{P}_{z}^{\pi_i}}\mathbb{E}_{\tau_{h}^t\sim\mathbb{P}_{z}^{\pi^t}}\left[\min\left\{1/|\mathcal{X}^{t-1}_{\mathtt{exp}}|,1\right\}\right],
	\label{eq:multillmpdl}
\end{align}
Note that $\mathds{1}(\pi^t=\pi_\mathtt{exp})\overset{\rm iid}{\sim}\text{Bernuolli}(\epsilon)$ for all $t\in[T]$. Besides, with probability at least $1-\delta$, the following event $\cE_2$ holds:
\begin{align}
	\sum_{t=1}^{T}\min\left\{1/|\mathcal{X}^{t-1}_{\mathtt{exp}}|,1\right\}\leq \cO(\epsilon^{-1}\log(T\log T/\delta)).
	\label{eq:multiplanningerr}
\end{align}
based on Lemma \ref{lem:epsilon_concen}. Combine \eqref{eq:multiregretdecom}, \eqref{eq:multillmpdl} and \eqref{eq:multiplanningerr}, it follows that
\begin{align*}
	{\rm Reg}_z(T)&\leq \sum_{t=1}^T\sum_{h=1}^H\EE_{\cH_t\sim\bigotimes_{i=1}^{t-1}\mathbb{P}_{z}^{\pi_i}}\mathbb{E}_{(s_h^t,\tau_{h}^t)\sim\mathbb{P}_{z}^{\pi^t}}\left[\left(\pi_{z,h}^*-{\pi^t_{h,\mathtt{LLM}}}\right)Q_{z,h}^*(s_h,\tau_{h},\omega^t)\ind\left(\cE_1\cap\cE_2\text{~holds}\right)\right]\\
	&\quad+ \sum_{t=1}^T\sum_{h=1}^H\EE_{\cH_t\sim\bigotimes_{i=1}^{t-1}\mathbb{P}_{z}^{\pi_i}}\mathbb{E}_{(s_h^t,\tau_{h}^t)\sim\mathbb{P}_{z}^{\pi^t}}\left[\left(\pi_{z,h}^*-{\pi^t_{h,\mathtt{LLM}}}\right)Q_{z,h}^*(s_h,\tau_{h},\omega^t)\ind\left(\cE_1\cap\cE_2\text{~fails}\right)\right]+ HT\epsilon\\
	&\leq\mathcal{O}\Big(\log(c_\mathcal{Z}|\mathcal{Z}|/\delta)H^2K\log(T\log T/\delta)\cdot(\eta\epsilon)^{-1}+ HT\epsilon+H\sqrt{T}\log(1/\delta)+2HT\delta\Big)\\
	&\leq\tilde{\mathcal{O}}\left(H^\frac{3}{2}\sqrt{TK/\eta\cdot\log\left(c_\mathcal{Z}|\mathcal{Z}|/\delta\right)}\right),
\end{align*}  
where we choose to expolre with probability $\epsilon=(HK\log\left(c_\mathcal{Z}|\mathcal{Z}|/\delta\right)/T\eta)^{1/2}$ in the last inequality. If we take $\delta=1/\sqrt{T}$ in the arguments above, then we conclude the proof of Corollary \ref{corol:perfectregret}.\hfill$\Box$

%% file: tex/appendix/appendix_lemma.tex
\section{Technical Lemmas}\label{ap:lemma}
\begin{lemma}[Martingale Concentration Inequality]
	Let $X_1,\dots,X_T$ be a sequence of real-valued random variables adapted to a filter $(\sF_t)_{t\leq T}$. For any $\delta\in(0,1)$ and $\lambda>0$,  it holds that
	$$
	\PP\left(\exists T'\in[T]: -\sum_{t=1}^{T'}X_t\geq\sum_{t=1}^{T'}\frac{1}{\lambda}\log\mathbb{E}\left[\exp(-\lambda X_t)|\sF_{t-1}\right]+\frac{1}{\lambda}\log\left(1/\delta\right)\right)\leq\delta.
	$$
	\label{lem:exptrick}
\end{lemma}
\vspace{-10pt}
\noindent\emph{Proof of Lemma \ref{lem:exptrick}}. See Lemma A.4 in \citet{foster2021statistical} and Theorem 13.2 in \citet{zhang2023mathematical} for detailed proof. Lemma A.4 in \citet{foster2021statistical} is a special case by taking $\lambda=1$.

\begin{lemma}[Donsker-Varadhan]
Let $P$ and $Q$ be the probability measures over $\mathcal{X}$, then
	$$
	D_{\rm KL}(P\wVert Q)=\underset{f\in\mathcal{F}}{\sup}\left\{\mathbb{E}_{x\sim P}\left[f(x)\right]-\log\mathbb{E}_{x\sim Q}\left[\exp(f(x))\right]\right\},
	$$
	where $\mathcal{F}=\{f:\mathcal{X}\mapsto\mathbb{R}\wvert\mathbb{E}_{x\sim Q}\left[\exp(f(x))\right]\leq\infty\}$.
	\label{lem:dvrepresent}
\end{lemma}
\noindent\emph{Proof of Lemma \ref{lem:dvrepresent}}. See \citet{donsker1976asymptotic} for detailed proof.

\begin{lemma}[MLE guarantee]\label{lem:mle}
	Let $\mathcal{F}$ be finite function class and there exists $f^*\in\mathcal{F}$ such that $f^*(x,y)=\mathbb{P}(y\wvert x)$, where $\mathbb{P}(y\wvert x)$ is the conditional distribution for estimation. Given a dataset $\mathcal{D}=\{x_i,y_i\}_{i\in[N]}$ where $x_i\sim\PP_\mathcal{D}(\cdot\wvert x_{1:i-1},y_{1:i-1})$ and $y_i\sim\mathbb{P}_\cD(\cdot\wvert x_i)$ for all $i\in[N]$, we have
	$$
	\bar{\mathbb{E}}_\mathcal{D}\left[D_{\rm TV}^2\left(\hat{f}(x,\cdot),f^*(x,\cdot)\right)\right]\leq 2\log(N|\mathcal{F}|/\delta)/N
	$$
	with propbability at least $1-\delta$, where $\hat{f}$ is the maximum likelihood estimator such that
	$$
	\hat{f}:=\underset{f\in\mathcal{F}}{\rm argmax}\ \mathbb{\hat E}_{\mathcal{D}}\left[\log f(x,y)\right].
	$$
\end{lemma}
\noindent\emph{Proof of Lemma \ref{lem:mle}}. See Theorem 21 in \citet{agarwal2020flambe} for detailed proof.

\begin{lemma}[Performance Difference Lemma for POMDP]
	Consider policies $\pi,\pi'\in\Pi$, it holds
	$$
	\mathcal{J}(\pi)-\mathcal{J}(\pi')=\sum_{h=1}^H\mathbb{E}_{\pi}\left[Q_{h}^{\pi'}(s_h,\tau_{h},g_h)-V_{h}^{\pi'}(s_h,\tau_{h})\right].
	$$
	For fixed policy $\pi\in\Pi$ under different POMDPs, denoted by $\mathcal{M}$ and $\mathcal{M}'$, then it holds that
	$$
	\mathcal{J}_\mathcal{M}(\pi)-\mathcal{J}_\mathcal{M'}(\pi)=\sum_{h=1}^H\mathbb{E}_{\mathcal{M}}^\pi\left[(\mathbb{P}_{h,\mathcal{M}}V_{h+1,\mathcal{M}'}^{\pi}-\mathbb{P}_{h,\mathcal{M}'}V_{h+1,\mathcal{M}'}^{\pi})(s_h,\tau_{h},g_h)\right],
	$$	
	where $\mathbb{P}_{h,\mathcal{M}}V_{h+1,\mathcal{M}'}^{\pi}(s_h,\tau_{h},g_h)=\langle V_{h+1,\mathcal{M}'}^{\pi}(\cdot,\cdot),\mathbb{P}_{h,\mathcal{M}}(\cdot,\cdot\wvert s_h,\tau_{h},g_h)\rangle_{\cS\times\cT^*}$.
	\label{lem:pdl}
\end{lemma}
\begin{lemma}
	Let $X_t\overset{\rm iid}{\sim}\text{Bernuolli}(\rho)$ and $Y_t=\sum_{\tau=1}^tX_\tau$. For any $\delta\in(0,1)$ and $\rho>0$, with probability greater than $1-\delta$, it holds that
	$\sum_{t=1}^T\min\left\{1/Y_t,1\right\}\leq \cO(\rho^{-1}\log(T\log T/\delta))$.
	\label{lem:epsilon_concen}
\end{lemma}
\begin{proof}[Proof of Lemma \ref{lem:epsilon_concen}]
Note that $\{Y_t\}_{t\in[T]}$ is non-decreasing and it holds that
	\begin{equation}
		\sum_{t=1}^T\min\left\{\frac{1}{Y_t},1\right\}=\#\{t\in[T]:Y_t=0\}+\sum_{t\in[T]:Y_t>0}\frac{1}{Y_t},
	\end{equation}
	and with probability at least $1-\delta$, the following event $\cE_0$ holds:
	$$
	t_0:=\#\{t\in[T]:Y_t=0\}\leq\frac{\log(\delta)}{\log(1-\rho)}\leq\rho^{-1}\log(1/\delta),
	$$
	where the first inequality results from the property of Bernuolli random variable, and the second inequality uses fact that $\log(1-x)\leq-x$ for all $x\leq1$. 
For notational simplicy, we write $\{t\in[T]:Y_t>0\}=\{t_0,\dots,t_0+2^{N_T}-1\}$. With probability at least $1-\delta$, the following event $\cE_n$ holds:
	\begin{equation}
		Y_{t_0+2^n}=\sum_{\tau=1}^{t_0+2^n}X_t=\sum_{\tau=t_0+1}^{t_0+2^n}X_t\geq2^n\rho-\sqrt{2^{n-1}\log(1/\delta)}.
		\label{eq:zero}
	\end{equation}
	based on the Hoeffding inequality. Suppose that $\{\cE_n\}_{n\in[N_T]}$ holds, then we have
	\begin{align}
		\sum_{t\in[T]:Y_t>0}\frac{1}{Y_t}=\sum_{n=0}^{N_T}\sum_{t=t_0+2^n}^{2^{n+1}-1}\frac{1}{Y_t}\leq\sum_{n=0}^{N_T}\frac{2^n}{Y_{t_0+2^n}}\leq\sum_{n=0}^{N_T}\frac{2^n}{\max\{2^n\rho-\sqrt{2^{n-1}\log(1/\delta)},1\}}.
		\label{eq:concen_hoeffding}
	\end{align}
	Let $n_0=1+\lceil\log_2(\rho^{-2}\log(1/\delta))\rceil$ such that $\rho-\sqrt{\log(1/\delta)/2^{n+1}}\geq\rho/2$. Following \eqref{eq:concen_hoeffding}, it holds
	\begin{align}
		\sum_{t\in[T]:Y_t>0}\frac{1}{Y_t}\leq\sum_{n=0}^{n_0}2^n+\sum_{n=n_0+1}^{N_T}2\rho^{-1}\leq2^{n_0+1}+2\rho^{-1}N_T\leq8\rho^{-2}\log(1/\delta)+4\rho^{-1}\log T.
		\label{eq:nonzero}
	\end{align}
	Combine \eqref{eq:zero} and \eqref{eq:nonzero}, by taking a union bound over $\cE_0,\dots,\cE_{N_T}$, then we can get
	\begin{align*}
		\sum_{t=1}^T\min\left\{\frac{1}{Y_t},1\right\}&\leq8\rho^{-2}\log(2N_T/\delta)+4\rho^{-1}\log(2TN_T/\delta)\\
		&\leq 8\rho^{-2}\log(4\log T/\delta)+4\rho^{-1}\log(4T\log T/\delta)\leq \cO(\rho^{-1}\log(T\log T/\delta)),
	\end{align*}
	where we use the fact that $\log_2 T\leq2\log T$, and then we finish the proof of Lemma \ref{lem:epsilon_concen}.
\end{proof}

%% file: main.bbl
\begin{thebibliography}{}

\bibitem[Agarwal et~al., 2020]{agarwal2020flambe}
Agarwal, A., Kakade, S., Krishnamurthy, A., and Sun, W. (2020).
\newblock Flambe: Structural complexity and representation learning of low rank mdps.
\newblock {\em Advances in neural information processing systems}, 33:20095--20107.

\bibitem[Ahuja et~al., 2023]{ahuja2023context}
Ahuja, K., Panwar, M., and Goyal, N. (2023).
\newblock In-context learning through the bayesian prism.
\newblock {\em arXiv preprint arXiv:2306.04891}.

\bibitem[Barto et~al., 1995]{barto1995learning}
Barto, A.~G., Bradtke, S.~J., and Singh, S.~P. (1995).
\newblock Learning to act using real-time dynamic programming.
\newblock {\em Artificial intelligence}, 72(1-2):81--138.

\bibitem[Barto and Mahadevan, 2003]{barto2003recent}
Barto, A.~G. and Mahadevan, S. (2003).
\newblock Recent advances in hierarchical reinforcement learning.
\newblock {\em Discrete event dynamic systems}, 13(1-2):41--77.

\bibitem[Ba{\c{s}}ar and Olsder, 1998]{bacsar1998dynamic}
Ba{\c{s}}ar, T. and Olsder, G.~J. (1998).
\newblock {\em Dynamic noncooperative game theory}.
\newblock SIAM.

\bibitem[Blei et~al., 2003]{blei2003latent}
Blei, D.~M., Ng, A.~Y., and Jordan, M.~I. (2003).
\newblock Latent dirichlet allocation.
\newblock {\em Journal of machine Learning research}, 3(Jan):993--1022.

\bibitem[Bonet and Geffner, 2001]{bonet2001planning}
Bonet, B. and Geffner, H. (2001).
\newblock Planning as heuristic search.
\newblock {\em Artificial Intelligence}, 129(1-2):5--33.

\bibitem[Brohan et~al., 2023]{brohan2023can}
Brohan, A., Chebotar, Y., Finn, C., Hausman, K., Herzog, A., Ho, D., Ibarz, J., Irpan, A., Jang, E., Julian, R., et~al. (2023).
\newblock Do as i can, not as i say: Grounding language in robotic affordances.
\newblock In {\em Conference on Robot Learning}, pages 287--318. PMLR.

\bibitem[Brown et~al., 2020]{brown2020language}
Brown, T., Mann, B., Ryder, N., Subbiah, M., Kaplan, J.~D., Dhariwal, P., Neelakantan, A., Shyam, P., Sastry, G., Askell, A., et~al. (2020).
\newblock Language models are few-shot learners.
\newblock {\em Advances in neural information processing systems}, 33:1877--1901.

\bibitem[Browne et~al., 2012]{browne2012survey}
Browne, C.~B., Powley, E., Whitehouse, D., Lucas, S.~M., Cowling, P.~I., Rohlfshagen, P., Tavener, S., Perez, D., Samothrakis, S., and Colton, S. (2012).
\newblock A survey of monte carlo tree search methods.
\newblock {\em IEEE Transactions on Computational Intelligence and AI in games}, 4(1):1--43.

\bibitem[Chan et~al., 2022]{chan2022data}
Chan, S.~C., Santoro, A., Lampinen, A.~K., Wang, J.~X., Singh, A., Richemond, P.~H., McClelland, J., and Hill, F. (2022).
\newblock Data distributional properties drive emergent few-shot learning in transformers.
\newblock {\em arXiv preprint arXiv:2205.05055}.

\bibitem[Chane-Sane et~al., 2021]{chane2021goal}
Chane-Sane, E., Schmid, C., and Laptev, I. (2021).
\newblock Goal-conditioned reinforcement learning with imagined subgoals.
\newblock In {\em International Conference on Machine Learning}, pages 1430--1440. PMLR.

\bibitem[Dann et~al., 2022]{dann2022guarantees}
Dann, C., Mansour, Y., Mohri, M., Sekhari, A., and Sridharan, K. (2022).
\newblock Guarantees for epsilon-greedy reinforcement learning with function approximation.
\newblock In {\em International conference on machine learning}, pages 4666--4689. PMLR.

\bibitem[Devlin et~al., 2018]{devlin2018bert}
Devlin, J., Chang, M.-W., Lee, K., and Toutanova, K. (2018).
\newblock Bert: Pre-training of deep bidirectional transformers for language understanding.
\newblock {\em arXiv preprint arXiv:1810.04805}.

\bibitem[Donsker and Varadhan, 1976]{donsker1976asymptotic}
Donsker, M.~D. and Varadhan, S.~S. (1976).
\newblock Asymptotic evaluation of certain markov process expectations for large time—iii.
\newblock {\em Communications on pure and applied Mathematics}, 29(4):389--461.

\bibitem[Du et~al., 2023]{du2023video}
Du, Y., Yang, M., Florence, P., Xia, F., Wahid, A., Ichter, B., Sermanet, P., Yu, T., Abbeel, P., Tenenbaum, J.~B., et~al. (2023).
\newblock Video language planning.
\newblock {\em arXiv preprint arXiv:2310.10625}.

\bibitem[Duan et~al., 2020]{duan2020minimax}
Duan, Y., Jia, Z., and Wang, M. (2020).
\newblock Minimax-optimal off-policy evaluation with linear function approximation.
\newblock In {\em International Conference on Machine Learning}, pages 2701--2709. PMLR.

\bibitem[Edwards et~al., 2019]{edwards2019imitating}
Edwards, A., Sahni, H., Schroecker, Y., and Isbell, C. (2019).
\newblock Imitating latent policies from observation.
\newblock In {\em International conference on machine learning}, pages 1755--1763. PMLR.

\bibitem[Foster et~al., 2021]{foster2021statistical}
Foster, D.~J., Kakade, S.~M., Qian, J., and Rakhlin, A. (2021).
\newblock The statistical complexity of interactive decision making.
\newblock {\em arXiv preprint arXiv:2112.13487}.

\bibitem[Fu et~al., 2024]{fu2024drive}
Fu, D., Li, X., Wen, L., Dou, M., Cai, P., Shi, B., and Qiao, Y. (2024).
\newblock Drive like a human: Rethinking autonomous driving with large language models.
\newblock In {\em Proceedings of the IEEE/CVF Winter Conference on Applications of Computer Vision}, pages 910--919.

\bibitem[Garg et~al., 2022]{garg2022can}
Garg, S., Tsipras, D., Liang, P.~S., and Valiant, G. (2022).
\newblock What can transformers learn in-context? a case study of simple function classes.
\newblock {\em Advances in Neural Information Processing Systems}, 35:30583--30598.

\bibitem[Geer, 2000]{geer2000empirical}
Geer, S.~A. (2000).
\newblock {\em Empirical Processes in M-estimation}, volume~6.
\newblock Cambridge university press.

\bibitem[Ghallab et~al., 2004]{ghallab2004automated}
Ghallab, M., Nau, D., and Traverso, P. (2004).
\newblock {\em Automated Planning: theory and practice}.
\newblock Elsevier.

\bibitem[Grigsby et~al., 2023]{grigsby2023amago}
Grigsby, J., Fan, L., and Zhu, Y. (2023).
\newblock Amago: Scalable in-context reinforcement learning for adaptive agents.
\newblock {\em arXiv preprint arXiv:2310.09971}.

\bibitem[Gutmann and Hyv{\"a}rinen, 2010]{gutmann2010noise}
Gutmann, M. and Hyv{\"a}rinen, A. (2010).
\newblock Noise-contrastive estimation: A new estimation principle for unnormalized statistical models.
\newblock In {\em Proceedings of the thirteenth international conference on artificial intelligence and statistics}, pages 297--304. JMLR Workshop and Conference Proceedings.

\bibitem[Hahn and Goyal, 2023]{hahn2023theory}
Hahn, M. and Goyal, N. (2023).
\newblock A theory of emergent in-context learning as implicit structure induction.
\newblock {\em arXiv preprint arXiv:2303.07971}.

\bibitem[Hao et~al., 2023]{hao2023reasoning}
Hao, S., Gu, Y., Ma, H., Hong, J.~J., Wang, Z., Wang, D.~Z., and Hu, Z. (2023).
\newblock Reasoning with language model is planning with world model.
\newblock {\em arXiv preprint arXiv:2305.14992}.

\bibitem[Hoeting et~al., 1999]{hoeting1999bayesian}
Hoeting, J.~A., Madigan, D., Raftery, A.~E., and Volinsky, C.~T. (1999).
\newblock Bayesian model averaging: a tutorial.
\newblock {\em Statistical science}, 14(4):382--417.

\bibitem[Honovich et~al., 2022]{honovich2022instruction}
Honovich, O., Shaham, U., Bowman, S.~R., and Levy, O. (2022).
\newblock Instruction induction: From few examples to natural language task descriptions.
\newblock {\em arXiv preprint arXiv:2205.10782}.

\bibitem[Hu et~al., 2023]{hu2023tree}
Hu, M., Mu, Y., Yu, X., Ding, M., Wu, S., Shao, W., Chen, Q., Wang, B., Qiao, Y., and Luo, P. (2023).
\newblock Tree-planner: Efficient close-loop task planning with large language models.
\newblock {\em arXiv preprint arXiv:2310.08582}.

\bibitem[Hu and Shu, 2023]{hu2023language}
Hu, Z. and Shu, T. (2023).
\newblock Language models, agent models, and world models: The law for machine reasoning and planning.
\newblock {\em arXiv preprint arXiv:2312.05230}.

\bibitem[Huang et~al., 2022]{huang2022language}
Huang, W., Abbeel, P., Pathak, D., and Mordatch, I. (2022).
\newblock Language models as zero-shot planners: Extracting actionable knowledge for embodied agents.
\newblock In {\em International Conference on Machine Learning}, pages 9118--9147. PMLR.

\bibitem[Iyer et~al., 2022]{iyer2022opt}
Iyer, S., Lin, X.~V., Pasunuru, R., Mihaylov, T., Simig, D., Yu, P., Shuster, K., Wang, T., Liu, Q., Koura, P.~S., et~al. (2022).
\newblock Opt-iml: Scaling language model instruction meta learning through the lens of generalization.
\newblock {\em arXiv preprint arXiv:2212.12017}.

\bibitem[Jia et~al., 2021]{jia2021scaling}
Jia, C., Yang, Y., Xia, Y., Chen, Y.-T., Parekh, Z., Pham, H., Le, Q., Sung, Y.-H., Li, Z., and Duerig, T. (2021).
\newblock Scaling up visual and vision-language representation learning with noisy text supervision.
\newblock In {\em International conference on machine learning}, pages 4904--4916. PMLR.

\bibitem[Jiang, 2023]{jiang2023latent}
Jiang, H. (2023).
\newblock A latent space theory for emergent abilities in large language models.
\newblock {\em arXiv preprint arXiv:2304.09960}.

\bibitem[Kidambi et~al., 2021]{kidambi2021mobile}
Kidambi, R., Chang, J., and Sun, W. (2021).
\newblock Mobile: Model-based imitation learning from observation alone.
\newblock {\em Advances in Neural Information Processing Systems}, 34:28598--28611.

\bibitem[Kim et~al., 2022]{kim2022self}
Kim, H.~J., Cho, H., Kim, J., Kim, T., Yoo, K.~M., and Lee, S.-g. (2022).
\newblock Self-generated in-context learning: Leveraging auto-regressive language models as a demonstration generator.
\newblock {\em arXiv preprint arXiv:2206.08082}.

\bibitem[Kusner et~al., 2017]{kusner2017grammar}
Kusner, M.~J., Paige, B., and Hern{\'a}ndez-Lobato, J.~M. (2017).
\newblock Grammar variational autoencoder.
\newblock In {\em International conference on machine learning}, pages 1945--1954. PMLR.

\bibitem[Lee et~al., 2023]{lee2023supervised}
Lee, J.~N., Xie, A., Pacchiano, A., Chandak, Y., Finn, C., Nachum, O., and Brunskill, E. (2023).
\newblock Supervised pretraining can learn in-context reinforcement learning.
\newblock {\em arXiv preprint arXiv:2306.14892}.

\bibitem[Li et~al., 2023a]{li2023interactive}
Li, B., Wu, P., Abbeel, P., and Malik, J. (2023a).
\newblock Interactive task planning with language models.
\newblock {\em arXiv preprint arXiv:2310.10645}.

\bibitem[Li et~al., 2023b]{li2023multimodal}
Li, C., Gan, Z., Yang, Z., Yang, J., Li, L., Wang, L., and Gao, J. (2023b).
\newblock Multimodal foundation models: From specialists to general-purpose assistants.
\newblock {\em arXiv preprint arXiv:2309.10020}, 1(2):2.

\bibitem[Li et~al., 2022]{li2022pre}
Li, S., Puig, X., Paxton, C., Du, Y., Wang, C., Fan, L., Chen, T., Huang, D.-A., Aky{\"u}rek, E., Anandkumar, A., et~al. (2022).
\newblock Pre-trained language models for interactive decision-making.
\newblock {\em Advances in Neural Information Processing Systems}, 35:31199--31212.

\bibitem[Lin et~al., 2023a]{lin2023grounded}
Lin, B.~Y., Huang, C., Liu, Q., Gu, W., Sommerer, S., and Ren, X. (2023a).
\newblock On grounded planning for embodied tasks with language models.
\newblock In {\em Proceedings of the AAAI Conference on Artificial Intelligence}, volume~37, pages 13192--13200.

\bibitem[Lin et~al., 2023b]{lin2023transformers}
Lin, L., Bai, Y., and Mei, S. (2023b).
\newblock Transformers as decision makers: Provable in-context reinforcement learning via supervised pretraining.
\newblock {\em arXiv preprint arXiv:2310.08566}.

\bibitem[Liu et~al., 2021]{liu2021makes}
Liu, J., Shen, D., Zhang, Y., Dolan, B., Carin, L., and Chen, W. (2021).
\newblock What makes good in-context examples for gpt-$3 $?
\newblock {\em arXiv preprint arXiv:2101.06804}.

\bibitem[Liu et~al., 2022a]{liu2022goal}
Liu, M., Zhu, M., and Zhang, W. (2022a).
\newblock Goal-conditioned reinforcement learning: Problems and solutions.
\newblock {\em arXiv preprint arXiv:2201.08299}.

\bibitem[Liu et~al., 2022b]{liu2022p}
Liu, X., Ji, K., Fu, Y., Tam, W., Du, Z., Yang, Z., and Tang, J. (2022b).
\newblock P-tuning: Prompt tuning can be comparable to fine-tuning across scales and tasks.
\newblock In {\em Proceedings of the 60th Annual Meeting of the Association for Computational Linguistics (Volume 2: Short Papers)}, pages 61--68.

\bibitem[Liu et~al., 2023]{liu2023reason}
Liu, Z., Hu, H., Zhang, S., Guo, H., Ke, S., Liu, B., and Wang, Z. (2023).
\newblock Reason for future, act for now: A principled framework for autonomous llm agents with provable sample efficiency.
\newblock {\em arXiv preprint arXiv:2309.17382}.

\bibitem[Mandi et~al., 2023]{mandi2023roco}
Mandi, Z., Jain, S., and Song, S. (2023).
\newblock Roco: Dialectic multi-robot collaboration with large language models.
\newblock {\em arXiv preprint arXiv:2307.04738}.

\bibitem[Merity et~al., 2016]{merity2016pointer}
Merity, S., Xiong, C., Bradbury, J., and Socher, R. (2016).
\newblock Pointer sentinel mixture models.
\newblock {\em arXiv preprint arXiv:1609.07843}.

\bibitem[Michel et~al., 2019]{michel2019sixteen}
Michel, P., Levy, O., and Neubig, G. (2019).
\newblock Are sixteen heads really better than one?
\newblock {\em Advances in neural information processing systems}, 32.

\bibitem[M{\"u}ller et~al., 2021]{muller2021transformers}
M{\"u}ller, S., Hollmann, N., Arango, S.~P., Grabocka, J., and Hutter, F. (2021).
\newblock Transformers can do bayesian inference.
\newblock {\em arXiv preprint arXiv:2112.10510}.

\bibitem[Munos, 2005]{munos2005error}
Munos, R. (2005).
\newblock Error bounds for approximate value iteration.
\newblock In {\em Proceedings of the National Conference on Artificial Intelligence}, volume~20, page 1006. Menlo Park, CA; Cambridge, MA; London; AAAI Press; MIT Press; 1999.

\bibitem[Nottingham et~al., 2023]{nottingham2023embodied}
Nottingham, K., Ammanabrolu, P., Suhr, A., Choi, Y., Hajishirzi, H., Singh, S., and Fox, R. (2023).
\newblock Do embodied agents dream of pixelated sheep?: Embodied decision making using language guided world modelling.
\newblock {\em arXiv preprint arXiv:2301.12050}.

\bibitem[OpenAI, 2023]{openai2023gpt}
OpenAI, R. (2023).
\newblock Gpt-4 technical report.
\newblock {\em arXiv}, pages 2303--08774.

\bibitem[Pateria et~al., 2021]{pateria2021hierarchical}
Pateria, S., Subagdja, B., Tan, A.-h., and Quek, C. (2021).
\newblock Hierarchical reinforcement learning: A comprehensive survey.
\newblock {\em ACM Computing Surveys (CSUR)}, 54(5):1--35.

\bibitem[Paulin, 2015]{paulin2015concentration}
Paulin, D. (2015).
\newblock Concentration inequalities for markov chains by marton couplings and spectral methods.

\bibitem[Polyanskiy and Wu, 2022]{polyanskiy2022information}
Polyanskiy, Y. and Wu, Y. (2022).
\newblock Information theory: From coding to learning.
\newblock {\em Book draft}.

\bibitem[Qiu et~al., 2022]{qiu2022contrastive}
Qiu, S., Wang, L., Bai, C., Yang, Z., and Wang, Z. (2022).
\newblock Contrastive ucb: Provably efficient contrastive self-supervised learning in online reinforcement learning.
\newblock In {\em International Conference on Machine Learning}, pages 18168--18210. PMLR.

\bibitem[Radford et~al., 2021]{radford2021learning}
Radford, A., Kim, J.~W., Hallacy, C., Ramesh, A., Goh, G., Agarwal, S., Sastry, G., Askell, A., Mishkin, P., Clark, J., et~al. (2021).
\newblock Learning transferable visual models from natural language supervision.
\newblock In {\em International conference on machine learning}, pages 8748--8763. PMLR.

\bibitem[Raffel et~al., 2020]{raffel2020exploring}
Raffel, C., Shazeer, N., Roberts, A., Lee, K., Narang, S., Matena, M., Zhou, Y., Li, W., and Liu, P.~J. (2020).
\newblock Exploring the limits of transfer learning with a unified text-to-text transformer.
\newblock {\em The Journal of Machine Learning Research}, 21(1):5485--5551.

\bibitem[Reid et~al., 2022]{reid2022can}
Reid, M., Yamada, Y., and Gu, S.~S. (2022).
\newblock Can wikipedia help offline reinforcement learning?
\newblock {\em arXiv preprint arXiv:2201.12122}.

\bibitem[Ross and Bagnell, 2010]{ross2010efficient}
Ross, S. and Bagnell, D. (2010).
\newblock Efficient reductions for imitation learning.
\newblock In {\em Proceedings of the thirteenth international conference on artificial intelligence and statistics}, pages 661--668. JMLR Workshop and Conference Proceedings.

\bibitem[Ross et~al., 2011]{ross2011reduction}
Ross, S., Gordon, G., and Bagnell, D. (2011).
\newblock A reduction of imitation learning and structured prediction to no-regret online learning.
\newblock In {\em Proceedings of the fourteenth international conference on artificial intelligence and statistics}, pages 627--635. JMLR Workshop and Conference Proceedings.

\bibitem[Schiappa et~al., 2023]{schiappa2023self}
Schiappa, M.~C., Rawat, Y.~S., and Shah, M. (2023).
\newblock Self-supervised learning for videos: A survey.
\newblock {\em ACM Computing Surveys}, 55(13s):1--37.

\bibitem[Singh et~al., 2023]{singh2023progprompt}
Singh, I., Blukis, V., Mousavian, A., Goyal, A., Xu, D., Tremblay, J., Fox, D., Thomason, J., and Garg, A. (2023).
\newblock Progprompt: Generating situated robot task plans using large language models.
\newblock In {\em 2023 IEEE International Conference on Robotics and Automation (ICRA)}, pages 11523--11530. IEEE.

\bibitem[Sutton and Barto, 2018]{sutton2018reinforcement}
Sutton, R.~S. and Barto, A.~G. (2018).
\newblock {\em Reinforcement learning: An introduction}.
\newblock MIT press.

\bibitem[Team et~al., 2023]{team2023gemini}
Team, G., Anil, R., Borgeaud, S., Wu, Y., Alayrac, J.-B., Yu, J., Soricut, R., Schalkwyk, J., Dai, A.~M., Hauth, A., et~al. (2023).
\newblock Gemini: a family of highly capable multimodal models.
\newblock {\em arXiv preprint arXiv:2312.11805}.

\bibitem[Tokic and Palm, 2011]{tokic2011value}
Tokic, M. and Palm, G. (2011).
\newblock Value-difference based exploration: adaptive control between epsilon-greedy and softmax.
\newblock In {\em Annual conference on artificial intelligence}, pages 335--346. Springer.

\bibitem[Touvron et~al., 2023]{touvron2023llama}
Touvron, H., Martin, L., Stone, K., Albert, P., Almahairi, A., Babaei, Y., Bashlykov, N., Batra, S., Bhargava, P., Bhosale, S., et~al. (2023).
\newblock Llama 2: Open foundation and fine-tuned chat models.
\newblock {\em arXiv preprint arXiv:2307.09288}.

\bibitem[Van~Handel, 2014]{van2014probability}
Van~Handel, R. (2014).
\newblock Probability in high dimension.
\newblock {\em Lecture Notes (Princeton University)}.

\bibitem[Wang et~al., 2023a]{wang2023large}
Wang, X., Zhu, W., Saxon, M., Steyvers, M., and Wang, W.~Y. (2023a).
\newblock Large language models are latent variable models: Explaining and finding good demonstrations for in-context learning.
\newblock In {\em Thirty-seventh Conference on Neural Information Processing Systems}.

\bibitem[Wang et~al., 2023b]{wang2023empowering}
Wang, Y., Jiao, R., Lang, C., Zhan, S.~S., Huang, C., Wang, Z., Yang, Z., and Zhu, Q. (2023b).
\newblock Empowering autonomous driving with large language models: A safety perspective.
\newblock {\em arXiv preprint arXiv:2312.00812}.

\bibitem[Wei et~al., 2022a]{wei2022emergent}
Wei, J., Tay, Y., Bommasani, R., Raffel, C., Zoph, B., Borgeaud, S., Yogatama, D., Bosma, M., Zhou, D., Metzler, D., et~al. (2022a).
\newblock Emergent abilities of large language models.
\newblock {\em arXiv preprint arXiv:2206.07682}.

\bibitem[Wei et~al., 2022b]{wei2022chain}
Wei, J., Wang, X., Schuurmans, D., Bosma, M., Xia, F., Chi, E., Le, Q.~V., Zhou, D., et~al. (2022b).
\newblock Chain-of-thought prompting elicits reasoning in large language models.
\newblock {\em Advances in Neural Information Processing Systems}, 35:24824--24837.

\bibitem[Wies et~al., 2023]{wies2023learnability}
Wies, N., Levine, Y., and Shashua, A. (2023).
\newblock The learnability of in-context learning.
\newblock {\em arXiv preprint arXiv:2303.07895}.

\bibitem[Xie et~al., 2021]{xie2021explanation}
Xie, S.~M., Raghunathan, A., Liang, P., and Ma, T. (2021).
\newblock An explanation of in-context learning as implicit bayesian inference.
\newblock {\em arXiv preprint arXiv:2111.02080}.

\bibitem[Xu et~al., 2023]{xu2023multimodal}
Xu, P., Zhu, X., and Clifton, D.~A. (2023).
\newblock Multimodal learning with transformers: A survey.
\newblock {\em IEEE Transactions on Pattern Analysis and Machine Intelligence}.

\bibitem[Yao et~al., 2023a]{yao2023tree}
Yao, S., Yu, D., Zhao, J., Shafran, I., Griffiths, T.~L., Cao, Y., and Narasimhan, K. (2023a).
\newblock Tree of thoughts: Deliberate problem solving with large language models.
\newblock {\em arXiv preprint arXiv:2305.10601}.

\bibitem[Yao et~al., 2023b]{yao2023retroformer}
Yao, W., Heinecke, S., Niebles, J.~C., Liu, Z., Feng, Y., Xue, L., Murthy, R., Chen, Z., Zhang, J., Arpit, D., et~al. (2023b).
\newblock Retroformer: Retrospective large language agents with policy gradient optimization.
\newblock {\em arXiv preprint arXiv:2308.02151}.

\bibitem[Yarotsky, 2017]{yarotsky2017error}
Yarotsky, D. (2017).
\newblock Error bounds for approximations with deep relu networks.
\newblock {\em Neural Networks}, 94:103--114.

\bibitem[Zhang, 2022]{zhang2022feel}
Zhang, T. (2022).
\newblock Feel-good thompson sampling for contextual bandits and reinforcement learning.
\newblock {\em SIAM Journal on Mathematics of Data Science}, 4(2):834--857.

\bibitem[Zhang, 2023]{zhang2023mathematical}
Zhang, T. (2023).
\newblock {\em Mathematical analysis of machine learning algorithms}.
\newblock Cambridge University Press.

\bibitem[Zhang et~al., 2022]{zhang2022making}
Zhang, T., Ren, T., Yang, M., Gonzalez, J., Schuurmans, D., and Dai, B. (2022).
\newblock Making linear mdps practical via contrastive representation learning.
\newblock In {\em International Conference on Machine Learning}, pages 26447--26466. PMLR.

\bibitem[Zhang et~al., 2023]{zhang2023and}
Zhang, Y., Zhang, F., Yang, Z., and Wang, Z. (2023).
\newblock What and how does in-context learning learn? bayesian model averaging, parameterization, and generalization.
\newblock {\em arXiv preprint arXiv:2305.19420}.

\end{thebibliography}
